\documentclass[11pt]{article}
\usepackage{arxiv}

\usepackage[utf8]{inputenc} % allow utf-8 input
\usepackage[T1]{fontenc}    % use 8-bit T1 fonts
\usepackage{hyperref}       % hyperlinks
\usepackage{url}            % simple URL typesetting
\usepackage{booktabs}       % professional-quality tables
\usepackage{amsfonts}       % blackboard math symbols
\usepackage{nicefrac}       % compact symbols for 1/2, etc.
\usepackage{microtype}      % microtypography
\usepackage{lipsum}		% Can be removed after putting your text content
\usepackage{graphicx}
\usepackage{algorithm}
\usepackage{algorithmic}
\usepackage{authblk}
\usepackage{natbib}
\usepackage{times}
\usepackage{soul}
\usepackage{cancel}

\usepackage{url}
\usepackage{graphicx}
\usepackage{amsmath}
\usepackage{amsthm}
\usepackage{booktabs}
\usepackage{amsfonts}       % blackboard math symbols
\usepackage{nicefrac}       % compact symbols for 1/2, etc.
\usepackage{microtype}      % microtypography
\usepackage{xcolor}         % colors
\usepackage{caption}
\usepackage{subcaption}
\usepackage{amssymb}
\usepackage{mathtools}
\usepackage{xr}
\usepackage{threeparttable}

\newtheorem{theorem}{Theorem}
\newtheorem*{theorem*}{Theorem}
\newtheorem*{corollary*}{Corollary}
\newtheorem{assumption}{Assumption}

\newtheorem*{assumption*}{Assumption}
\newtheorem{definition}{Definition}
\newtheorem{proposition}{Proposition}
\newtheorem{lemma}{Lemma}

\newtheorem*{claim*}{Claim}

\newcounter{inlineenum}
\renewcommand{\theinlineenum}{\roman{inlineenum}}
\newenvironment{inlineenum}
  {\unskip\ignorespaces
  \setcounter{inlineenum}{0}%
   \renewcommand{\item}{\refstepcounter{inlineenum}{\textit{\theinlineenum})~}}
   }
  {\ignorespacesafterend}
 
\newtheorem{corollary}[theorem]{Corollary}

\DeclareMathOperator*{\essinf}{ess\,inf}

\DeclareFontFamily{U}{matha}{\hyphenchar\font45}
\DeclareFontShape{U}{matha}{m}{n}{
	<5> <6> <7> <8> <9> <10> gen * matha
	<10.95> matha10 <12> <14.4> <17.28> <20.74> <24.88> matha12
}{}
\DeclareSymbolFont{matha}{U}{matha}{m}{n}
\DeclareFontSubstitution{U}{matha}{m}{n}

\DeclareFontFamily{U}{mathx}{\hyphenchar\font45}
\DeclareFontShape{U}{mathx}{m}{n}{
	<5> <6> <7> <8> <9> <10>
	<10.95> <12> <14.4> <17.28> <20.74> <24.88>
	mathx10
}{}
\DeclareSymbolFont{mathx}{U}{mathx}{m}{n}
\DeclareFontSubstitution{U}{mathx}{m}{n}

\DeclareMathDelimiter{\vvvert}{0}{matha}{"7E}{mathx}{"17}

\title{Distributionally Robust Policy Evaluation and Learning for Continuous Treatment with Observational Data}

\author[1]{Cheuk Hang Leung$^*$}
\author[2]{Yiyan Huang$^*$}
\author[1]{Yijun Li}
\author[1]{Qi Wu $^\mathcal{y}$}

\affil[1]{City University of Hong Kong}
\affil[2]{The Hong Kong Polytechnic University}

\date{}

\renewcommand{\headeright}{}
\renewcommand{\undertitle}{}
\renewcommand{\shorttitle}{39th Annual AAAI Conference on Artificial Intelligence, Philadelphia, Pennsylvania, USA.}

\hypersetup{
pdftitle={A template for the arxiv style},
pdfsubject={q-bio.NC, q-bio.QM},
pdfauthor={David S.~Hippocampus, Elias D.~Striatum},
pdfkeywords={First keyword, Second keyword, More},
}

\begin{document}
\maketitle
\makeatletter\def\Hy@Warning#1{}\makeatother
\def\thefootnote{*}\footnotetext{These authors contributed equally to this work.}
\def\thefootnote{$\mathcal{y}$}\footnotetext{The corresponding author (qiwu@cityu.edu.hk).}

\begin{abstract}
Using offline observational data for policy evaluation and learning allows decision-makers to evaluate and learn a policy that connects characteristics and interventions. Most existing literature has focused on either discrete treatment spaces or assumed no difference in the distributions between the policy-learning and policy-deployed environments. These restrict applications in many real-world scenarios where distribution shifts are present with continuous treatment. To overcome these challenges, this paper focuses on developing a distributionally robust policy under a continuous treatment setting. The proposed distributionally robust estimators are established using the Inverse Probability Weighting (IPW) method extended from the discrete one for policy evaluation and learning under continuous treatments. Specifically, we introduce a kernel function into the proposed IPW estimator to mitigate the exclusion of observations that can occur in the standard IPW method to continuous treatments. We then provide finite-sample analysis that guarantees the convergence of the proposed distributionally robust policy evaluation and learning estimators. The comprehensive experiments further verify the effectiveness of our approach when distribution shifts are present.
\end{abstract}

\section{Introduction}\label{sec:introduction}
Most decision-making problems necessitate learning an effective personalized policy based on individual features from observational data. This process, commonly referred to as offline policy evaluation/learning, has diverse applications across various domains, including healthcare \cite{tang2021model}, recommendation \cite{li2010contextual}, and finance \cite{qin2022neorl}. Many studies have investigated offline policy evaluation/learning in discrete treatment settings which assume that the deployment environment is identical to the environment generating the training data, i.e., that there are no distributional shifts. This assumption, however, is often unrealistic in many real-world applications \cite{huang2023towards}. For instance, an investment firm has developed an automated investment strategy for the US stock market based on extensive historical trading data. When attempting to apply this strategy directly to the UK stock market, it may lose the predictive power due to the substantial difference between financial market environments. Similarly, a pharmaceutical company has developed a strategy for individualized Warfarin dosage adjustment according to their recent research on older adults. This strategy may perform well in the original clinical trial population, yet it may falter when applied to a new population, such as young adults, due to significant differences in physical conditions.

To address the challenge of distribution shifts in policy evaluation and learning, the problem can be formulated as a Distributionally Robust Optimization (DRO) problem. In the DRO framework, the goal is to find the worst-case solution within a set of distributions under certain degrees of model uncertainties. The uncertainty set is assumed to contain the distributions due to potential distribution shifts, and it can be characterized by constraining certain moments of order \cite{delage2010distributionally, zymler2013distributionally} or by using divergence measures \cite{hu2013kullback, kuhn2019wasserstein, chen2019distributionally, gao2022wasserstein} to define appropriate deviations from a nominal distribution. The resulting solution provides robust, reliable, and conservative guarantees which can cope with the most adverse situations.
%}
 
Furthermore, the objective function of the formulated DRO utilizes an inverse probability weighting (IPW) estimator \cite{wooldridge2007inverse} to estimate the expected potential reward under continuous treatment. Specifically, we extend the existing IPW estimator designed for discrete treatment settings to accommodate continuous treatments. Generally, the discrete-based IPW approach cannot be directly applied in continuous treatment settings, as it would reject most observed data. Moreover, although discretizing continuous treatments into categories is an intuitive and simple solution, it can lead to information loss and may fail to produce inferences that vary continuously with the treatment. We introduce a modified IPW approach incorporating a scaled kernel function with a bandwidth parameter, serving as a smooth nonparametric extension for computing histogram ``buckets". The proposed IPW estimator enables the distributionally robust policy evaluation and learning using observational datasets.

In summary, our framework addresses the challenges of policy evaluation and learning in continuous treatment settings in the presence of distribution shifts. The key contributions of our paper are threefold:
\begin{enumerate}
\item We formulate the DRO problem with an IPW approach for policy evaluation/learning under the continuous treatment setting, and convert it to its equivalent dual form. As the standard IPW approach is not directly applicable in this context, we develop a tractable kernel-based form to approximate the dual problem.
\item We establish estimators for policy evaluation/learning and investigate their asymptotic properties. Specifically, the established estimators of distributionally robust values exhibit asymptotic normality, and the finite-sample regret decays to zero asymptotically.
\item Through simulated and empirical studies, we demonstrate that the policy learned using our method provides robustness to distribution shifts compared to standard nonrobust policy learning methods.
\end{enumerate}

\subsection{Literature Review}\label{sec:literature}
Considerable research has focused on causality in discrete treatment settings. However, exploring causality under continuous treatment remains limited in many real-world applications. Existing research on continuous treatment settings primarily focuses on directly modeling the relationship among response, treatment, and covariates. Notable contributions include \cite{Schwab_Linhardt_Bauer_Buhmann_Karlen_2020}, who construct a multi-head neural network for this purpose; \cite{bica2020estimating}, who propose an end-to-end neural network based on generative adversarial networks (GANs); and \cite{bahadori2022end}, who introduce a novel algorithm within the entropy balancing framework to optimize accuracy through end-to-end optimization. Another approach modifies IPW-based and Doubly Robust-based estimators (e.g., \cite{chernozhukov2018double, huang2022robust}) from discrete treatment settings by incorporating kernel functions to mitigate the direct rejection of observed data, as demonstrated by \cite{su2019non} and \cite{colangelo2019double}.

Recent studies have focused on offline policy evaluation and learning. \cite{kitagawa2018should} establish finite sample regret bounds with a rate of $O_{P}\left(1/{\sqrt{N}}\right)$ for policy learning over a policy class with finite VC dimension. \cite{athey2021policy} extend this analysis to examine regret bounds from an asymptotic perspective. \cite{zhao2012estimating} and \cite{zhou2017residual} propose algorithms for policy learning and explore the statistical properties of learned policies and associated regret bounds. \cite{10.5555/3104482.3104620} utilize classic estimators for policy evaluation. \cite{kallus2018balanced} proposes a balance-based approach to reweight historical data and mimic datasets generated by evaluated or learned policies. \cite{zhou2023offline} exploit a cross-fitted approach for policy learning. The aforementioned studies primarily assume discrete treatment and do not account for distributional shifts. Nevertheless, distributional shifts are common since the studies are often conducted in different environments, highlighting the significance of studying distributionally robust policies. For instance, \cite{yang2023distributionally}, \cite{shen2024wasserstein}, \cite{mo2021learning}, and \cite{Faury_Tanielian_Dohmatob_Smirnova_Vasile_2020} primarily focus on shifts in covariates, whereas \cite{si2023distributionally} and \cite{kallus2022doubly} address shifts in the joint distribution of responses, features, and treatments.
%Papers investigating offline policy evaluation and learning in the presence of distributional shift include \cite{si2023distributionally}, \cite{kallus2022doubly}, \cite{Faury_Tanielian_Dohmatob_Smirnova_Vasile_2020}, \cite{yang2023distributionally}, \cite{shen2024wasserstein} and \cite{mo2021learning}. 
Notably, to the best of our knowledge, studying policy evaluation and learning in the presence of distribution shifts under continuous treatment settings is still an open problem.

\section{Background}\label{sec:background}
\subsection{Notations and Assumptions}\label{sec:notations}
Throughout the paper, we denote $A\in\mathcal{A}\subset\mathbb{R}$, $X\in\mathcal{X}\subset\mathbb{R}^{d}$, and $Y\in\mathcal{Y}\subset\mathbb{R}$ as the continuous treatment (also known as action or intervention), the covariates, and the continuous response (also known as outcome), respectively. We write $Y(A)$ to mean the potential response variable under the treatment $A$. We also assume that $Y$ and $Y(A)$ are non-negative bounded variables, i.e., there exists $M>0$ such that $0\leq Y(a),\;Y\leq M$. Finally, we let $(X_{i},A_{i},Y_{i})_{i=1}^{N}$ be $N$ independent and identically distributed (i.i.d.) triples from a fixed underlying distribution, and the probability measure of the underlying distribution is denoted as $\mathbb{P}_{0}$.

%Further, we adopt the Rubin potential outcome framework (e.g., \cite{rubin1974estimating, imbens2004nonparametric, imbens2015causal, Huang_Leung_Yan_Wu_Peng_Wang_Huang_2021, huang2024dignet, Li_Leung_Sun_Wang_Huang_Yan_Wu_Wang_Huang_2024}). Throughout the paper, we impose the following (\textit{causal}) assumptions that are standard in the causal inference literature: 
%\begin{assumption}[\textit{Consistency}]\label{ass:consistency}
%	If $A=a$, we have $Y=Y(a)$.
%\end{assumption}
%\begin{assumption}[\textit{Unconfoundedness}]\label{ass:unconfoundedness}
%	$Y(a)\perp\!\!\!\!\perp A|X, \; \forall a$.
%\end{assumption}
%The positivity assumption in standard discrete treatment setting is modified as follows (e.g., \cite{kallus2018policy, colangelo2019double}):
%\begin{assumption}[\textit{Positivity}]\label{ass:positivity}
%There exists a positive constant $\epsilon>0$ such that $\underset{a\in\mathcal{A}}{\inf}\underset{x\in\mathcal{X}}{\essinf}f_{0}(a|x)\geq \epsilon>0$.
%\end{assumption}
%We also assume that $f_{0}(y|a,x)$ and $f_{0}(a|x)$ are three-times differentiable w.r.t. $a$ and bounded uniformly on $(y,a,x)\in(\mathcal{Y},\mathcal{A},\mathcal{X})$. All proofs of the theorems are in the Appendix\footnote{Available in the ``proof'' file of the Supplementary Material.}.

Further, we adopt the Rubin potential outcome framework (e.g., \cite{rubin1974estimating, imbens2004nonparametric, imbens2015causal, Huang_Leung_Yan_Wu_Peng_Wang_Huang_2021, huang2024dignet, Li_Leung_Sun_Wang_Huang_Yan_Wu_Wang_Huang_2024}). Throughout the paper, we impose the following (\textit{causal}) assumptions that are standard in the causal inference literature: 
\begin{assumption}[\textit{Consistency}]\label{ass:consistency}
	If $A=a$, we have $Y=Y(a)$.
\end{assumption}
\begin{assumption}[\textit{Unconfoundedness}]\label{ass:unconfoundedness}
	$Y(a)\perp\!\!\!\!\perp A|X, \; \forall a$.
\end{assumption}
\begin{assumption}[\textit{Positivity}]\label{ass:positivity}
	There exists a positive constant $\epsilon>0$ such that $\underset{a\in\mathcal{A}}{\inf}\underset{x\in\mathcal{X}}{\essinf}f_{0}(a|x)\geq \epsilon>0$.
\end{assumption}

Additionally, we follow \cite{kallus2018policy, colangelo2019double} and impose differentiability assumptions on the probability density functions $f_{0}(y|a,x)$ and $f_{0}(a|x)$: $f_{0}(y|a,x)$ and $f_{0}(a|x)$ are three-times differentiable w.r.t. $a$ and bounded uniformly on $(y,a,x)\in(\mathcal{Y},\mathcal{A},\mathcal{X})$. All proofs of the theorems are in the Appendix\footnote{Available in the ``proof'' file of the Supplementary Material.}.
\subsection{Problem Setup}\label{sec:setup}
Our objective is to find a policy $\pi^{\ast}$ that maps $\mathcal{X}$ to $\mathcal{A}$ within a policy class $\Pi$ that maximizes the expected outcomes
%\footnote{In some cases, one may need to minimize the expected outcomes. Under this scenario, Eqn. \eqref{eqt:policy learning} becomes $\pi^{\ast}=\underset{\pi\in\Pi}{\arg\min}\;Q(\pi)\coloneqq\underset{\pi\in\Pi}{\arg\min}\;\mathbb{E}[Y(\pi(X))]$. Since $\underset{\pi\in\Pi}{\max}\;\mathbb{E}[Y(\pi(X))]=-\underset{\pi\in\Pi}{\min}\;\mathbb{E}[-Y(\pi(X))]$, we have $\pi^{\ast}=\underset{\pi\in\Pi}{\arg\max}\;\mathbb{E}[Y(\pi(X))]=\underset{\pi\in\Pi}{\arg\min}\;\mathbb{E}[-Y(\pi(X))]$. As such, we stick with our settings.}
, i.e.,
\begin{equation}
\begin{aligned}\label{eqt:policy learning}
\pi^{\ast}=\underset{\pi\in\Pi}{\arg\max}\;Q(\pi)\coloneqq\underset{\pi\in\Pi}{\arg\max}\;\mathbb{E}_{\mathbb{P}_{0}}[Y(\pi(X))].
\end{aligned}
\end{equation}
The learned policy obtained in Eqn. \eqref{eqt:policy learning} may not generalize well to a new environment with a distribution that differs from $\mathbb{P}_{0}$. As such, we can consider distributionally robust formulation of Eqn. \eqref{eqt:policy learning}:
\begin{equation}
\begin{gathered}\label{eqt:robust policy learning}
\pi_{\mathrm{DRO}}^{*}=\underset{\pi\in\Pi}{\arg\max}\;Q_{\mathrm{DRO}}(\pi),\quad\text{where}\\
\begin{aligned}
Q_{\mathrm{DRO}}(\pi)&=\underset{\mathbb{P}\in\mathcal{U}_{\mathbb{P}_{0}}(\eta)}{\inf}\;\mathbb{E}_{\mathbb{P}}[Y(\pi(X))],\\
\quad\mathcal{U}_{\mathbb{P}_{0}}(\eta)&=\{\mathbb{P}:\mathbb{D}(\mathbb{P}\|\mathbb{P}_{0})\leq\eta\}.
\end{aligned}
\end{gathered}
\end{equation}
Here, $\mathbb{D}(\cdot\|\cdot)$ denotes the distribution discrepancy. Throughout the paper, we choose it as the Kullback–Leibler (KL) divergence \cite{Kullback59, Kullback_1951}\footnote{Other measures such as the Wasserstein metric or other $\phi$-divergence measures can be utilized (e.g., see \cite{kuhn2019wasserstein} and \cite{husain2023distributionally}). However, these approaches typically involve solving multi-level optimization problems which can be challenging to analyze.}.
$\mathcal{U}_{\mathbb{P}_{0}}(\eta)$ is the \textit{ambiguity set} (also known as the uncertainty set) with an \textit{ambiguity radius} $\eta$. The ambiguity set contains all the possible  distributions $\mathbb{P}$ such that the discrepancy of $\mathbb{P}$ relative to $\mathbb{P}_{0}$ is at most $\eta$.

\section{Distributionally Robust Policy Evaluation}\label{sec:theory}
\subsection{The Estimation of \texorpdfstring{$Q_{\mathrm{DRO}}(\pi)$}{}}\label{sec:estimator of QDRO}
%$Q_{\mathrm{DRO}}(\pi)$ is an optimization problem.
%For ease of computation, we can consider the Lagrangian dual of $Q_{\mathrm{DRO}}(\pi)$. As proven in \cite{hu2013kullback}, Eqn. \eqref{eqt:robust policy learning} can be reformulated as follows:

As proven in \cite{hu2013kullback}, Eqn. \eqref{eqt:robust policy learning} is equivalent to solving its Lagrangian dual, which is given as follows:
%{\small
\begin{equation}
\begin{aligned}\label{eqt:unconstrained dual equivalent}
&-\underset{\alpha\geq 0}{\min}\Bigg\{\alpha\log\mathbb{E}\bigg[e^{\frac{-Y(\pi(X))}{\alpha}}\bigg]+\alpha\eta\Bigg\}=\underset{\alpha\geq 0}{\max}\Bigg\{-\alpha\log\mathbb{E}\bigg[e^{\frac{-Y(\pi(X))}{\alpha}}\bigg]-\alpha\eta\Bigg\}\coloneqq\underset{\alpha\geq 0}{\max}\;\phi(\pi,\alpha).
\end{aligned}
\end{equation}
%}\noindent
%$\mathbb{E}\bigg[e^{\frac{-Y(\pi(X))}{\alpha}}\bigg]$ in Eqn. \eqref{eqt:unconstrained dual equivalent} is not \textit{identifiable}, meaning that we cannot estimate it from an observed dataset since the potential outcomes $Y(\pi(X))$ are unobservable. It is thus necessary to reformulate it to other forms for causal study. Following the proving notions presented in \cite{horvitz1952generalization}, we arrive Lemma \ref{lemma:expectation equal IPW}.

Since $Y(\pi(X))$ in Eqn. \eqref{eqt:unconstrained dual equivalent} is inaccessible, we reformulate $\mathbb{E}[e^{\frac{-Y(\pi(X))}{\alpha}}]$ to an IPW form similar to that in \cite{horvitz1952generalization}. The result is given in Lemma \ref{lemma:expectation equal IPW}.
%Since the potential outcome $Y(\pi(X))$ in Eqn. \eqref{eqt:unconstrained dual equivalent} cannot be accessed directly from observational data, we follow notions presented in \cite{horvitz1952generalization} to reformulate $\mathbb{E}[e^{\frac{-Y(\pi(X))}{\alpha}}]$ to an IPW form, as given in Lemma \ref{lemma:expectation equal IPW}.

\begin{lemma}\label{lemma:expectation equal IPW}
Under Assumptions \ref{ass:consistency} - \ref{ass:positivity}, we have
%{\small
\begin{equation}
\begin{aligned}\label{eqt:equivalent form delta}
\mathbb{E}\bigg[e^{\frac{-Y(\pi(X))}{\alpha}}\bigg]=\mathbb{E}\bigg[\frac{\delta(\pi(X)-A)}{f_{0}(A|X)}e^{\frac{-Y}{\alpha}}\bigg]
\end{aligned}
\end{equation}
%}\noindent
%$\mathbb{E}\bigg[e^{\frac{-Y(\pi(X))}{\alpha}}\bigg]=\mathbb{E}\bigg[\frac{\delta(\pi(X)-A)}{f_{0}(A|X)}e^{\frac{-Y}{\alpha}}\bigg]$
for any $\alpha\geq 0$, where $\delta(\cdot)$ is the Dirac Delta function\footnote{$\delta(x)=\begin{cases}
\infty & x=0\\
0& \text{otherwise}
\end{cases}$ such that i) $\int_{\mathbb{R}}\delta(x)dx=1$ and ii) $\int_{\mathbb{R}}\delta(x)f(x)dx=f(0)$ for any arbitrary $f$ defined on $\mathbb{R}$.}.
\end{lemma}

\noindent Using Lemma \ref{lemma:expectation equal IPW}, the expectation in Eqn. \eqref{eqt:unconstrained dual equivalent} can be replaced according to Eqn. \eqref{eqt:equivalent form delta}. Note that the Dirac function $\delta(\cdot)$ is a theoretical generalized function and is often approximated by the scaled kernel function $K_{h}(\cdot)$\footnote{A bounded differentiable function $K(\cdot)$ (i.e., $|K(\cdot)|\leq M_{K}$) is said to be a second-order kernel function if it satisfies 
\begin{inlineenum}
\item $K(\cdot)$ is a symmetric function;
\item $\int_{-\infty}^{\infty} uK(u)du=0$;
\item $\int_{-\infty}^{\infty} K(u)du=1$. 
\end{inlineenum}
The scaled kernel function $K_{h}(\cdot)$ is defined such that $K_{h}(x)=\frac{1}{h}K\big(\frac{x}{h}\big)$, where $h$ is termed as the \textit{bandwidth} parameter. Note that $K_{h}(x)\overset{w}{\rightarrow}\delta(x)$ when $h\rightarrow 0$. Examples of kernels include Gaussian kernels or the Epanechnikov kernel.}. As a result, we can consider the following approximated form:
\noindent
%{\small
\begin{gather}
Q_{\mathrm{DRO}}^{h}(\pi)=\underset{\alpha\geq 0}{\sup}\Bigg\{-\alpha\log\mathbb{E}\Bigg[\frac{e^{-\frac{Y}{\alpha}}K_{h}(\pi(X)-A)}{f_{0}(A|X)}\Bigg]-\alpha\eta\Bigg\}.\label{eqt:IPW form robust kernel}
\end{gather}
%}\noindent 
The above two quantities, $Q_{\text{DRO}}(\pi)$ and $Q_{\text{DRO}}^{h}(\pi)$, bring two important insights: (1) The optimal solutions of $Q_{\text{DRO}}(\pi)$ and $Q_{\text{DRO}}^{h}(\pi)$ are obtained by solving Eqns. \eqref{eqt:unconstrained dual equivalent} and \eqref{eqt:IPW form robust kernel} which are attainable for positive $\alpha$ due to the causal assumption. Further, the optimal solutions are finite for any $\pi$ (see Auxiliary Result 2 in Appendix for details); (2) $Q_{\text{DRO}}^{h}(\pi)\rightarrow Q_{\text{DRO}}(\pi)$ as $h\rightarrow 0$. These two insights, consequently, guarantee that the optimal solutions of $Q_{\text{DRO}}^{h}(\pi)$ also converge to the optimal solutions of $Q_{\text{DRO}}(\pi)$. Therefore, we can construct an estimator of the IPW-based distributionally robust value $Q_{\mathrm{DRO}}^{h}(\pi)$ to study the original distributionally robust value $Q_{\mathrm{DRO}}(\pi)$ in Eqn. \eqref{eqt:robust policy learning}. We define
\begin{subequations}
	\begin{align}
		&\bar{W}_{N}^{h}(\pi,\alpha)=\frac{1}{N}\underset{i=1}{\overset{N}{\sum}}\frac{K_{h}(\pi(X_{i})-A_{i})}{f_{0}(A_{i}|X_{i})}e^{\frac{-Y_{i}}{\alpha}} \label{eqt:W_bar_IPW}.\\
		&\hat{W}_{N}^{h}(\pi,\alpha)=\frac{\bar{W}_{N}^{h}(\pi,\alpha)}{S_{N}^{h}}=\frac{\bar{W}_{N}^{h}(\pi,\alpha)}{\frac{1}{N}\underset{i=1}{\overset{N}{\sum}}\frac{K_{h}(\pi(X_{i})-A_{i})}{f_{0}(A_{i}|X_{i})}}. \label{eqt:W_hat IPW and S_hat}
	\end{align}
\end{subequations}
It is known that the IPW-based estimator $\bar{W}_{N}^{h}(\pi,\alpha)$ in Eqn. \eqref{eqt:W_bar_IPW} suffers from high-variance \cite{swaminathan2015self,khan2023adaptive}. To address this challenge, we can use a normalized estimator $\hat{W}_{N}^{h}(\pi,\alpha)$ with a normalization factor $S_{N}^{h}$ in Eqn. \eqref{eqt:W_hat IPW and S_hat} to approximate $\bar{W}_{N}^{h}(\pi,\alpha)$. Note that $\mathbb{E}[S_{N}^{h}]=1$ and $S_{N}^{h} \rightarrow 1$ almost surely (see Auxiliary Result 1 in Appendix). Thus, $\hat{W}_{N}^{h}(\pi,\alpha)$ is asymptotically equivalent to $\bar{W}_{N}^{h}(\pi,\alpha)$. Consequently, we use the following $\hat{Q}_{\mathrm{DRO}}^{h}(\pi)$ as the estimator of $Q_{\mathrm{DRO}}^{h}(\pi)$ in Eqn. \eqref{eqt:IPW form robust kernel}:
\begin{equation}
\begin{aligned}
\hat{Q}_{\mathrm{DRO}}^{h}(\pi)&=\underset{\alpha\geq 0}{\max}\;\hat{\phi}_{N}^{h}(\pi,\alpha):=\underset{\alpha\geq 0}{\max}\big\{-\alpha\log\hat{W}_{N}^{h}(\pi,\alpha)-\alpha\eta\big\}.
\end{aligned}\label{eqt:Q_hat IPW}
\end{equation}
$\hat{Q}_{\mathrm{DRO}}^{h}(\pi)$ can be used to evaluate distributional robustness of a policy $\pi$. To summarize, we present the specific steps of obtaining $\hat{Q}_{\mathrm{DRO}}^{h}(\pi)$ in the following Algorithm \ref{alg:algorithm eval}.

\begin{algorithm}
\caption{Distributionally robust policy evaluation 
%{\small$\hat{Q}_{\mathrm{DRO}}^{h}(\pi)$}
}
\label{alg:algorithm eval}
%\textbf{Parameter}: Optional list of parameters\\
%\textbf{Output} $\hat{Q}_{\mathrm{DRO}}^{h}(\pi)$.
\begin{algorithmic}[1] %[1] enables line numbers
%\State \textbf{Input} observed dataset $(X_{i},A_{i},Y_{i})_{i=1}^{N}$, $h$, policy $\pi\in\Pi$. Initialize: $\alpha\in\mathbb{R}^{+}\cup{0}$.
%\State \textbf{Output} $\hat{Q}_{\mathrm{DRO}}^{h}(\pi)$.
\STATE \textbf{Input} observed dataset $(X_{i},A_{i},Y_{i})_{i=1}^{N}$, $h$, policy $\pi\in\Pi$. Initialize: $\alpha\in\mathbb{R}^{+}\cup{0}$.

\REPEAT
\STATE Compute $\hat{W}_{N}^{h}(\pi,\alpha)$ given in Eqn. \eqref{eqt:W_hat IPW and S_hat}.
\STATE Update $\alpha$: $\alpha\leftarrow\alpha-\frac{\frac{\partial\hat{\phi}_{N}^{h}}{\partial\alpha}}{\frac{\partial^{2}\hat{\phi}_{N}^{h}}{\partial\alpha^{2}}}$, where 
%{\small
\begin{equation*}
\begin{aligned}
\frac{\partial\hat{\phi}_{N}^{h}}{\partial\alpha}&=-\eta-\log\hat{W}_{N}^{h}-\frac{\alpha\frac{\partial\hat{W}_{N}^{h}}{\partial\alpha}}{\hat{W}_{N}^{h}},\\
\frac{\partial^{2}\hat{\phi}_{N}^{h}}{\partial\alpha^{2}}&=-\frac{\frac{1}{N}\underset{i=1}{\overset{N}{\sum}}\frac{K_{h}(\pi(X_{i})-A_{i})}{f_{0}(A_{i}|X_{i})}Y_{i}^{2}e^{-\frac{Y_{i}}{\alpha}}}{\alpha^{3}S_{N}^{h}\hat{W}_{N}^{h}}+\frac{\alpha\big(\frac{\partial\hat{W}_{N}^{h}}{\partial\alpha}\big)^{2}}{(\hat{W}_{N}^{h})^{2}},\\
\frac{\partial\hat{W}_{N}^{h}}{\partial\alpha}&=\frac{\frac{1}{N}\underset{i=1}{\overset{N}{\sum}}\frac{K_{h}(\pi(X_{i})-A_{i})}{f_{0}(A_{i}|X_{i})}Y_{i}e^{-\frac{Y_{i}}{\alpha}}}{\alpha^{2}S_{N}^{h}}.
\end{aligned}
\end{equation*}
%}\noindent
\UNTIL $\alpha$ converges
\STATE \textbf{Return}  $\hat{Q}_{\mathrm{DRO}}^{h}(\pi)\leftarrow\hat{\pi}_{N}^{h}(\pi,\alpha)$
\end{algorithmic}
\end{algorithm}

\subsection{The Statistical Property of \texorpdfstring{$\hat{Q}_{\mathrm{DRO}}^{h}(\pi)$}{}}
As $\hat{Q}_{\mathrm{DRO}}^{h}(\pi)$ is an estimator established using observed empirical samples, it is important to delve into the finite-sample statistical performance guarantee for the estimator $\hat{Q}_{\mathrm{DRO}}^{h}(\pi)$. To achieve this, we first discuss the theoretical property of $\hat{W}_{N}^{h}$ in Theorem \ref{lemma:asy_res_normalized and thm:asy result normalized}.
\begin{theorem}\label{lemma:asy_res_normalized and thm:asy result normalized}
Suppose that $N\rightarrow\infty$, $h\rightarrow 0$ such that $Nh\rightarrow\infty$ and $Nh^{5}\rightarrow C\in[0,\infty)$. Then we have
\begin{equation}
\begin{aligned}\label{eqt:W_IPW hat asy result B}
&\sqrt{Nh}\bigg(\hat{W}_{N}^{h}-\mathbb{E}[e^{-\frac{Y(\pi(X))}{\alpha}}]-B_{\pi}(\alpha)h^{2}\bigg)
\overset{d}{\rightarrow}\mathcal{N}(0,\mathbb{V}_{\pi}(\alpha)),\\
&\text{{\normalsize where}} \quad B_{\pi}(\alpha)=\frac{\big(\int u^{2}K(u)du\big)}{2}\times\\
&\quad \quad \quad \quad \mathbb{E}\bigg[\mathbb{E}\bigg[e^{\frac{-Y}{\alpha}}\frac{\partial_{aa}^{2}f_{0}(Y|\pi(X),X)}{f_{0}(Y|\pi(X),X)}\bigg|A=\pi(X),X\bigg]\bigg],
\end{aligned}
\end{equation}
%and 
%{\small
\begin{gather}
\begin{aligned}\label{eqt:W_IPW hat asy result V}
\mathbb{V}_{\pi}(\alpha)&=\bigg(\int K(u)^{2}du\bigg)\times\\
&\;\Bigg\{\mathbb{E}\bigg[\mathbb{E}\bigg[\frac{e^{-\frac{2Y}{\alpha}}}{f_{0}(\pi(X)|X)}\bigg|A=\pi(X),X\bigg]\bigg]\\
&\;+\mathbb{E}\bigg[\frac{1}{f_{0}(\pi(X)|X)}\bigg](\mathbb{E}[e^{-\frac{Y(\pi(X))}{\alpha}}])^{2}\\
&\;-2\mathbb{E}\bigg[\mathbb{E}\bigg[\frac{e^{-\frac{Y}{\alpha}}}{f_{0}(\pi(X)|X)}\bigg|A=\pi(X),X\bigg]\bigg]\mathbb{E}[e^{-\frac{Y(\pi(X))}{\alpha}}]\Bigg\}.
\end{aligned}\raisetag{60pt}
\end{gather}
%}\noindent
\end{theorem}

The estimator $\hat{W}_{N}^{h}$ is the key component of $\hat{Q}_{\mathrm{DRO}}^{h}(\pi)$, as shown in Eqn. \eqref{eqt:Q_hat IPW}. Consequently, based on the statistical property of $\hat{W}_{N}^{h}$, we can derive the asymptotic normality of $\hat{Q}_{\mathrm{DRO}}^{h}(\pi)$ in Theorem \ref{lemma2:asy_res_normalized and thm:asy result normalized}.
\begin{theorem}\label{lemma2:asy_res_normalized and thm:asy result normalized}
Suppose that $N\rightarrow\infty$, $h\rightarrow 0$ such that $Nh\rightarrow\infty$ and $Nh^{5}\rightarrow C\in[0,\infty)$. Furthermore, denote $\alpha_{\ast}(\pi)$ s.t. $\phi(\pi,\alpha_{\ast}(\pi))\geq \phi(\pi,\alpha)$ $\forall\;\alpha\geq 0$.
Then we have
\begin{equation*}
		\begin{aligned}\label{eqt:Q_IPW hat asy result}
			&\sqrt{Nh}\Biggl(\hat{Q}_{\mathrm{DRO}}^{h}(\pi)-Q_{\mathrm{DRO}}(\pi)+\frac{\alpha_{\ast}(\pi)B_{\pi}(\alpha_{\ast}(\pi))}{\mathbb{E}\bigg[e^{-\frac{Y(\pi(X))}{\alpha_{\ast}(\pi)}}\bigg]}h^{2}\Biggl)\overset{d}{\rightarrow}\mathcal{N}\left(0,\frac{\alpha_{\ast}^{2}(\pi)\mathbb{V}_{\pi}(\alpha_{\ast}(\pi))}{\bigg(\mathbb{E}\bigg[e^{-\frac{Y(\pi(X))}{\alpha_{\ast}(\pi)}}\bigg]\bigg)^{2}}\right).
		\end{aligned}
\end{equation*}
\end{theorem}
A good choice of bandwidth is essential for effective policy learning and evaluation. We can use a rule-of-thumb bandwidth (see e.g., \cite{su2019non}), or select $h^{\ast}$ by minimizing the asymptotic mean squared error (AMSE) (e.g., \cite{kallus2018policy}) of $\hat{Q}_{\mathrm{DRO}}^{h}(\pi)$:
\begin{equation}
	\begin{aligned}\label{eqt:optimal bandwidth}
		h^{\ast}&\coloneqq\underset{}{\arg\min}\bigg[B_{\pi}(\alpha_{\ast}(\pi))^{2}h^{4}+\frac{\mathbb{V}_{\pi}(\alpha_{\ast}(\pi))}{Nh}\bigg]\\
		\Rightarrow h^{\ast}&=\bigg(\frac{\mathbb{V}_{\pi}(\alpha_{\ast}(\pi))}{4NB_{\pi}(\alpha_{\ast}(\pi))}\bigg)^{\frac{1}{5}}=\Theta(N^{-\frac{1}{5}}).
	\end{aligned}
\end{equation}
Empirically, we would follow the notions presented in \cite{kallus2018policy}, of which we choose the optimal bandwidth via a plug-in estimator.
\section{Distributionally Robust Policy Learning}
\subsection{The Estimation of \texorpdfstring{$\pi^{\ast}_{\mathrm{DRO}}$}{}}\label{sec:estimator of policy}
In the preceding section, we have established $\hat{Q}_{\mathrm{DRO}}^{h}(\pi)$ as an estimator for $Q_{\mathrm{DRO}}(\pi)$. Next, we aim to construct an estimator for the optimal policy $\pi^{\ast}_{\mathrm{DRO}}$. Specifically, we derive $\hat{\pi}_{\mathrm{DRO}}^{h}$ from $\hat{Q}_{\mathrm{DRO}}^{h}(\pi)$ such that
\begin{equation*}
\begin{aligned}\label{eqt:optimal policy estimator}
\hat{\pi}_{\mathrm{DRO}}^{h}&=\underset{\pi\in\Pi}{\arg\max}\;\hat{Q}_{\mathrm{DRO}}^{h}(\pi)=\underset{\pi\in\Pi}{\arg\max}\;\underset{\alpha\geq 0}{\max}\;\big\{-\alpha\log\hat{W}_{N}^{h}(\pi,\alpha)-\alpha\eta\big\}.
\end{aligned}
\end{equation*}
$\hat{\pi}_{\mathrm{DRO}}^{h}$ is the distributionally robust policy learned from $\hat{Q}_{\mathrm{DRO}}^{h}(\pi)$. To summarize, we present the specific steps of obtaining $\hat{\pi}_{\mathrm{DRO}}^{h}$ in Algorithm \ref{alg:algorithm policy}.

\begin{algorithm}
\caption{Distributionally robust policy learning 
%$\hat{\pi}_{\mathrm{DRO}}^{h}$
}\label{alg:algorithm policy}
%\textbf{Parameter}: Optional list of parameters\\
%\textbf{Output} $\hat{\pi}_{\mathrm{DRO}}^{h}$.
\begin{algorithmic}[1] %[1] enables line numbers
%\State \textbf{Input} observed dataset $(X_{i},A_{i},Y_{i})_{i=1}^{N}$, $h$. Initialize: $\pi\in\Pi$ and $\alpha\in\mathbb{R}^{+}\cup{0}$.
%\State \textbf{Output} $\hat{\pi}_{\mathrm{DRO}}^{h}$.
\STATE\textbf{Input} observed dataset $(X_{i},A_{i},Y_{i})_{i=1}^{N}$, $h$. Initialize: $\pi\in\Pi$ and $\alpha\in\mathbb{R}^{+}\cup{0}$.
\REPEAT
\STATE Compute $\hat{W}_{N}^{h}(\pi,\alpha)$ given in Eqn. \eqref{eqt:W_hat IPW and S_hat}.
%\Statex \vspace{-1pt}
%\Statex \vspace{-0.5pt}
%\Statex \vspace{-0.5pt}
%\Statex \vspace{-1.5pt}
\STATE Solve $\underset{\pi\in\Pi}{\min}\;\hat{W}_{N}^{h}(\pi,\alpha)$ using any numerical methods. Update $\pi$: $\pi\leftarrow\underset{\pi\in\Pi}{\arg\min}\;\hat{W}_{N}^{h}(\pi,\alpha)$.
\STATE Solve $\underset{\alpha\geq 0}{\max}\;\hat{\phi}_{N}^{h}(\pi,\alpha)$ using any numerical methods where $\hat{\phi}_{N}^{h}(\pi,\alpha)$ is given in Eqn. \eqref{eqt:Q_hat IPW}. Update $\alpha$: $\alpha\leftarrow\underset{\alpha\geq 0}{\arg\max}\;\hat{\phi}_{N}^{h}(\pi,\alpha)$.
%\Statex \vspace{-1pt}
%\Statex \vspace{-0.5pt}
%\Statex \vspace{-0.5pt}
%\Statex \vspace{-1.5pt}
\UNTIL $\alpha$ converges
\STATE \textbf{Return}  $\hat{\pi}_{\mathrm{DRO}}^{h}\leftarrow\pi$
\end{algorithmic}
\end{algorithm}
%\begin{algorithm}
%\caption{Distributionally robust policy learning 
%%$\hat{\pi}_{\mathrm{DRO}}^{h}$
%}\label{alg:algorithm policy}
%\textbf{Input} observed dataset $(X_{i},A_{i},Y_{i})_{i=1}^{N}$, $h$. Initialize: $\pi\in\Pi$ and $\alpha\in\mathbb{R}^{+}\cup{0}$.
%%\textbf{Parameter}: Optional list of parameters\\
%%\textbf{Output} $\hat{\pi}_{\mathrm{DRO}}^{h}$.
%\begin{algorithmic}[1] %[1] enables line numbers
%%\State \textbf{Input} observed dataset $(X_{i},A_{i},Y_{i})_{i=1}^{N}$, $h$. Initialize: $\pi\in\Pi$ and $\alpha\in\mathbb{R}^{+}\cup{0}$.
%%\State \textbf{Output} $\hat{\pi}_{\mathrm{DRO}}^{h}$.
%\Repeat
%\State Compute $\hat{W}_{N}^{h}(\pi,\alpha)$ given in Eqn. \eqref{eqt:W_hat IPW and S_hat}.
%%\Statex \vspace{-1pt}
%%\Statex \vspace{-0.5pt}
%%\Statex \vspace{-0.5pt}
%%\Statex \vspace{-1.5pt}
%\State Solve $\underset{\pi\in\Pi}{\min}\;\hat{W}_{N}^{h}(\pi,\alpha)$ using any numerical methods. Update $\pi$: $\pi\leftarrow\underset{\pi\in\Pi}{\arg\min}\;\hat{W}_{N}^{h}(\pi,\alpha)$.
%\State Solve $\underset{\alpha\geq 0}{\max}\;\hat{\phi}_{N}^{h}(\pi,\alpha)$ using any numerical methods where $\hat{\phi}_{N}^{h}(\pi,\alpha)$ is given in Eqn. \eqref{eqt:Q_hat IPW}. Update $\alpha$: $\alpha\leftarrow\underset{\alpha\geq 0}{\arg\max}\;\hat{\phi}_{N}^{h}(\pi,\alpha)$.
%%\Statex \vspace{-1pt}
%%\Statex \vspace{-0.5pt}
%%\Statex \vspace{-0.5pt}
%%\Statex \vspace{-1.5pt}
%\Until $\alpha$ converges
%\State \textbf{Return}  $\hat{\pi}_{\mathrm{DRO}}^{h}\leftarrow\pi$
%\end{algorithmic}
%\end{algorithm}

\subsection{The Statistical Property of \texorpdfstring{$\hat{\pi}_{\mathrm{DRO}}^{h}$}{}}
%An important issue is studying the statistical performance guarantee of $\hat{\pi}_{\mathrm{DRO}}^{h}$. The performance guarantee allows researchers to determine how good the learned policy is. The metric that can be used is the distributionally robust regret of $\hat{\pi}_{\mathrm{DRO}}^{h}$. The regret is indeed defined to minimize the discrepancy between the performance of the optimal policy and the performance of the learned robust policy. Definition \ref{def:robust policy learning regret} summarizes the mathematical formulation of the distributionally robust regret.

An essential aspect of our study is examining the statistical performance guarantee of $\hat{\pi}_{\mathrm{DRO}}^{h}$, which enables researchers to assess the gap between the learned policy $\hat{\pi}_{\mathrm{DRO}}^{h}$ and the optimal distributionally robust policy $\pi_{\mathrm{DRO}}^{\ast}=\underset{\tilde{\pi}\in\Pi}{\max}\;Q_{\mathrm{DRO}}(\tilde{\pi})$. To achieve this, we use the distributionally robust regret defined in Definition \ref{def:robust policy learning regret} as the evaluation metric.

\begin{definition}\label{def:robust policy learning regret}
Let the optimal distributionally robust policy be $\pi_{\mathrm{DRO}}^{\ast}=\underset{\tilde{\pi}\in\Pi}{\arg\max}\;Q_{\mathrm{DRO}}(\tilde{\pi})$. The distributionally robust regret of a policy $\pi\in\Pi$, denoted by $R_{\mathrm{DRO}}(\pi)$, is then defined as
\begin{equation*}
\begin{aligned}\label{eqt:robust policy learning regret}
R_{\mathrm{DRO}}(\pi)&=\underset{\tilde{\pi}\in\Pi}{\max}\;\underset{\mathbb{P}\in\mathcal{U}_{\mathbb{P}_{0}}(\eta)}{\inf}\mathbb{E}_{\mathbb{P}}[Y(\tilde{\pi}(X))]-\underset{\mathbb{P}\in\mathcal{U}_{\mathbb{P}_{0}}(\eta)}{\inf}\mathbb{E}_{\mathbb{P}}[Y(\pi(X))]\\
&=\underset{\tilde{\pi}\in\Pi}{\max}Q_{\mathrm{DRO}}(\tilde{\pi})-Q_{\mathrm{DRO}}(\pi)=Q_{\mathrm{DRO}}(\pi_{\mathrm{DRO}}^{\ast})-Q_{\mathrm{DRO}}(\pi).
\end{aligned}
\end{equation*}
\end{definition}
Before studying $R_{\mathrm{DRO}}(\hat{\pi}_{\mathrm{DRO}}^{h})$, we will now introduce the required notions of the Rademacher complexity and the covering number of a functional class \cite{shalev2014understanding, mohri2018foundations, wainwright2019high}, which are stated in Definition \ref{def:Rademacher}.
\begin{definition}\label{def:Rademacher}
Let $\mathcal{F}$ be a family of real-valued functions $f$ where $f:\mathcal{Z}\rightarrow\mathbb{R}$. Given $Z_{1},\cdots,Z_{N}\in\mathcal{Z}$, the Rademacher complexity of $\mathcal{F}$ is defined as $\mathcal{R}_{N}(\mathcal{F})$ such that
%{\small
\begin{equation*}
\begin{gathered}
\mathcal{R}_{N}(\mathcal{F})=\mathbb{E}_{Z}[\hat{\mathcal{R}}_{N}(\mathcal{F})]=\mathbb{E}_{Z,\sigma}\bigg[\underset{f\in\mathcal{F}}{\sup}\bigg|\frac{1}{N}\underset{i=1}{\overset{N}{\sum}}\sigma_{i}f(Z_{i})\bigg|\bigg],\\
\hat{\mathcal{R}}_{N}(\mathcal{F})\coloneqq\mathbb{E}_{\sigma}\bigg[\underset{f\in\mathcal{F}}{\sup}\bigg|\frac{1}{N}\underset{i=1}{\overset{N}{\sum}}\sigma_{i}f(Z_{i})\bigg|\bigg|Z_{1},\cdots,Z_{N}\bigg].
\end{gathered}
\end{equation*}
%}\noindent
Here, $\sigma_{1},\;\cdots,\sigma_{N}$ are i.i.d. with the distribution $\mathbb{P}\{\sigma_{i}=1\}=\mathbb{P}\{\sigma_{i}=-1\}=\frac{1}{2}$. Additionally, consider a set $\{X_{1},\cdots,X_{N}\}$ in a metric space with metric $\|\cdot\|$. A set $\mathcal{A}_{\{X_{1},\cdots,X_{N}\}}\subset\mathcal{F}$ is said to be a $t$-covering of $\mathcal{F}$ if, for any $f\in\mathcal{F}$, there exists $\tilde{f}\in\mathcal{A}_{\{X_{1},\cdots,X_{N}\}}$ such that $\|(f(X_{1}),\cdots,f(X_{N}))-(\tilde{f}(X_{1}),\cdots,\tilde{f}(X_{N}))\|\leq t$. The size of the smallest $t$-covering, denoted by $\mathfrak{N}\big(t,\mathcal{F}(\{X_{1},\cdots,X_{N}\}),\|\cdot\|\big)$, is the $t$-\textit{covering number}.
\end{definition}
\noindent With Definition \ref{def:Rademacher}, the regret $R_{\mathrm{DRO}}(\pi)$ can be generally upper bounded as the following Theorem \ref{thm:statistical performance normalized}.
\begin{theorem}\label{thm:statistical performance normalized}
Suppose that the kernel function $K(x)$ is bounded where $|K(x)|\leq M_{K}$. Given $\delta>0$, $h>0$, and a policy class $\Pi$, denote 
%{\small
\begin{equation*}
\begin{gathered}
\mathcal{F}_{\Pi}\coloneqq\Bigg\{\frac{K_{h}(\pi(X)-A)}{f_{0}(A|X)}:\pi\in\Pi\Bigg\},\\
\mathcal{F}_{\Pi,x}\coloneqq\Bigg\{\frac{K_{h}(\pi(X)-A)\mathbf{1}_{\{Y(\pi(X))\leq x\}}}{f_{0}(A|X)}:\pi\in\Pi,\;x\in[0,M]\Bigg\}.
\end{gathered}
\end{equation*}
%}\noindent
Then, with probability $1-\delta$, we have
\begin{gather}
\begin{aligned}\label{eqt:statistical performance normalized Rademacher results}
R_{\mathrm{DRO}}(\hat{\pi}_{\mathrm{DRO}}^{h})\leq &\frac{4}{\epsilon}\mathcal{R}_{N}(\mathcal{F}_{\Pi,x})+\frac{4}{\epsilon}\mathcal{R}_{N}(\mathcal{F}_{\Pi})+\frac{4\sqrt{2}M_{K}\sqrt{\ln\big(\frac{2}{\delta}\big)}}{h\epsilon^{2}\sqrt{N}}+O(h^{2}).
\end{aligned}
\end{gather}
\end{theorem}
\noindent The Rademacher complexities in Eqn. \eqref{eqt:statistical performance normalized Rademacher results} can be further bounded using covering numbers (see, for instance, \cite{shalev2014understanding}). Under certain conditions, such as when the square root of the metric entropy (i.e., the logarithm of the covering number) is summable, we can bound $\mathcal{R}_{N}(\mathcal{F}_{\Pi,x})$ and $\mathcal{R}_{N}(\mathcal{F}_{\Pi})$ by the covering number of $\Pi$. This result is presented in detail in Corollary \ref{cor:RDRO bound finite covering}.
\begin{corollary}\label{cor:RDRO bound finite covering}
If the kernel function $K(x)$ is Lipschitz continuous with constant $L_{K}>0$ (i.e., $|K(x)-K(y)|\leq L_{K}|x-y|$) and there exists a finite value $\kappa$ which equals
%{\small
\begin{equation*}
\begin{aligned}
\mathbb{E}\Bigg[\int_{0}^{\frac{2M_{K}h}{L_{K}}}\sqrt{\log\;\mathfrak{N}\big(t,\Pi(\{X_{1},\cdots,X_{N}\}),\|\cdot\|_{\mathcal{L}_{2}(\mathbb{P}_{N})}\big)}dt\Bigg].
\end{aligned}
\end{equation*}
%}\noindent
%{\small
%\begin{equation*}
%\begin{aligned}
%\kappa\coloneqq\mathbb{E}\Bigg[\int_{0}^{\frac{2M_{K}h}{L_{K}}}\sqrt{\log\;\mathfrak{N}\big(t,\Pi(\{X_{1},\cdots,X_{N}\}),\|\cdot\|_{\mathcal{L}_{2}(\mathbb{P}_{N})}\big)}dt\Bigg].
%\end{aligned}
%\end{equation*}
%}
Then, for some constant $\mathcal{K}$, Eqn. \eqref{eqt:statistical performance normalized Rademacher results} becomes
\begin{equation}
\begin{aligned}\label{eqt:statistical performance normalized}
&R_{\mathrm{DRO}}(\hat{\pi}_{\mathrm{DRO}}^{h})\leq \frac{288L_{K}\kappa}{\sqrt{N}h^{2}\epsilon^{2}}+\frac{192M_{K}(\sqrt{\log\mathcal{K}}+2\sqrt{2})}{\sqrt{N}h\epsilon^{2}}+\frac{4M_{K}\sqrt{2\log\big(\frac{2}{\delta}\big)}}{\sqrt{N}h\epsilon^{2}}+O(h^{2}).
\end{aligned}
\end{equation}
%\begin{gather}
%\begin{aligned}\label{eqt:statistical performance normalized}
%&R_{\mathrm{DRO}}(\hat{\pi}_{\mathrm{DRO}}^{h})\leq \frac{288L_{K}\kappa}{\sqrt{N}h^{2}\epsilon^{2}}+\frac{192M_{K}(\sqrt{\log\mathcal{K}}+2\sqrt{2})}{\sqrt{N}h\epsilon^{2}}\\
%&\quad\quad\quad\quad\quad\quad\quad\quad+\frac{4M_{K}\sqrt{2\log\big(\frac{2}{\delta}\big)}}{\sqrt{N}h\epsilon^{2}}+O(h^{2}).
%\end{aligned}
%%\raisetag{50pt}
%\end{gather}
\end{corollary}
%{\color{blue}
%\begin{remark}
%
%\end{remark}
%}
Note that the distributionally robust regret is independent of $\eta$, as it is unaffected by the expectation term in the dual problem. In conjunction with Eqn. \eqref{eqt:optimal bandwidth}, selecting $h=O(N^{-\frac{1}{5}})$ in Corollary \ref{cor:RDRO bound finite covering} ensures consistent learning of the optimal linear policy, as the distributionally robust regret $R_{\mathrm{DRO}}(\hat{\pi}_{\mathrm{DRO}}^{h})$ converges to zero when $N$ tends to infinity.

To conclude this section, we discuss the covering numbers of various policy classes. A common policy class is the linear policy class, defined as $\Pi=\{\pi:\mathcal{X}\rightarrow\mathcal{A}|\pi(X)=w^{\top}X,\;\|w\|_{p}\leq a,\;\|X\|_{q}\leq b,\;w,\;X\in\mathbb{R}^{d}\}$. For instance, when $p=\infty$ and $q=1$, we can demonstrate that
%{\small 
\begin{equation*}
\begin{aligned}
\mathfrak{N}\big(t,\Pi({X_{1},\cdots,X_{N}}),\|\cdot\|_{\mathcal{L}_{\infty}(\mathbb{P}_{N})}\big)\leq \Bigg(\frac{\underset{1\leq i\leq N}{\max}\|X_{i}\|_{1}}{t}+2\Bigg)^{d}.
\end{aligned}
\end{equation*}
%$$\mathfrak{N}\big(t,\Pi({X_{1},\cdots,X_{N}}),|\cdot|{\mathcal{L}{\infty}(\mathbb{P}{N})}\big)\leq \Bigg(\frac{\underset{1\leq i\leq N}{\max}|X{i}|{1}}{t}+2\Bigg)^{d}.$$
%}\noindent
Consequently, $\kappa$ in Eqn. \eqref{eqt:statistical performance normalized} is bounded above by $\sqrt{d}\bigg\{\frac{2\sqrt{2}M_{K}h}{L_{K}}+2\sqrt{\frac{2M_{K}h}{L_{K}}}\mathbb{E}\big[\underset{1\leq i\leq N}{\max}\|X_{i}\|_{1}\big]\bigg\}$, as per its definition. \cite{DBLP:journals/jmlr/Zhang02} provide the covering number of linear policy class for $2\leq p<\infty$.

%Beyond linear policies, we can explore non-linear classes such as neural networks (NNs) and support vector machines (SVMs). For example, policies of shallow NNs can be expressed as linear functions composed with Lipschitz activations, and their covering number can be bounded by the Lipschitz constant and the linear class \citep{DBLP:journals/jmlr/Zhang02, anthony1999neural}. Covering numbers for other classes can be found in sources such as \citep{bartlett2017spectrally}.
We can extend the study from linear policy classes to classes containing non-linear policies such as neural networks or support vector machines (SVMs). For example, shallow neural networks can be represented as linear functions composed with Lipschitz activations. The covering number for the class can be bounded by the Lipschitz constant and the linear class \citep{DBLP:journals/jmlr/Zhang02, anthony1999neural}. Covering numbers for other classes can be found in sources such as \citep{bartlett2017spectrally}.

\section{Experiments}\label{sec:experiment}
In this section, we mainly investigate the robustness of the proposed policy $\pi_{\mathrm{DRO}}^{h}$ against distribution shift. Our analysis includes two parts: simulation and empirical studies. First, in the ``Simulation Experiment" subsection, we compare results under continuous treatments with those under discretized treatments, as well as outcomes with and without robustness. We evaluate these results in a distributionally robust manner to assess the policy's performance under varying conditions. Following this, in the ``Empirical Experiment" subsection, the experiments on Warfarin dataset compare the robustness performance of the robust and nonrobust policies. All experiments are run on a Dell 3640 with an Intel Xeon W-1290P 3.60GHz CPU\footnote{
%Available in the ``code'' file of the Supplementary Material. 
The code for the experiments can be found at {\texttt{https://github.com/cleung87/Distributionally-Robust-\linebreak Policy-Evaluation-and-Learning-for-Continuous-Treatment-with-Observed-Data}}.}.

\subsection{Simulation Experiment}\label{sec:simulation}
\paragraph{Continuous v.s. Discrete.}
We begin by comparing our distributionally robust policy $\hat{\pi}_{\mathrm{DRO}}^{h}$ under continuous treatment with the distributionally robust policy $\hat{\pi}_{\mathrm{DRO}}^{dis\text{-}k}$, where the continuous treatment is discretized using the method proposed in \cite{si2023distributionally} into $k$ bins ($k \in {2, 3, 4}$) based on the discretized strategy in \cite{zhou2017residual}. To enable a fair comparison between these two forms, we consider a simple data generating process with known optimal values. Specifically, we assume: $X \sim \mathrm{Uniform}(0,1)$, $A|X \sim \mathrm{Uniform}(X, X+1)$, $Y=5+X/A+\epsilon$, $\epsilon \sim \mathrm{Uniform}(0,1)$. We define the policy class $\Pi$ as $\{\beta X: 1 \leq \beta \leq 3\}$, and set the ambiguity radius $\eta=0.05$. With these specifications, we compute the optimal distributionally robust value $Q^{*}=\underset{\pi \in \Pi}{\max}\;\underset{\mathbb{P} \in \mathcal{U}_{\mathbb{P}_{0}}(\eta)}{\inf}\;\mathbb{E}_{\mathbb{P}}[Y(\pi(X))]$, which evaluates to $6.41$ using numerical approaches. 
%{\color{red}Should we discuss the distribution shift in this experiment?: Yes. I think it is created by the treatment assignment, i.e.., covariate sfhit. Not sure whether my understanding is right. Could you help finalize this expression?} 
For the bandwidth parameter $h$,
% we follow the notion of selecting the bandwidth given by \cite{kallus2018policy} through a plug-in estimator based on Eqn. \eqref{eqt:optimal bandwidth}.
we follow the approach of selecting the bandwidth as given by \cite{kallus2018policy} using a plug-in estimator based on Eqn. \eqref{eqt:optimal bandwidth}. 
We generate 100 different datasets, each consisting of 2500 training samples and 2500 test samples.

For policy learning on training data, both $\hat{\pi}_{\mathrm{DRO}}^{h}$ and $\hat{\pi}_{\mathrm{DRO}}^{dis\text{-}k}$ are learned using $\eta^{train}=0.05$. We then compute $\hat{Q}^{h}_{\mathrm{DRO}}(\hat{\pi}_{\mathrm{DRO}}^{h})$ and $\hat{Q}^{dis}_{\mathrm{DRO}}(\hat{\pi}_{\mathrm{DRO}}^{dis\text{-}k})$ on the test data and compare the results. For various $\eta^{test}\in\{0.05,0.1,0.2,0.3,0.4\}$, $\hat{Q}^{h}_{\mathrm{DRO}}(\hat{\pi}_{\mathrm{DRO}}^{h})$ is estimated according to Algorithm \ref{alg:algorithm eval}, while $\hat{Q}^{dis}_{\mathrm{DRO}}(\hat{\pi}_{\mathrm{DRO}}^{dis\text{-}k})$ is estimated by solving $\underset{\alpha\geq 0}{\max}\{-\alpha \log \hat{W}_{N}(\hat{\pi}_{\mathrm{DRO}}^{dis\text{-}k},\alpha)-\alpha\eta\}$. Here $\hat{W}_{N}(\pi,\alpha)=\bigg(\underset{i=1}{\overset{N}{\sum}}\frac{\mathbf{1}_{\{\pi(X_{i})=A_{i}\}}e^{-\frac{Y_{i}}{\alpha}}}{\hat{p}_{0}(A_{i}|X_{i})}\bigg)/\bigg(\underset{j=1}{\overset{N}{\sum}}\frac{\mathbf{1}_{\{\pi(X_{j})=A_{j}\}}}{\hat{p}_{0}(A_{j}|X_{j})}\bigg)$ and $\hat{p}_{0}(A|X)$ is the estimated probability of receiving treatment $A$ conditioning on $X$. The results given in Table \ref{tab:cont_vs_dis} indicate that the learned policy $\hat{\pi}_{\mathrm{DRO}}^{h}$ achieves the best robust performance when evaluated using the $\hat{Q}_{\mathrm{DRO}}^{h}$ metric (see the first row of Table \ref{tab:cont_vs_dis}). Furthermore, the mean value $6.24$ exhibits a significantly smaller gap with the optimal distributionally robust value of $6.41$ compared to the discrete-treatment policies evaluated using $\hat{Q}_{\mathrm{DRO}}^{dis}$.
\noindent

\begin{table}[htbp]
%\small 
\centering
%\scalebox{0.95}{
\setlength{\tabcolsep}{0.8mm}{
%\resizebox{\columnwidth}{!}{
\begin{tabular}{cccccc}
\toprule
& \multicolumn{5}{c}{$\eta^{test}$}\\
%\cline{2-6}
 & $0.05$  & $0.1$   & $0.2$   & $0.3$   & $0.4$ \\
\midrule
%%$\hat{Q}_{\mathrm{DRO}}^{h}(\hat{\pi}_{\mathrm{DRO}}^{h})$ 
%& 6.24$\pm$3.2 & 6.19$\pm$3.3 & 6.11$\pm$3.6 & 6.04$\pm$3.8 & 5.99$\pm$4.0 \\
%%$\hat{Q}^{dis}_{\mathrm{DRO}}(\hat{\pi}^{dis\text{-}2}_{\mathrm{DRO}})$
%& 5.88$\pm$1.5 & 5.81$\pm$1.5 & 5.71$\pm$1.5 & 5.64$\pm$1.5 & 5.58$\pm$1.5 \\
%%$\hat{Q}^{dis}_{\mathrm{DRO}}(\hat{\pi}^{dis\text{-}3}_{\mathrm{DRO}})$ 
%& 5.85$\pm$1.2 & 5.79$\pm$1.2 & 5.70$\pm$1.2 & 5.63$\pm$1.2 & 5.58$\pm$1.2 \\
%%$\hat{Q}^{dis}_{\mathrm{DRO}}(\hat{\pi}^{dis\text{-}4}_{\mathrm{DRO}})$ 
%& 5.83$\pm$1.2 & 5.77$\pm$1.2 & 5.68$\pm$1.2 & 5.61$\pm$1.2 & 5.56$\pm$1.2 \\
%%%%%%%%%%%%%%%%%%%%%%%%%%%%%%%%%%%%%%%%%%%%%%%%%%%%%%%%%%%%%%%%%%%%%%%%%%%%%%%%%%%%%%%%%%%%%
%$\hat{Q}_{\mathrm{DRO}}^{h}(\hat{\pi}_{\mathrm{DRO}}^{h})$ 
& 6.24$\pm$0.32 & 6.19$\pm$0.33 & 6.11$\pm$0.36 & 6.04$\pm$0.38 & 5.99$\pm$0.40 \\
%$\hat{Q}^{dis}_{\mathrm{DRO}}(\hat{\pi}^{dis\text{-}2}_{\mathrm{DRO}})$ 
& 5.88$\pm$0.15 & 5.81$\pm$0.15 & 5.71$\pm$0.15 & 5.64$\pm$0.15 & 5.58$\pm$0.15 \\
%$\hat{Q}^{dis}_{\mathrm{DRO}}(\hat{\pi}^{dis\text{-}3}_{\mathrm{DRO}})$ 
& 5.85$\pm$0.12 & 5.79$\pm$0.12 & 5.70$\pm$0.12 & 5.63$\pm$0.12 & 5.58$\pm$0.12 \\
%$\hat{Q}^{dis}_{\mathrm{DRO}}(\hat{\pi}^{dis\text{-}4}_{\mathrm{DRO}})$ 
& 5.83$\pm$0.12 & 5.77$\pm$0.12 & 5.68$\pm$0.12 & 5.61$\pm$0.12 & 5.56$\pm$0.12 \\
\bottomrule
\end{tabular}%
}
%}
\caption{Comparison of robustness performance (continuous v.s. discrete) with $\eta^{train}=0.05$ for policy learning and various $\eta^{test}$ for policy evaluation. When $\eta^{train}=\eta^{test}=0.05$, the optimal distributionally robust value is $Q^{*}=6.41$. The reported Mean $\pm$ Standard Error (the Standard Error is in $\%$) is computed over 100 runs. The first/second/third/fourth row records values produced by $\hat{Q}_{\mathrm{DRO}}^{h}(\hat{\pi}_{\mathrm{DRO}}^{h})$/$\hat{Q}^{dis}_{\mathrm{DRO}}(\hat{\pi}^{dis\text{-}2}_{\mathrm{DRO}})$/$\hat{Q}^{dis}_{\mathrm{DRO}}(\hat{\pi}^{dis\text{-}3}_{\mathrm{DRO}})$/$\hat{Q}^{dis}_{\mathrm{DRO}}(\hat{\pi}^{dis\text{-}4}_{\mathrm{DRO}})$.\label{tab:cont_vs_dis}}
\end{table}%

\paragraph{Robust v.s. Nonrobust.}
We then compare our distributionally robust policy $\hat{\pi}_{\mathrm{DRO}}$ with the non-robust policy $\hat{\pi}_{\mathrm{NRO}} \in \underset{\pi \in \Pi}{\arg \max}\;\hat{W}_{N}^{h}(\pi)$, where $\hat{W}_{N}^{h}(\pi)=\frac{\bar{W}_{N}^{h}(\pi)}{S_{N}^{h}}$ and $\bar{W}_{N}^{h}(\pi)=\frac{1}{N}\underset{i=1}{\overset{N}{\sum}}\frac{K_{h}(\pi(X_{i})-A_{i})}{f_{0}(A_{i}|X_{i})}Y_{i}$, as given in \cite{kallus2018policy}, with $\Pi=\{\zeta^{\top}X: \|\zeta\|_{\infty} \leq 2\}$. We follow \cite{kallus2018policy} to simulate i.i.d. data as follows: $X_k \sim \mathrm{Uniform}(-0.2, 0.2)$ for $k=1$ to 10, {\small $A|X \sim \mathcal{N}(\theta^{\top}X, 0.1) + X_{1} + 2X_{2} - 3X_{3}$}, and {\small $Y=5+\beta_{1}^{\top}X + \beta_{2}^{\top}XA + \beta_3A$}. Here, $\theta,\;\beta_{1},\;\beta_{2}\in\mathbb{R}^{10}$ such that $\theta^{\top}=\beta_{1}^{\top}=\beta_{2}^{\top}=\mathbf{1}^{10}\coloneqq[1,\cdots,1]^{\top}$, $\beta_3=1$. To induce sparsity, we randomly set three dimensions of the coefficients $\beta_{1}^{\top}$ and $\beta_{2}^{\top}$ and two dimensions of $\theta^{\top}$ to zero. For the bandwidth parameter $h$, 
%we follow the notion of selecting the optimal bandwidth given by \cite{kallus2018policy} through a plug-in estimator based on Eqn. \eqref{eqt:optimal bandwidth}. 
we follow the approach of selecting the bandwidth as given by \cite{kallus2018policy} using a plug-in estimator based on Eqn. \eqref{eqt:optimal bandwidth}. 
We repeat the data generating process to create 100 different datasets, each consisting of $N_{train}$ ($N_{train} \in {500, 1000, 1500, 2000, 2500}$) training samples and $N_{test}=2000$ test samples.

For policy learning on the training data, both $\hat{\pi}_{\mathrm{DRO}}^{h}$ and $\hat{\pi}_{\mathrm{NRO}}$ are learned within a linear policy class, and $\hat{\pi}_{\mathrm{DRO}}$ is learned with $\eta^{train}=0.2$, denoted by $\hat{\pi}_{\mathrm{DRO}}^{h,\eta=0.2}$. For policy evaluation on the test data, in addition to the evaluation metric $\hat{Q}_{\mathrm{DRO}}^{h}(\pi)$ (Eqn. \eqref{eqt:Q_hat IPW}), we also introduce another metric $\hat{Q}_{pert}(\pi)$ based on a data perturbation strategy. For each of the total $100$ original datasets, we perturb each original test dataset $(X_{i},A_{i},Y_{i})_{i=1}^{N_{test}}$ to obtain a new dataset $(\tilde{X}_{i},\tilde{A}_{i},\tilde{Y}_{i})_{i=1}^{N_{test}}$ such that the new dataset lies within a KL-ball centred at the original test dataset with a radius $\eta^{test}$, introducing a distribution shift in the new dataset relative to the original test dataset. Then we can evaluate each policy using $\hat{Q}_{pert}(\pi)=\underset{1\leq j \leq 100}{\min}\{\frac{1}{N_{\mathrm{test}}}\sum_{i=1}^{N_{test}}\tilde{Y}_{i}^{(j)}(\pi(\tilde{X}_{i}^{(j)}))\}$.

The results presented in Table \ref{tab:dro_vs_nro_eta} and Table \ref{tab:dro_vs_nro_sample} demonstrate that $\hat{\pi}^{h,\eta=0.2}_{\mathrm{DRO}}$ exhibits superior robustness compared to the non-robust policy $\hat{\pi}_{\mathrm{NRO}}$. Specifically, $\hat{\pi}^{h,\eta=0.2}_{\mathrm{DRO}}$ showcases significantly lower sensitivity to data perturbations than $\hat{\pi}_{\mathrm{NRO}}$, consistently achieving higher reward in most cases. Moreover, in Table \ref{tab:dro_vs_nro_eta}, as the level of data perturbation $\eta^{test}$ increases from $0.05$ to $0.4$, $\hat{\pi}^{h,\eta=0.2}_{\mathrm{DRO}}$ shows more stable performance than $\hat{\pi}_{\mathrm{NRO}}$. Notably, in Table \ref{tab:dro_vs_nro_sample}, even with an increase in training sample size, $\hat{\pi}_{\mathrm{NRO}}$ shows no improvement when faced with a distribution shift $\eta^{test}$. In contrast, $\hat{\pi}^{h,\eta=0.2}_{\mathrm{DRO}}$ demonstrates significant improvement as the number of training samples increases.

\begin{table}[t]
%\small
\centering
%\resizebox{\columnwidth}{!}{
%\scalebox{0.95}{
\setlength{\tabcolsep}{0.7mm}{
\begin{tabular}{cccccc}
\toprule
& \multicolumn{5}{c}{$\eta^{test}$}\\
%\cline{2-6}
& $0.05$  & $0.1$ & $0.2$ & $0.3$ & $0.4$ \\
\midrule
%$\hat{Q}_{\mathrm{DRO}}^{h}(\hat{\pi}^{h,\eta=0.2}_{\mathrm{DRO}})$ 
& 5.66$\pm$12.06 & 5.60$\pm$11.97 & 5.52$\pm$11.85 & 5.45$\pm$11.77 & 5.40$\pm$11.70 \\
%$\hat{Q}_{\mathrm{DRO}}^{h}(\hat{\pi}_{\mathrm{NRO}})$ 
& 5.05$\pm$8.68 & 4.99$\pm$8.60 & 4.91$\pm$8.50 & 4.85$\pm$8.44 & 4.80$\pm$8.39 \\
%$\hat{Q}_{pert}(\hat{\pi}^{h,\eta=0.2}_{\mathrm{DRO}})$ 
& 5.48$\pm$15.46 & 5.47$\pm$15.46 & 5.46$\pm$15.46 & 5.45$\pm$15.47 & 5.44$\pm$15.47 \\
%$\hat{Q}_{pert}(\hat{\pi}_{\mathrm{NRO}})$ 
& 5.02$\pm$10.37 & 5.01$\pm$10.36 & 5.00$\pm$10.34 & 4.99$\pm$10.33 & 4.98$\pm$10.32 \\
\bottomrule
\end{tabular}%
}
%}
\caption{Comparison of robustness performance (robust v.s. nonrobust) with $\eta^{train}=0.2$ and $N_{train}=2000$ for policy learning and various $\eta^{test}$ for policy evaluation. The reported Mean $\pm$ Standard Error (the Standard Error is in $\%$) is computed over 100 runs. The first/second/third/fourth row records values produced by {\small $\hat{Q}_{\mathrm{DRO}}^{h}(\hat{\pi}^{h,\eta=0.2}_{\mathrm{DRO}})$}/{\small$\hat{Q}_{\mathrm{DRO}}^{h}(\hat{\pi}_{\mathrm{NRO}})$}/{\small $\hat{Q}_{pert}(\hat{\pi}^{h,\eta=0.2}_{\mathrm{DRO}})$}/{\small $\hat{Q}_{pert}(\hat{\pi}_{\mathrm{NRO}})$}. \label{tab:dro_vs_nro_eta}}
\end{table}%

\begin{table}[t]
%\small
\centering
%\scalebox{0.93}{
\setlength{\tabcolsep}{0.7mm}{
%\resizebox{\columnwidth}{!}{
\begin{tabular}{cccccc}
\toprule
& \multicolumn{5}{c}{$N_{train}$}\\
%\cline{2-6}
& $500$   & $1000$  & $1500$  & $2000$  & $2500$ \\
\midrule
%$\hat{Q}_{\mathrm{DRO}}^{h}(\hat{\pi}^{h,\eta=0.2}_{\mathrm{DRO}})$ 
& 5.19$\pm$11.45 & 5.32$\pm$11.55 & 5.43$\pm$14.31 & 5.48$\pm$12.04 & 5.52$\pm$11.85 \\
%$\hat{Q}_{\mathrm{DRO}}^{h}(\hat{\pi}_{\mathrm{NRO}})$ 
& 4.85$\pm$8.19 & 4.79$\pm$8.10 & 4.83$\pm$7.95 & 4.84$\pm$7.86 & 4.91$\pm$8.50 \\
%$\hat{Q}_{pert}(\hat{\pi}^{h,\eta=0.2}_{\mathrm{DRO}})$ 
& 4.94$\pm$15.64 & 5.12$\pm$16.70 & 5.19$\pm$15.82 & 5.21$\pm$16.05 & 5.46$\pm$15.46 \\
%$\hat{Q}_{pert}(\hat{\pi}_{\mathrm{NRO}})$ 
& 4.95$\pm$10.46 & 4.99$\pm$10.16 & 5.02$\pm$10.34 & 5.00$\pm$10.35 & 5.00$\pm$10.34 \\
\bottomrule
\end{tabular}%
}
%}
\caption{Comparison of robustness performance (robust vs. nonrobust) for various $N_{train}$. $\eta$ is chosen as $0.2$ for both policy learning and evaluation.
%where $N_{train} \in \{500, 1000, 1500, 2000, 2500\}$ for policy learning and $\eta^{test}=0.2$ for policy evaluation.
The reported Mean $\pm$ Standard Error (the Standard Error is in $\%$) is computed over 100 runs. The first/second/third/fourth row records values due to {\small $\hat{Q}_{\mathrm{DRO}}^{h}(\hat{\pi}^{h,\eta=0.2}_{\mathrm{DRO}})$}/{\small $\hat{Q}_{\mathrm{DRO}}^{h}(\hat{\pi}_{\mathrm{NRO}})$}/{\small $\hat{Q}_{pert}(\hat{\pi}^{h,\eta=0.2}_{\mathrm{DRO}})$}/{\small $\hat{Q}_{pert}(\hat{\pi}_{\mathrm{NRO}})$}.\label{tab:dro_vs_nro_sample}}
\end{table}%
\subsection{Empirical Experiment - The Warfarin Case Study}\label{sec:empirical Warfarin}

\paragraph{Description.} 
%Since there are no available datasets for continuous treatment that can provide all the potential outcomes $Y(a)$ for all $a \in \mathcal{A}$, we follow \cite{kallus2018policy} to conduct a semi-synthetic study using the Warfarin dataset \cite{international2009estimation}. 
%Due to the lack of datasets for continuous treatments that provide all potential outcomes $Y(a)$ for all $a \in \mathcal{A}$, we follow \cite{kallus2018policy} and conduct a semi-synthetic study using the Warfarin dataset \cite{international2009estimation}.
We follow \cite{kallus2018policy} to conduct a semi-synthetic study using the Warfarin dataset \cite{international2009estimation}. The dataset contains 5528 patients' medical records, including personal information (e.g., age, gender, race, height, weight), medical problems (e.g., comorbidities and diabetes), medical medication history (e.g., aspirin, atorvastatin, etc.), and their genotypes. The dataset also provides the suggested treatment dose (therapeutic dose).

\paragraph{Setting.} We employ a random forest regressor on the therapeutic dose and select 41-dimensional covariates based on the feature importance ranking. There are 3306 samples after dropping those with missing values. The observed dataset is generated as follows: $A|X \sim \mathcal{N}(\theta^{\top}X+1, 0.1)$ and $Y=5+\beta_{1}^{\top}X + \beta_{2}^{\top}XA + \epsilon$, where $\beta_{1},\;\beta_2,\;\theta\in\mathbb{R}^{p}$ (with $p=41$ in our setting), $\beta_{1}^{\top}=0.2 \cdot \mathbf{1}^{p}$, $\beta_{2}^{\top}=0.1 \cdot \mathbf{1}^{p}$, $\theta^{\top}=0.1 \cdot \mathbf{1}^{p}$. We also assume $\Pi=\{\zeta^{\top}X:\|\zeta\|_{\infty}\leq 2\}$. 
%Again, we follow the notion of selecting the optimal bandwidth $h$ given by \cite{kallus2018policy} through a plug-in estimator based on Eqn. \eqref{eqt:optimal bandwidth}. 
We again follow the approach of selecting the bandwidth as given by \cite{kallus2018policy} using a plug-in estimator based on Eqn. \eqref{eqt:optimal bandwidth}. 
To create distribution shifts, we split the training and test data based on patients' age information. We select 1983 patients aged 10-69 as the training set and 1323 patients older than 70 as the test set. We repeat this process 1000 times to create a total of 1000 semi-synthetic Warfarin datasets.

\paragraph{Results.} We learn a non-robust policy $\hat{\pi}_{\mathrm{NRO}}$ and robust policies $\hat{\pi}^{h,\eta}_{\mathrm{DRO}}$ for $\eta \in \{0.3, 0.4, 0.5, 0.6, 0.7\}$ on the training set, and we evaluate the six policies on test set based on the sample averaged potential outcome: {\small$\hat{Q}_{mean}(\pi)=\frac{1}{N}\sum_{i=1}^{N}Y_i(\pi(X_i))$}. Consequently, we obtain 1000 values of $\hat{Q}_{\mathcal{M}}(\pi)$ w.r.t. each of the six policies. We then report the mean, standard error, and the $5^{\mathrm{th}}/10^{\mathrm{th}}/15^{\mathrm{th}}/20^{\mathrm{th}}/25^{\mathrm{th}}/30^{\mathrm{th}}$ percentile of the total 1000 values in Table \ref{tab:Qmean}.
\begin{table}[htbp]
%\small
\centering
%\scalebox{1}{
%\resizebox{\columnwidth}{!}{
\setlength{\tabcolsep}{1mm}{
\begin{tabular}{lcccccccccccccc}
\toprule
& \multicolumn{8}{c}{percentile}\\
%\cline{4-9}
& Mean  & SE & $5^{\mathrm{th}}$ & $10^{\mathrm{th}}$ & $15^{\mathrm{th}}$ & $20^{\mathrm{th}}$ & $25^{\mathrm{th}}$ & $30^{\mathrm{th}}$  \\
\midrule
%{\scriptsize $\hat{Q}_{\mathcal{M}}(\hat{\pi}_{\mathrm{NRO}})$}
& 6.377 & 4.5 & 4.058 & 4.525 & 4.887 & 5.114 & 5.350 & 5.605  \\
%{\scriptsize $\hat{Q}_{\mathcal{M}}(\hat{\pi}^{h,\eta=0.3}_{\mathrm{DRO}})$}
& 6.372 & 4.3 & 4.054 & 4.648 & 5.020 & 5.298 & 5.533 & 5.703  \\
%{\scriptsize $\hat{Q}_{\mathcal{M}}(\hat{\pi}^{h,\eta=0.4}_{\mathrm{DRO}})$}
& 6.454 & 4.3 & 4.244 & 4.653 & 5.086 & 5.306 & 5.523 & 5.737  \\
%{\scriptsize $\hat{Q}_{\mathcal{M}}(\hat{\pi}^{h,\eta=0.5}_{\mathrm{DRO}})$}
& 6.409 & 4.5 & 4.112 & 4.552 & 4.884 & 5.212 & 5.449 & 5.671  \\
%{\scriptsize $\hat{Q}_{\mathcal{M}}(\hat{\pi}^{h,\eta=0.6}_{\mathrm{DRO}})$}
& 6.355 & 4.5 & 4.085 & 4.536 & 4.774 & 5.106 & 5.344 & 5.548  \\
%{\scriptsize $\hat{Q}_{\mathcal{M}}(\hat{\pi}^{h,\eta=0.7}_{\mathrm{DRO}})$}
& 6.350 & 4.4 & 4.092 & 4.613 & 4.908 & 5.240 & 5.434 & 5.620  \\
\bottomrule
\end{tabular}%
}
%}
\caption{Comparison of the rewards of robust and nonrobust policies in Warfarin study. The reported result are computed over 1000 runs. Note that the SE in the table represents the Standard Error which is reported in $\%$. The first/second/third/fourth/fifth/sixth row records values produced by $\hat{Q}_{mean}(\hat{\pi}_{\mathrm{NRO}})$/$\hat{Q}_{mean}(\hat{\pi}^{h,\eta=0.3}_{\mathrm{DRO}})$/$\hat{Q}_{mean}(\hat{\pi}^{h,\eta=0.4}_{\mathrm{DRO}})$/ $\hat{Q}_{mean}(\hat{\pi}^{h,\eta=0.5}_{\mathrm{DRO}})$/$\hat{Q}_{mean}(\hat{\pi}^{h,\eta=0.6}_{\mathrm{DRO}})$/$\hat{Q}_{mean}(\hat{\pi}^{h,\eta=0.7}_{\mathrm{DRO}})$.\label{tab:Qmean}}
\end{table}%

Table \ref{tab:Qmean} demonstrates four important insights: (1) {\small $\hat{Q}_{mean}(\hat{\pi}^{h,\eta=0.3}_{\text{DRO}})$} exhibits a comparable mean value to {\small $\hat{Q}_{mean}(\hat{\pi}_{\text{NRO}})$}. (2) The expected reward initially increases, reaching the optimal (e.g., when $\eta^{train} \in \{0.4, 0.5\}$), and then decreases with larger $\eta$. 
This trend is reasonable, as a very small $\eta$ neglects the robustness effect and results in a relatively aggressive policy, while a very large $\eta$ results in an overly conservative policy%This is reasonable because a too small $\eta$ neglects the robustness effect and give a relatively aggressive policy, while a too large $\eta$ results in an overly conservative policy
\footnote{Determining the optimal $\eta$ is beyond the scope of this study and is left for future work. Some useful guidances are provided. For example, \cite{pardo2018statistical} show that the distance $\mathbb{D}(\cdot|\cdot)$ is asymptotically $\chi^{2}$ distributed which enables us to select proper $\eta$.}. (3) The standard error of all robust policies is smaller than that of the non-robust policy. (4) From the percentile results, most robust policies outperform the non-robust policy in ``bad'' scenarios, underscoring the robustness of the proposed $\hat{\pi}_{\text{DRO}}^{h}$.

\section{Conclusion, Limitation and Future Work}\label{sec:conclusion}
%\paragraph{Conclusion.} We investigate offline policy evaluation and learning under continuous treatment in the distributionally robust optimization (DRO) setting. We propose an estimator ($\hat{Q}_{\mathrm{DRO}}^{h}(\pi)$) for offline policy evaluation and obtain a distributionally robust policy ($\hat{\pi}_{\mathrm{DRO}}^{h}$) based on $\hat{Q}_{\mathrm{DRO}}^{h}(\pi)$. We analyze the asymptotic distribution and statistical guarantee of $\hat{Q}_{\mathrm{DRO}}^{h}(\pi)$ and $\hat{\pi}_{\mathrm{DRO}}^{h}$. Experiments include a simulation study and an empirical study on the Warfarin dataset demonstrate the superior performance of our approach.
\paragraph{Conclusion.} We investigate offline policy evaluation and learning under continuous treatment in the distributionally robust optimization (DRO) setting. We propose an estimator, $\hat{Q}_{\mathrm{DRO}}^{h}(\pi)$, for offline policy evaluation and obtain a distributionally robust policy, $\hat{\pi}_{\mathrm{DRO}}^{h}$, based on $\hat{Q}_{\mathrm{DRO}}^{h}(\pi)$. We study the asymptotic distribution and the statistical guarantee of $\hat{Q}_{\mathrm{DRO}}^{h}(\pi)$ and $\hat{\pi}_{\mathrm{DRO}}^{h}$. Experimental results demonstrate the superior performance of our approach.

\paragraph{Future Work and Limitations.} 
	The proposed framework can be applied in various fields where distribution shifts occur in the context of continuous-valued treatments. For instance, doctors may seek to determine a robust dosage that minimizes potential disease risks for target patients, while policymakers might aim to establish a robust credit-increasing strategy that maximizes potential consumption for target customers. Thus, applying our framework to real-world scenarios represents a significant next step. Additionally, several potential technical investigations can be further explored. First, selecting the divergence measures and determining the ambiguity radius for the distributional ambiguity set pose significant challenges in both the Operations Research and Machine Learning communities. Future research would benefit from establishing statistical guarantees for other metrics (e.g., Wasserstein metric) and offering guidance on setting the radius for policy evaluation and learning. Second, exploring methods for the generalized propensity score when it is unknown would be interesting. Third, expanding our framework to include the doubly robust estimator might improve the convergence rate of policy learning. 
Lastly, strictly limiting the policy class to linear functions may fail to capture complex relationships between covariates and treatment, leading to suboptimal results. Considering broader policy classes (e.g., nonlinear policy classes with infinite VC dimensions) is therefore essential.
%Lastly, as strictly limiting the policy class to linear functions may lead to suboptimal policy learning, it is necessary to consider broader policy classes (e.g., nonlinear policy classes with infinite VC dimensions) beyond linear policy class.

\section{Acknowledgements}
Qi WU acknowledges the support from The CityU-JD Digits Joint Laboratory in Financial Technology and  Engineering and The Hong Kong Research Grants Council [General Research Fund 11219420/9043008]. The work described in this paper was partially supported by the InnoHK initiative, the  Government of the HKSAR, and the Laboratory for AI-Powered Financial Technologies.

\bibliographystyle{apalike}
\bibliography{template}

\clearpage
\appendix

\section{Proofs}\label{sec:proof}
Before presenting the proofs of our main results in the paper, we present the necessary auxiliary results in Section `Auxiliary Results'. The proofs of the main results are given in Section `Proofs of results in the main paper'.
\subsubsection{Auxiliary Results}\label{sec:Auxiliary Results}
We present the auxiliary results that are required when proving the results in the main paper.
\paragraph{Auxiliary Result 1 - Convergence Studies of $S_{N}^{h}$}\label{sec:auxiliary result 1}
\begin{proposition}\ \label{result:convergence S_N_h}
Given that $S_{N}^{h}=\frac{1}{N}\underset{i=1}{\overset{N}{\sum}}\frac{K_{h}(\pi(X_{i})-A_{i})}{f_{0}(A_{i}|X_{i})}$. We have:
\begin{enumerate}
\item if $N\rightarrow \infty$, $h\rightarrow 0$, $\frac{1}{Nh}\rightarrow 0$, then $S_{N}^{h}\overset{p}{\rightarrow} 1$;\label{result:convergence probability h}
\item if $N\rightarrow\infty$ and $h\rightarrow 0$ such that $Nh^{2}\rightarrow\infty$, then $S_{N}^{h}\;\overset{a.s.}{\longrightarrow}\;1$.\label{lemma:useful asy result}
\end{enumerate}
\end{proposition}

\begin{proof}[\textit{\textbf{Proof of Claim \ref{result:convergence S_N_h}.}}] Note that $\mathbb{P}\{|S_{N}^{h}-1|\geq\gamma\}\leq \frac{\mathbb{E}[|S_{N}^{h}-1|^{2}]}{\gamma^{2}}$. We consider the term $\mathbb{E}[|S_{N}^{h}-1|^{2}]$:
\begin{align}
&\mathbb{E}[|S_{N}^{h}-1|^{2}]=\mathbb{E}\bigg[\bigg|\frac{1}{N}\underset{i=1}{\overset{N}{\sum}}\bigg(\frac{K_{h}(\pi(X_{i})-A_{i})}{f_{0}(A_{i}|X_{i})}-1\bigg)\bigg|^{2}\bigg]\nonumber\\
&=\frac{1}{N^{2}}\underset{i=1}{\overset{N}{\sum}}\mathbb{E}\bigg[\bigg(\frac{K_{h}(\pi(X_{i})-A_{i})}{f_{0}(A_{i}|X_{i})}-1\bigg)^{2}\bigg]\\
&\quad+\frac{1}{N^2}\underset{\substack{i,j=1\\i\neq j}}{\overset{N}{\sum}}\mathbb{E}\bigg[\bigg(\frac{K_{h}(\pi(X_{i})-A_{i})}{f_{0}(A_{i}|X_{i})}-1\bigg)\times\bigg(\frac{K_{h}(\pi(X_{j})-A_{j})}{f_{0}(A_{j}|X_{j})}-1\bigg)\bigg]\nonumber\\
%=&\frac{1}{N^{2}}\underset{i=1}{\overset{N}{\sum}}\mathbb{E}\bigg[\bigg(\frac{K_{h}(\pi(X_{i})-A_{i})}{f_{0}(A_{i}|X_{i})}-1\bigg)^{2}\bigg]+\frac{N(N-1)}{N^2}O(h^{4})\nonumber\\
&=\frac{1}{N}\mathbb{E}\bigg[\bigg(\frac{K_{h}(\pi(X)-A)}{f_{0}(A|X)}-1\bigg)^{2}\bigg]+\frac{N(N-1)}{N^2}O(h^{4}).\label{eqt:normalizing factor study}
\end{align}
Note that
%{\small
\begin{equation*}
\begin{aligned}
&\mathbb{E}\bigg[\bigg(\frac{K_{h}(\pi(X)-A)}{f_{0}(A|X)}-1\bigg)^{2}\bigg]=\mathbb{E}\bigg[\frac{K_{h}^{2}(\pi(X)-A)}{f_{0}^{2}(A|X)}\bigg]-2\mathbb{E}\bigg[\frac{K_{h}(\pi(X)-A)}{f_{0}(A|X)}\bigg]+1.
\end{aligned}
\end{equation*}
%}\noindent
Furthermore, we have
{\small
\begin{equation*}
\begin{aligned}
&\mathbb{E}\bigg[\frac{K_{h}^{2}(\pi(X)-A)}{f_{0}^{2}(A|X)}\bigg]=\mathbb{E}\bigg[\mathbb{E}\bigg[\frac{K_{h}^{2}(\pi(X)-A)}{f_{0}^{2}(A|X)}\bigg|X\bigg]\bigg]=\mathbb{E}\bigg[\int K_{h}^{2}(\pi(X)-a)\tilde{g}_{0}(a,y,X)dyda\bigg]\\
=&\frac{1}{h}\mathbb{E}\bigg[\int K^{2}(u)\bigg\{\tilde{g}_{0}(\pi(X),y,X)+\partial_{a}\tilde{g}_{0}(\pi(X),y,X)uh+\frac{\partial_{aa}^{2}\tilde{g}_{0}(\pi(X),y,X)}{2}u^{2}h^{2}+O_{P}(h^{3})\bigg\}dydu\bigg]\\
=&\frac{1}{h}\mathbb{E}\left[\bigg(\int K^{2}(u)du\bigg)\bigg(\int\tilde{g}_{0}(\pi(X),y,X)dy\bigg)+\bigg(\int u^{2}K^{2}(u)du\bigg)\bigg(\frac{\int\partial_{aa}^{2}\tilde{g}_{0}(\pi(X),y,X)dy}{2}\bigg)h^{2}+O_{P}(h^{3})\right]\\
=&\frac{1}{h}\bigg(\int K^{2}(u)du\bigg)\mathbb{E}\bigg[\bigg(\int\tilde{g}_{0}(\pi(X),y,X)dy\bigg)\bigg]+\frac{1}{h}\bigg(\int u^{2}K^{2}(u)du\bigg)\mathbb{E}\bigg[\bigg(\frac{\int\partial_{aa}^{2}\tilde{g}_{0}(\pi(X),y,X)dy}{2}\bigg)\bigg]h^{2}+O(h^{2})
\end{aligned}
\end{equation*}
}\noindent
and
%{\small
\begin{equation*}
\begin{aligned}
&\mathbb{E}\bigg[\frac{K_{h}(\pi(X)-A)}{f_{0}(A|X)}\bigg]=\mathbb{E}\bigg[\mathbb{E}\bigg[\frac{K_{h}(\pi(X)-A)}{f_{0}(A|X)}\bigg|X\bigg]\bigg]=\mathbb{E}\bigg[\int K_{h}(\pi(X)-a)g_{0}(a,y,X)dyda\bigg]\\
&=1+\bigg(\int u^{2}K(u)du\bigg)\mathbb{E}\bigg[\bigg(\frac{\int\partial_{aa}^{2}g_{0}(\pi(X),y,X)dy}{2}\bigg)\bigg]h^{2}+O(h^{3}).
\end{aligned}
\end{equation*}
%}\noindent
Here, {\small $\tilde{g}_{0}(a,y,X)=\frac{f_{0}(y|a,X)}{f_{0}(a|X)}$} and {\small $g_{0}(a,y,X)=f_{0}(y|a,X)$} such that $\partial_{a}\tilde{g}_{0}(a,y,X)$, $\partial_{a}g_{0}(a,y,X)$, $\partial_{aa}^{2}\tilde{g}_{0}(a,y,X)$, and $\partial_{aa}^{2}\tilde{g}_{0}(a,y,X)$ are the first-order derivative of $\tilde{g}_{0}(a,y,X)$, the first-order derivative of $g_{0}(a,y,X)$, the second-order derivative of $\tilde{g}_{0}(a,y,X)$, and the second-order derivative of $g_{0}(a,y,X)$ respectively. Substituting the above two results into Eqn. \eqref{eqt:normalizing factor study}, we conclude that
%{\small
\begin{equation*}
\begin{aligned}
&\mathbb{P}\{|S_{N}^{h}-1|\geq\gamma\}\leq \frac{\mathbb{E}[|S_{N}^{h}-1|^{2}]}{\gamma^{2}}\\
\leq&\frac{1}{\gamma^{2}Nh}\bigg(\int K^{2}(u)du\bigg)\mathbb{E}\bigg[\bigg(\int\tilde{g}_{0}(\pi(X),y,X)dy\bigg)\bigg]\\
&+\frac{h^{2}}{\gamma^{2}Nh}\bigg(\int u^{2}K^{2}(u)du\bigg)\mathbb{E}\bigg[\bigg(\frac{\int\partial_{aa}^{2}\tilde{g}_{0}(\pi(X),y,X)dy}{2}\bigg)\bigg]\\
&+O\bigg(\frac{h^{2}}{N\gamma^{2}}\bigg)-\frac{2}{N\gamma^{2}}\mathbb{E}\bigg[\bigg(\int g_{0}(\pi(X),y,X)dy\bigg)\bigg]\\
&-\frac{2}{N\gamma^{2}}\bigg(\int u^{2}K(u)du\bigg)\mathbb{E}\bigg[\bigg(\frac{\int\partial_{aa}^{2}g_{0}(\pi(X),y,X)dy}{2}\bigg)\bigg]h^{2}\\
&+O\bigg(\frac{h^{3}}{N\gamma^{2}}\bigg)+\frac{1}{N\gamma^{2}}+\frac{1}{\gamma^{2}}\bigg(1-\frac{1}{N}\bigg)O(h^{4})\rightarrow 0
\end{aligned}
\end{equation*}
%}\noindent
according to the given conditions.

\textit{\textbf{Proof of Claim \ref{lemma:useful asy result}.}} We prove the assertion using Hoeffding's Inequality. Indeed, for each $h$, since $0\leq\frac{K_{h}(\pi(X)-A)}{f_{0}(A|X)}\leq\frac{M_{K}}{h\epsilon}$, by Hoeffding's Inequality, we then obtain
%{\small
\begin{equation}
\begin{aligned}\label{eqt:density limit}
&\mathbb{P}\bigg\{\bigg|\frac{1}{N}\underset{i=1}{\overset{N}{\sum}}\frac{K_{h}(\pi(X_{i})-A_{i})}{f_{0}(A_{i}|X_{i})}-\mathbb{E}\bigg[\frac{K_{h}(\pi(X)-A)}{f_{0}(A|X)}\bigg]\bigg|\geq\gamma\bigg\}\leq 2\exp\bigg(-\frac{2Nh^{2}\epsilon^{2}\gamma^{2}}{M_{K}^{2}}\bigg).
\end{aligned}
\end{equation}
%}\noindent
Usual derivations show that $\mathbb{E}\bigg[\frac{K_{h}(\pi(X)-A)}{f_{0}(A|X)}\bigg]=1+o(h^{2})$. Thus, with probability at least $1-\delta$, for $N\geq \frac{M_{K}^{2}}{2h^{2}\epsilon^{2}\gamma^{2}}\log\big(\frac{2}{\delta}\big)$, we have
%{\small
\begin{equation}
\begin{aligned}\label{eqt:density neglect result}
&1+o(h^{2})-\gamma\leq\frac{1}{N}\underset{i=1}{\overset{N}{\sum}}\frac{K_{h}(\pi(X_{i})-A_{i})}{f_{0}(A_{i}|X_{i})}\leq 1+o(h^{2})+\gamma \\
&\text{or\quad$\Bigg(\bigg|\frac{1}{N}\underset{i=1}{\overset{N}{\sum}}\frac{K_{h}(\pi(X_{i})-A_{i})}{f_{0}(A_{i}|X_{i})}-1\bigg|\leq o(h^{2})+\gamma\Bigg)$}.
\end{aligned}
\end{equation}
%}\noindent
Taking the infinite sum on both sides of Eqn. \eqref{eqt:density limit}, we also have
{\small
\begin{equation*}
\begin{aligned}
&\underset{N=1}{\overset{\infty}{\sum}}\mathbb{P}\bigg\{\bigg|\frac{1}{N}\underset{i=1}{\overset{N}{\sum}}\frac{K_{h}(\pi(X_{i})-A_{i})}{f_{0}(A_{i}|X_{i})}-\mathbb{E}\bigg[\frac{K_{h}(\pi(X)-A)}{f_{0}(A|X)}\bigg]\bigg|\geq\gamma\bigg\}\leq 2\underset{N=1}{\overset{\infty}{\sum}}e^{-\frac{2Nh^{2}\epsilon^{2}\gamma^{2}}{M_{K}^{2}}}=2\frac{e^{-\frac{2\epsilon^{2}h^{2}\gamma^{2}}{M_{K}^{2}}}}{1-e^{-\frac{2\epsilon^{2}h^{2}\gamma^{2}}{M_{K}^{2}}}}<\infty.
\end{aligned}
\end{equation*}
}\noindent
Hence, we can conclude that 
%{\small
\begin{equation*}
\begin{aligned}
\frac{1}{N}\underset{i=1}{\overset{N}{\sum}}\frac{K_{h}(\pi(X_{i})-A_{i})}{f_{0}(A_{i}|X_{i})}-\mathbb{E}\bigg[\frac{K_{h}(\pi(X)-A)}{f_{0}(A|X)}\bigg]\;\overset{a.s.}{\rightarrow}\;0.
\end{aligned}
\end{equation*}
%}\noindent
When $h\rightarrow 0$, we have
%{\small
\begin{equation*}
\begin{aligned}
\frac{1}{N}\underset{i=1}{\overset{N}{\sum}}\frac{K_{h}(\pi(X_{i})-A_{i})}{f_{0}(A_{i}|X_{i})}\;\overset{a.s.}{\rightarrow}\;1.
\end{aligned}
\end{equation*}
%}\noindent
\end{proof}
\paragraph{Auxiliary Result 2 - Optimal Solutions Studies of $Q_{\mathrm{DRO}}(\pi)$ and $Q_{\mathrm{DRO}}^{h}(\pi)$}\label{sec:auxiliary Result 2}\ \\
Define 
\begin{equation*}
\begin{aligned}
Q_{\mathrm{DRO}}(\pi)&=-\underset{\alpha\geq 0}{\min}\Bigg\{\alpha\log\mathbb{E}\bigg[e^{\frac{-Y(\pi(X))}{\alpha}}\bigg]+\alpha\eta\Bigg\}\quad \text{and}\\
Q_{\mathrm{DRO}}^{h}(\pi)&=-\underset{\alpha\geq 0}{\sup}\Bigg\{\alpha\log\mathbb{E}\Bigg[\frac{e^{-\frac{Y}{\alpha}}K_{h}(\pi(X)-A)}{f_{0}(A|X)}\Bigg]+\alpha\eta\Bigg\}.
\end{aligned}
\end{equation*} 
We therefore analyse the following four functions of $\alpha$:
%{\small
\begin{equation*}
\begin{gathered}
\tilde{\phi}_{0}(\alpha)=\log\mathbb{E}[e^{-\alpha Y}],\quad\bar{\phi}_{0}(\alpha)=\log\mathbb{E}[Ze^{-\alpha Y}],\\
\tilde{\phi}_{1}(\alpha)=\alpha\log\mathbb{E}[e^{-\frac{Y}{\alpha}}]+\alpha\eta,\quad\bar{\phi}_{1}(\alpha)=\alpha\log\mathbb{E}[Ze^{-\frac{Y}{\alpha}}]+\alpha\eta,\\
\end{gathered}
\end{equation*}
%}\noindent
where $Z>0$ and $Z=\frac{K_{h}(\pi(X)-A)}{f_{0}(A|X)}$. First, we study the convexity of $\tilde{\phi}_{0}(\alpha)$ and $\bar{\phi}_{0}(\alpha)$. The corresponding results are summarized in Proposition \ref{lemma:convexity lemma}.
\begin{proposition}\label{lemma:convexity lemma}
Suppose that $\mathbb{V}(Y)>0$. Then we have
\begin{enumerate}
\item the functions $\tilde{\phi}_{0}(\alpha)$ and $\tilde{\phi}_{1}(\alpha)$ are convex functions;\label{lemma:convexity study}
\item the functions $\bar{\phi}_{0}(\alpha)$ and $\bar{\phi}_{1}(\alpha)$ are convex functions.\label{lemma:final convexity result}
\end{enumerate} 
\end{proposition}
\begin{proof}
We only prove Claim \ref{lemma:final convexity result} of Proposition \ref{lemma:convexity lemma} since Claim \ref{lemma:convexity study} follows immediately after setting $Z=1$ in the proofs. For $\lambda\in[0,1]$, $\alpha$ and $\bar{\alpha}$ are two arbitrary values where $\alpha\neq \bar{\alpha}$, we have
\begin{equation*}
\begin{aligned}
&\bar{\phi}_{0}(\lambda\alpha+(1-\lambda)\bar{\alpha})=\log\mathbb{E}[Z\exp(-(\lambda\alpha+(1-\lambda)\bar{\alpha}) Y)]\\
=&\log\mathbb{E}[Z\exp(-(\lambda\alpha+(1-\lambda)\bar{\alpha}) Y)]\\
=&\log\mathbb{E}[Z^{\lambda}Z^{1-\lambda}\exp(-\lambda\alpha Y)\exp(-(1-\lambda)\bar{\alpha} Y)]\\
\overset{\star}{\leq}&\log(\mathbb{E}[(Z^{\lambda}e^{-\lambda\alpha Y})^{\frac{1}{\lambda}}])^{\lambda}(\mathbb{E}[(Z^{1-\lambda}e^{-(1-\lambda)\bar{\alpha} Y})^{\frac{1}{1-\lambda}}])^{1-\lambda}\\
=&\lambda\log(\mathbb{E}[Ze^{-\alpha Y}])+(1-\lambda)\log\mathbb{E}[Ze^{-\bar{\alpha} Y}]\\
=&\lambda\bar{\phi}_{0}(\alpha)+(1-\lambda)\bar{\phi}_{0}(\bar{\alpha}).
\end{aligned}
\end{equation*}
Again, $\star$ is due to H\"older's inequality, and the equality in $\overset{\star}{\leq}$ holds if and only if $Ze^{-\alpha Y}=k Ze^{-\bar{\alpha}Y}$ for some constant $k$. Since $\mathbb{V}(Y)>0$, we conclude that $\bar{\phi}_{0}(\alpha)$ is a strictly convex function.

\noindent Next, note that $\bar{\phi}_{1}(\alpha)=\alpha\bar{\phi}_{0}\big(\frac{1}{\alpha}\big)+\alpha\eta$. Since $\bar{\phi}_{0}(\alpha)$ is strictly convex function, we have $\alpha\bar{\phi}_{0}\big(\frac{1}{\alpha}\big)$ is strictly convex. Together with the fact that $\alpha\eta$ is convex, we conclude that $\bar{\phi}_{1}(\alpha)$ is strictly convex since it equals $\alpha\bar{\phi}_{0}\big(\frac{1}{\alpha}\big)+\alpha\eta$.
\end{proof}

\begin{proposition}\ \label{lemma:finite optimal study}
\begin{enumerate}
\item The optimal solution of 
\begin{equation*}
\begin{aligned}
\underset{\alpha\geq 0}{\min}\;\tilde{\phi}_{1}(\alpha)=\underset{\alpha\geq 0}{\min}\{\alpha\log\mathbb{E}[e^{-\frac{Y}{\alpha}}]+\alpha\eta\}
\end{aligned}
\end{equation*}
is finite; \label{lemma:prop of fun phi}
\item Denote $f_{0}(y|a,x)$ as the conditional density function of the variable $Y$ conditioning on $A$ and $X$. Suppose that $\partial_{a}^{i}f_{0}(y|a,X)$ is bounded uniformly for any $(y,a,X)$ where $0\leq i\leq N+1$. The optimal solution of $\underset{\alpha\geq 0}{\min}\;\bar{\phi}_{1}(\alpha)=\underset{\alpha\geq 0}{\min}\bigg\{\alpha\log\mathbb{E}\bigg[\frac{K_{h}(\pi(X)-A)}{f_{0}(A|X)}e^{-\frac{Y}{\alpha}}\bigg]+\alpha\eta\bigg\}$ is finite. \label{lemma:prop of fun phi2}
\end{enumerate}
\end{proposition}

\begin{proof}
We prove Claim \ref{lemma:prop of fun phi2} first.

\textit{\textbf{Proof of Claim \ref{lemma:prop of fun phi2}}} We now consider the asymptotic properties of $\frac{\partial \bar{\phi}_{1}(\alpha)}{\partial \alpha}$. Note that
\begin{equation*}
\begin{aligned}
\frac{\partial \bar{\phi}_{1}(\alpha)}{\partial \alpha}&=\eta+\log\mathbb{E}[Ze^{-\frac{Y}{\alpha}}]+\frac{\mathbb{E}\big[\frac{ZY}{\alpha}e^{-\frac{Y}{\alpha}}\big]}{\mathbb{E}[Ze^{-\frac{Y}{\alpha}}]}.
\end{aligned}
\end{equation*}
When $\alpha\rightarrow \infty$, since $Y$ is bounded, $\log\mathbb{E}[Ze^{-\frac{Y}{\alpha}}]\rightarrow \log\mathbb{E}[Z]$. Further, we can also show that $\frac{\mathbb{E}\big[\frac{ZY}{\alpha}e^{-\frac{Y}{\alpha}}\big]}{\mathbb{E}[Ze^{-\frac{Y}{\alpha}}]}\rightarrow 0$. Hence, $\frac{\partial \bar{\phi}_{1}(\alpha)}{\partial \alpha}>0$. We then study the case when $\alpha\rightarrow 0$. First, denote $f_{Y}(\cdot)$ be the density of variable $Y$. Since $0 \leq Y\leq M$ and it is a continuous variable, $f_{Y}(\cdot)$ is continuous on a compact interval $[0,M]$ such that $\underset{x\in[0,M]}{\max}f_{Y}(x)$ and $\underset{x\in[0,M]}{\min}f_{Y}(x)$ are finite. Denote $\bar{b}=\underset{x\in[0,M]}{\max}f_{Y}(x)$ and $\underline{b}=\underset{x\in[0,M]}{\min}f_{Y}(x)$. We have
%{\small
\begin{equation*}
\begin{aligned}
&\mathbb{E}\big[Z e^{-\frac{Y}{\alpha}}\big]\lesssim\mathbb{E}\big[e^{-\frac{Y}{\alpha}}\big]=\int_{0}^{M} e^{-\frac{y}{\alpha}}f(y)dy\leq \bar{b}\int_{0}^{M} e^{-\frac{y}{\alpha}}dy\leq \bar{b}\int_{0}^{\infty} e^{-\frac{y}{\alpha}}dy=\bar{b}\alpha\\
\Rightarrow &\log\mathbb{E}\big[Z e^{-\frac{Y}{\alpha}}\big]\lesssim\log\bar{b}\alpha\Rightarrow \underset{\alpha\rightarrow 0}{\lim\sup}\log\mathbb{E}\big[Ze^{-\frac{Y}{\alpha}}\big]=-\infty.
\end{aligned}
\end{equation*}
%}\noindent
Second, since $Z=\frac{K_{h}(\pi(X)-A)}{f_{0}(A|X)}$ and $\partial_{a}^{i}f_{0}(y|a,X)$ is bounded uniformly for any $(y,a,X)$ where $0\leq i\leq N+1$ (i.e., $\underline{b}^{i}\leq \underset{y,a,X}{\inf}|\partial_{a}^{i}f_{0}(y|a,X)|\leq \partial_{a}^{i}f_{0}(y|a,X)\leq \underset{y,a,X}{\sup}|\partial_{a}^{i}f_{0}(y|a,X)|\leq \bar{b}^{i}$), we have
%{\small
\begin{equation*}
\begin{aligned}
\mathbb{E}\bigg[\frac{ZY}{\alpha}e^{-\frac{Y}{\alpha}}\bigg]&=\underset{i=0}{\overset{N}{\sum}}\frac{\mathbb{E}\bigg[\int_{0}^{M}\frac{y}{\alpha}e^{-\frac{y}{\alpha}}\partial_{a}^{i}f_{0}(y|\pi(X),X)dy\bigg]}{i!}\bigg(\int_{\mathbb{R}} u^{i}K(u)du\bigg)h^{i}\\
&\quad+\frac{\mathbb{E}\bigg[\int_{0}^{M}\frac{y}{\alpha}e^{-\frac{y}{\alpha}}\partial_{a}^{N+1}f_{0}(y|\theta(X),X)dy\bigg]}{(N+1)!}\bigg(\int_{\mathbb{R}} u^{N+1}K(u)du\bigg)h^{N+1}\\
&\quad\leq\underset{i=0}{\overset{N+1}{\sum}}\frac{1}{i!}\bigg(\int_{0}^{M}\frac{y}{\alpha}e^{-\frac{y}{\alpha}}\bar{b}^{i}dy\bigg)\bigg(\int_{\mathbb{R}} u^{i}K(u)du\bigg)h^{i}\\
&\quad\leq\underset{i=0}{\overset{N+1}{\sum}}\frac{\bar{b}^{i}\bigg(\int_{\mathbb{R}} u^{i}K(u)du\bigg)h^{i}}{i!}\bigg(\int_{0}^{\infty}\frac{y}{\alpha}e^{-\frac{y}{\alpha}}dy\bigg)\leq\underset{i=0}{\overset{N+1}{\sum}}\frac{\bar{b}^{i}\bigg(\int_{\mathbb{R}} u^{i}K(u)du\bigg)h^{i}}{i!}\alpha.
\end{aligned}
\end{equation*}
%}\noindent
Similarly, we can write
%{\small
\begin{equation*}
\begin{aligned}
&\mathbb{E}\bigg[Ze^{-\frac{Y}{\alpha}}\bigg]=\underset{i=0}{\overset{N+1}{\sum}}\frac{1}{i!}\mathbb{E}\bigg[\int_{0}^{M}e^{-\frac{y}{\alpha}}\partial_{a}^{i}f_{0}(y|\tilde{\theta}(X),X)dy\bigg]\bigg(\int_{\mathbb{R}} u^{i}K(u)du\bigg)h^{i},
\end{aligned}
\end{equation*}
%}\noindent
where $\tilde{\theta}(X)=\begin{cases}
\pi(X)&\text{if $0\leq i\leq N$}\\
\theta(X)&\text{if $i=N+1$}
\end{cases}$. Finally, we have
%{\small
\begin{equation*}
\begin{aligned}
\frac{\mathbb{E}\big[\frac{ZY}{\alpha}e^{-\frac{Y}{\alpha}}\big]}{\mathbb{E}\big[Ze^{-\frac{Y}{\alpha}}\big]}&\leq\frac{\underset{i=0}{\overset{N+1}{\sum}}\frac{\bar{b}^{i}\bigg(\int_{\mathbb{R}} u^{i}K(u)du\bigg)h^{i}}{i!}\alpha}{\underset{i=0}{\overset{N+1}{\sum}}\frac{\bigg(\int_{\mathbb{R}} u^{i}K(u)du\bigg)h^{i}}{i!}\mathbb{E}\big[\int_{0}^{M}\frac{y}{\alpha}e^{-\frac{y}{\alpha}}\partial_{a}^{i}f_{0}(y|\tilde{\theta}(X),X)dy\big]}\\
&=\frac{\underset{i=0}{\overset{N+1}{\sum}}\frac{\bar{b}^{i}\bigg(\int_{\mathbb{R}} u^{i}K(u)du\bigg)h^{i}}{i!}}{\underset{i=0}{\overset{N+1}{\sum}}\frac{\bigg(\int_{\mathbb{R}} u^{i}K(u)du\bigg)h^{i}}{i!}\mathbb{E}\big[\int_{0}^{\infty}\mathbf{1}_{\{z\leq\frac{M}{\alpha}\}}ze^{-z}\partial_{a}^{i}f_{0}(z\alpha|\tilde{\theta}(X),X)dz\big]}.
\end{aligned}
\end{equation*}
%}\noindent
Note that given $X$, $\mathbf{1}_{\{z\leq\frac{M}{\alpha}\}}ze^{-z}\partial_{a}^{i}f_{0}(z\alpha|\tilde{\theta}(X),X)\rightarrow ze^{-z}\partial_{a}^{i}f_{0}(0|\tilde{\theta}(X),X)$ when $\alpha\rightarrow 0$. Since 
\begin{equation*}
\begin{aligned}
\int_{0}^{\infty} ze^{-z}\partial_{a}^{i}f_{0}(0|\tilde{\theta}(X),X)dz=\partial_{a}^{i}f_{0}(0|\tilde{\theta}(X),X),
\end{aligned}
\end{equation*}
we can conclude that 
\begin{equation*}
\begin{aligned}
\underset{\alpha\rightarrow 0}{\lim}\mathbb{E}\big[\int_{0}^{\infty}\mathbf{1}_{\{z\leq\frac{M}{\alpha}\}}ze^{-z}\partial_{a}^{i}f_{0}(z\alpha|\tilde{\theta}(X),X)dz\big]\text{ is finite}
\end{aligned}
\end{equation*}
using the Lebesgue convergence theorem. Denote 
\begin{equation*}
\begin{aligned}
\underset{\alpha\rightarrow 0}{\lim}\mathbb{E}\big[\int_{0}^{\infty}\mathbf{1}_{\{z\leq\frac{M}{\alpha}\}}ze^{-z}\partial_{a}^{i}f_{0}(z\alpha|\tilde{\theta}(X),X)dz\big]=c_{i}.
\end{aligned}
\end{equation*}
Thus, we have
%{\small
\begin{equation*}
\begin{aligned}
&\underset{\alpha\rightarrow 0}{\lim}\frac{\mathbb{E}\big[\frac{ZY}{\alpha}e^{-\frac{Y}{\alpha}}\big]}{\mathbb{E}\big[Ze^{-\frac{Y}{\alpha}}\big]}\leq\underset{\alpha\geq 0}{\lim\sup}\frac{\mathbb{E}\big[\frac{ZY}{\alpha}e^{-\frac{Y}{\alpha}}\big]}{\mathbb{E}\big[Ze^{-\frac{Y}{\alpha}}\big]}\\
\leq&\underset{\alpha\geq 0}{\lim\sup}\frac{\underset{i=0}{\overset{N+1}{\sum}}\frac{\bar{b}^{i}\bigg(\int_{\mathbb{R}} u^{i}K(u)du\bigg)h^{i}}{i!}\alpha}{\underset{i=0}{\overset{N+1}{\sum}}\frac{\bigg(\int_{\mathbb{R}} u^{i}K(u)du\bigg)h^{i}}{i!}\mathbb{E}\big[\int_{0}^{M}\frac{y}{\alpha}e^{-\frac{y}{\alpha}}\partial_{a}^{i}f_{0}(y|\tilde{\theta}(X),X)dy\big]}
%,\quad\tilde{\theta}(X)=\begin{cases}
%\pi(X)&\text{if $0\leq i\leq N$}\\
%\theta(X)&\text{if $i=N+1$}
%\end{cases}\\
%=&\underset{\alpha\geq 0}{\lim\sup}\frac{\underset{i=0}{\overset{N+1}{\sum}}\frac{\bar{b}^{i}\bigg(\int_{\mathbb{R}} u^{i}K(u)du\bigg)h^{i}}{i!}}{\underset{i=0}{\overset{N+1}{\sum}}\frac{\bigg(\int_{\mathbb{R}} u^{i}K(u)du\bigg)h^{i}}{i!}\mathbb{E}\big[\int_{0}^{\infty}\mathbf{1}_{\{z\leq\frac{M}{\alpha}\}}ze^{-z}\partial_{a}^{i}f_{0}(z\alpha|\tilde{\theta}(X),X)dz\big]}\\
%=&\underset{\alpha\geq 0}{\lim}\frac{\underset{i=0}{\overset{N+1}{\sum}}\frac{\bar{b}^{i}\bigg(\int_{\mathbb{R}} u^{i}K(u)du\bigg)h^{i}}{i!}}{\underset{i=0}{\overset{N+1}{\sum}}\frac{\bigg(\int_{\mathbb{R}} u^{i}K(u)du\bigg)h^{i}}{i!}\mathbb{E}\big[\int_{0}^{\infty}\mathbf{1}_{\{z\leq\frac{M}{\alpha}\}}ze^{-z}\partial_{a}^{i}f_{0}(z\alpha|\tilde{\theta}(X),X)dz\big]}\\
\rightarrow\frac{\underset{i=0}{\overset{N+1}{\sum}}\frac{\bar{b}^{i}\bigg(\int_{\mathbb{R}} u^{i}K(u)du\bigg)h^{i}}{i!}}{\underset{i=0}{\overset{N+1}{\sum}}\frac{1}{i!}c_{i}\bigg(\int_{\mathbb{R}} u^{i}K(u)du\bigg)h^{i}},
\end{aligned}
\end{equation*}
%}\noindent
which is a finite quantity. As such, we can conclude that $\underset{\alpha\rightarrow 0}{\lim}\frac{\partial\bar{\phi}_{1}(\alpha)}{\partial\alpha}\leq \underset{\alpha\rightarrow 0}{\lim\sup}\frac{\partial\bar{\phi}_{1}(\alpha)}{\partial\alpha}\rightarrow-\infty$. 

\noindent From the above analysis, we know that $\bar{\phi}_{1}(\alpha)$ is a strictly convex function such that it decreases first and then increases. The optimal point is finite.

\textbf{\textit{Proof of Claim \ref{lemma:prop of fun phi}}} The arguments are almost the same when we replace $Z$ with 1 when presenting the proof of Claim \ref{lemma:prop of fun phi2}. The only difference is bounding the quantity $\frac{\mathbb{E}[\frac{Y}{\alpha}e^{-\frac{Y}{\alpha}}]}{\mathbb{E}[e^{-\frac{Y}{\alpha}}]}$:
\begin{equation*}
\begin{aligned}
&\frac{\mathbb{E}\big[\frac{Y}{\alpha}e^{-\frac{Y}{\alpha}}\big]}{\mathbb{E}[e^{-\frac{Y}{\alpha}}]}=\frac{\int_{0}^{M}\frac{y}{\alpha}e^{-\frac{y}{\alpha}}f_{Y}(y)dy}{\int_{0}^{M}e^{-\frac{y}{\alpha}}f_{Y}(y)dy}\leq \frac{\bar{b}\int_{0}^{M}\frac{y}{\alpha}e^{-\frac{y}{\alpha}}dy}{\underline{b}\int_{0}^{M}e^{-\frac{y}{\alpha}}dy}=\frac{\bar{b}\alpha\int_{0}^{\frac{M}{\alpha}}ze^{-z}dz}{\underline{b}\alpha (1-e^{-\frac{M}{\alpha}})}\leq\frac{\bar{b}}{\underline{b}(1-e^{-\frac{M}{\alpha}})}.
\end{aligned}
\end{equation*}
Hence, we conclude that $\underset{\alpha\rightarrow 0}{\lim\sup}\frac{\mathbb{E}\big[\frac{Y}{\alpha}e^{-\frac{Y}{\alpha}}\big]}{\mathbb{E}[e^{-\frac{Y}{\alpha}}]}\leq\underset{\alpha\rightarrow 0}{\lim\sup}\frac{\bar{b}}{\underline{b}(1-e^{-\frac{M}{\alpha}})}=\frac{\bar{b}}{\underline{b}}$. We therefore conclude that $\underset{\alpha\rightarrow 0}{\lim}\frac{\partial\tilde{\phi}_{1}(\alpha)}{\partial\alpha}\leq \underset{\alpha\rightarrow 0}{\lim\sup}\frac{\partial\tilde{\phi}_{1}(\alpha)}{\partial\alpha}\rightarrow-\infty$.

\noindent From the above analysis, we know that $\tilde{\phi}_{1}(\alpha)$ is a strictly convex function such that it decreases first and then increases. Thus, the optimal solution of $\tilde{\phi}_{1}(\alpha)$ is finite.
\end{proof}

From Proposition \ref{lemma:finite optimal study}, we conclude that the optimal solutions of $Q_{\mathrm{DRO}}(\pi)$ and $Q_{\mathrm{DRO}}^{h}(\pi)$ are finite. This result is useful in proving Theorem \ref{lemma:asy_res_normalized and thm:asy result normalized}.
%We are now ready to prove Assertion 2 of Theorem \ref{lemma:asy_res_normalized and thm:asy result normalized}.
\paragraph{Auxiliary Result 3 - Uniform boundedness of two probability measures}\label{sec:auxiliary result 3}\ \\
%Before presenting the proofs of Theorem \ref{thm:statistical performance normalized}, we present Proposition \ref{lemma:QDRO and quantile} which will be used when proving Theorem \ref{thm:statistical performance normalized}.
\begin{proposition}\label{lemma:QDRO and quantile}
For any probability measures $\mathbb{P}_{1}$ and $\mathbb{P}_{2}$ on the continuous variable $Y$, we have
%{\small
\begin{equation*}
\begin{aligned}
&\bigg|\underset{\alpha\geq 0}{\sup}\{-\alpha\log\mathbb{E}_{\mathbb{P}_{1}}[e^{-\frac{Y}{\alpha}}]-\alpha\eta\}-\underset{\alpha\geq 0}{\sup}\{-\alpha\log\mathbb{E}_{\mathbb{P}_{2}}[e^{-\frac{Y}{\alpha}}]-\alpha\eta\}\bigg|\\
&\leq \underset{\alpha\geq 0}{\sup}\alpha\bigg|\log\mathbb{E}_{\mathbb{P}_{1}}[e^{-\frac{Y}{\alpha}}]-\log\mathbb{E}_{\mathbb{P}_{2}}[e^{-\frac{Y}{\alpha}}]\bigg|\leq \underset{t\in[0,1]}{\sup}|\mathcal{Q}_{\mathbb{P}_{1}}(t)-\mathcal{Q}_{\mathbb{P}_{2}}(t)|.
\end{aligned}
\end{equation*}
%}\noindent
Here, $\mathcal{Q}_{\mathbb{P}}(t)=\underset{}{\inf}\{x\in\mathbb{R}:F_{\mathbb{P}}(x)\geq t\}$ is the $t$-quantile of the probability measure $\mathbb{P}$ where $F_{\mathbb{P}}(x)$ is the cumulative density function (CDF) of the probability measure $\mathbb{P}$. Additionally, suppose that one of the probability measure (e.g., $\mathbb{P}_{1}$) has a probability density $f_{\mathbb{P}_{1}}(\cdot)$ which is bounded below by a constant $c>0$. The we have 
\begin{equation*}
\begin{aligned}
\underset{t\in[0,1]}{\sup}|\mathcal{Q}_{\mathbb{P}_{1}}(t)-\mathcal{Q}_{\mathbb{P}_{2}}(t)|\leq\frac{1}{c}\underset{x\in[0,M]}{\sup}|F_{\mathbb{P}_{1}}(x)-F_{\mathbb{P}_{2}}(x)|.
\end{aligned}
\end{equation*}
\end{proposition}

\begin{proof}
The proof can be found in \cite{si2023distributionally}. We restate here for completeness. Note that $|\underset{x}{\sup}f_{1}(x)-\underset{x}{\sup}f_{2}(x)|\leq\underset{x}{\sup}|f_{1}(x)-f_{2}(x)|$. Recall that, given a variable $X$ with CDF $F_{\mathbb{P}}(\cdot)$ and a quantile function $\mathcal{Q}_{\mathbb{P}}(t)=\inf\{x:F_{\mathbb{P}}(x)\geq t\}$, we have $\mathcal{Q}_{\mathbb{P}}(U) \overset{d}{=}X$ under $\mathbb{P}$ where $U$ is a uniform random variable under measure $\mathbb{P}$. Hence, we have 
%{\small
\begin{equation*}
\begin{aligned}
&\bigg|\underset{\alpha\geq 0}{\sup}\{-\alpha\log\mathbb{E}_{\mathbb{P}_{1}}[e^{-\frac{Y}{\alpha}}]-\alpha\eta\}-\underset{\alpha\geq 0}{\sup}\{-\alpha\log\mathbb{E}_{\mathbb{P}_{2}}[e^{-\frac{Y}{\alpha}}]-\alpha\eta\}\bigg|\\
\leq&\underset{\alpha\geq 0}{\sup}\bigg|\alpha\log\mathbb{E}_{\mathbb{P}_{1}}[e^{-\frac{Y}{\alpha}}]-\alpha\log\mathbb{E}_{\mathbb{P}_{2}}[e^{-\frac{Y}{\alpha}}]\bigg|=\underset{\alpha\geq 0}{\sup}\bigg|\alpha\log\mathbb{E}_{\mathbb{P}}\bigg[e^{-\frac{\mathcal{Q}_{\mathbb{P}_{1}}(U)}{\alpha}}\bigg]-\alpha\log\mathbb{E}_{\mathbb{P}}\bigg[e^{-\frac{\mathcal{Q}_{\mathbb{P}_{2}}(U)}{\alpha}}\bigg]\bigg|\\
=&\underset{\alpha\geq 0}{\sup}\;\alpha\bigg|\log\Bigg(\int_{u\in[0,1]} e^{-\frac{\mathcal{Q}_{\mathbb{P}_{1}}(u)}{\alpha}}du\Bigg)-\log\Bigg(\int_{u\in[0,1]} e^{-\frac{\mathcal{Q}_{\mathbb{P}_{2}}(u)}{\alpha}}du\Bigg)\bigg|.
\end{aligned}
\end{equation*}
%}\noindent
Consider the term $\log\Bigg(\int_{u\in[0,1]} e^{-\frac{\mathcal{Q}_{\mathbb{P}_{1}}(u)}{\alpha}}du\Bigg)-\log\Bigg(\int_{u\in[0,1]} e^{-\frac{\mathcal{Q}_{\mathbb{P}_{2}}(u)}{\alpha}}du\Bigg)$. Since
{\small
\begin{equation*}
\begin{aligned}
&\log\Bigg(\int_{u\in[0,1]} e^{-\frac{\mathcal{Q}_{\mathbb{P}_{1}}(u)}{\alpha}}du\Bigg)-\log\Bigg(\int_{u\in[0,1]} e^{-\frac{\mathcal{Q}_{\mathbb{P}_{2}}(u)}{\alpha}}du\Bigg)\\
%=&\log\Bigg(\int_{u\in[0,1]} e^{-\frac{\mathcal{Q}_{\mathbb{P}_{1}}(u)}{\alpha}+\frac{\mathcal{Q}_{\mathbb{P}_{2}}(u)}{\alpha}}e^{-\frac{\mathcal{Q}_{\mathbb{P}_{2}}(u)}{\alpha}}du\Bigg)-\log\Bigg(\int_{u\in[0,1]} e^{-\frac{\mathcal{Q}_{\mathbb{P}_{2}}(u)}{\alpha}}du\Bigg)\\
\leq&\log\Bigg(\int_{u\in[0,1]} e^{\frac{\underset{u\in[0,1]}{\sup}|\mathcal{Q}_{\mathbb{P}_{1}}(u)-\mathcal{Q}_{\mathbb{P}_{2}}(u)|}{\alpha}}e^{-\frac{\mathcal{Q}_{\mathbb{P}_{2}}(u)}{\alpha}}du\Bigg)-\log\Bigg(\int_{u\in[0,1]} e^{-\frac{\mathcal{Q}_{\mathbb{P}_{2}}(u)}{\alpha}}du\Bigg)=\frac{\underset{u\in[0,1]}{\sup}|\mathcal{Q}_{\mathbb{P}_{1}}(u)-\mathcal{Q}_{\mathbb{P}_{2}}(u)|}{\alpha}
\end{aligned}
\end{equation*}
}\noindent
and
{\small
\begin{equation*}
\begin{aligned}
&\log\Bigg(\int_{u\in[0,1]} e^{-\frac{\mathcal{Q}_{\mathbb{P}_{1}}(u)}{\alpha}}du\Bigg)-\log\Bigg(\int_{u\in[0,1]} e^{-\frac{\mathcal{Q}_{\mathbb{P}_{2}}(u)}{\alpha}}du\Bigg)\\
%=&\log\Bigg(\int_{u\in[0,1]} e^{-\frac{\mathcal{Q}_{\mathbb{P}_{1}}(u)}{\alpha}}du\Bigg)-\log\Bigg(\int_{u\in[0,1]} e^{-\frac{\mathcal{Q}_{\mathbb{P}_{2}}(u)}{\alpha}+\frac{\mathcal{Q}_{\mathbb{P}_{1}}(u)}{\alpha}}e^{-\frac{\mathcal{Q}_{\mathbb{P}_{1}}(u)}{\alpha}}du\Bigg)\\
\geq&\log\Bigg(\int_{u\in[0,1]} e^{-\frac{\mathcal{Q}_{\mathbb{P}_{1}}(u)}{\alpha}}du\Bigg)-\log\Bigg(\int_{u\in[0,1]} e^{\frac{\underset{u\in[0,1]}{\sup}|\mathcal{Q}_{\mathbb{P}_{1}}(u)-\mathcal{Q}_{\mathbb{P}_{2}}(u)|}{\alpha}}e^{-\frac{\mathcal{Q}_{\mathbb{P}_{1}}(u)}{\alpha}}du\Bigg)=-\frac{\underset{u\in[0,1]}{\sup}|\mathcal{Q}_{\mathbb{P}_{1}}(u)-\mathcal{Q}_{\mathbb{P}_{2}}(u)|}{\alpha},
\end{aligned}
\end{equation*}
}\noindent
we have
%{\small
\begin{equation*}
\begin{aligned}
&\bigg|\underset{\alpha\geq 0}{\sup}\{-\alpha\log\mathbb{E}_{\mathbb{P}_{1}}[e^{-\frac{Y}{\alpha}}]-\alpha\eta\}-\underset{\alpha\geq 0}{\sup}\{-\alpha\log\mathbb{E}_{\mathbb{P}_{2}}[e^{-\frac{Y}{\alpha}}]-\alpha\eta\}\bigg|\\
\leq&\underset{\alpha\geq 0}{\sup}\;\alpha\bigg|\log\Bigg(\int_{u\in[0,1]} e^{-\frac{\mathcal{Q}_{\mathbb{P}_{1}}(u)}{\alpha}}du\Bigg)-\log\Bigg(\int_{u\in[0,1]} e^{-\frac{\mathcal{Q}_{\mathbb{P}_{2}}(u)}{\alpha}}du\Bigg)\bigg|\\
\leq&\underset{\alpha\geq 0}{\sup}\;\alpha \frac{\underset{u\in[0,1]}{\sup}|\mathcal{Q}_{\mathbb{P}_{1}}(u)-\mathcal{Q}_{\mathbb{P}_{2}}(u)|}{\alpha}=\underset{\alpha\geq 0}{\sup}\;\underset{u\in[0,1]}{\sup}|\mathcal{Q}_{\mathbb{P}_{1}}(u)-\mathcal{Q}_{\mathbb{P}_{2}}(u)|=\underset{u\in[0,1]}{\sup}|\mathcal{Q}_{\mathbb{P}_{1}}(u)-\mathcal{Q}_{\mathbb{P}_{2}}(u)|.
\end{aligned}
\end{equation*}
%}\noindent
Next, we further bound $\underset{u\in[0,1]}{\sup}|\mathcal{Q}_{\mathbb{P}_{1}}(u)-\mathcal{Q}_{\mathbb{P}_{2}}(u)|$. Write $x_{1}=\mathcal{Q}_{\mathbb{P}_{1}}(t)$ and $x_{2}=\mathcal{Q}_{\mathbb{P}_{2}}(t)$. Since the measure $\mathbb{P}_{1}$ is continuous, we have $F_{\mathbb{P}_{1}}(x_{1})=t$ where $F_{\mathbb{P}_{1}}(\cdot)$ is the CDF of measure $\mathbb{P}_{1}$. Simultaneously, we do not impose the continuity of measure $\mathbb{P}_{2}$, but the CDF of measure $\mathbb{P}_{2}$ (denoted as $F_{\mathbb{P}_{2}}(\cdot)$) must be right continuous with left limit. As a result, we have $F_{\mathbb{P}_{2}}(x_{2})\geq t$ and $F_{\mathbb{P}_{2}}(x_{2}-)\leq t$. We now consider two cases: 1) $x_{1}\geq x_{2}$ and 2) $x_{1}< x_{2}$.

\noindent For the case $x_{1}\geq x_{2}$, by the mean value theorem, we have
%{\small
\begin{align*}
&\begin{aligned}
&F_{\mathbb{P}_{1}}(x_{1})-F_{\mathbb{P}_{1}}(x_{2})=F_{\mathbb{P}_{1}}^{'}(\theta)(x_{1}-x_{2})\\
&=f_{\mathbb{P}_{1}}(\theta)(x_{1}-x_{2}) \quad \text{Here, $\theta$ lie in between $F_{\mathbb{P}_{1}}^{-1}(t)$ and $F_{\mathbb{P}_{2}}^{-1}(t)$}
\end{aligned}\nonumber\\
&\begin{aligned}
\Rightarrow (x_{1}-x_{2})&=\frac{1}{f_{\mathbb{P}_{1}}(\theta)} (F_{\mathbb{P}_{1}}(x_{1})-F_{\mathbb{P}_{1}}(x_{2}))\\
\Rightarrow (x_{1}-x_{2})&\leq\frac{1}{c} (F_{\mathbb{P}_{1}}(x_{1})-F_{\mathbb{P}_{1}}(x_{2}))\\
&=\frac{1}{c} (F_{\mathbb{P}_{1}}(x_{1})-F_{\mathbb{P}_{2}}(x_{2})+F_{\mathbb{P}_{2}}(x_{2})-F_{\mathbb{P}_{1}}(x_{2}))\\
&\leq\frac{1}{c} (t-t+F_{\mathbb{P}_{2}}(x_{2})-F_{\mathbb{P}_{1}}(x_{2}))\\
&\leq\frac{1}{c} \;\underset{x\in[0,M]}{\sup}|F_{\mathbb{P}_{2}}(x)-F_{\mathbb{P}_{1}}(x)|.
\end{aligned}
\end{align*}
%}\noindent
For the case $x_{1}<x_{2}$, let $(x^{(n)})_{n=1}^{\infty}$ be a sequence such that $x^{(n)}\uparrow x_{2}$. For any $\epsilon>0$, we have
%{\small
\begin{equation*}
\begin{aligned}
|F_{\mathbb{P}_{1}}(x^{(n)})-F_{\mathbb{P}_{1}}(x_{2})|\leq\epsilon\quad\forall\;n\geq N.
\end{aligned}
\end{equation*}
%}\noindent
Particularly, choosing $n=N$ gives $|F_{\mathbb{P}_{1}}(x^{(N)})-F_{\mathbb{P}_{1}}(x_{2})|\leq\epsilon$. Besides, we have $F_{\mathbb{P}_{2}}(x^{(N)})\leq t=F_{\mathbb{P}_{1}}(x_{1})$. Consequently, we have
%{\small
\begin{align*}
&\begin{aligned}
x_{2}-x_{1}&\leq\frac{1}{c}(F_{\mathbb{P}_{1}}(x_{2})-F_{\mathbb{P}_{1}}(x_{1}))\\
&=\frac{1}{c}\left(F_{\mathbb{P}_{1}}(x_{2})-F_{\mathbb{P}_{1}}(x^{(N)})+F_{\mathbb{P}_{1}}(x^{(N)})-F_{\mathbb{P}_{2}}(x^{(N)})+F_{\mathbb{P}_{2}}(x^{(N)})-F_{\mathbb{P}_{1}}(x_{1})\right)
\end{aligned}\\
&\begin{aligned}
\Rightarrow x_{2}-x_{1}&\leq\frac{(|F_{\mathbb{P}_{1}}(x_{2})-F_{\mathbb{P}_{1}}(x^{(N)})|+|F_{\mathbb{P}_{1}}(x^{(N)})-F_{\mathbb{P}_{2}}(x^{(N)})|+0)}{c}\\
&\leq\frac{1}{c}(\epsilon+\underset{x\in[0,M]}{\sup}|F_{\mathbb{P}_{1}}(x)-F_{\mathbb{P}_{2}}(x)|).
\end{aligned}
\end{align*}
%}\noindent
Since $\epsilon$ is arbitrary, we conclude that $x_{2}-x_{1}\leq \frac{1}{c}\underset{x\in[0,M]}{\sup}|F_{\mathbb{P}_{1}}(x)-F_{\mathbb{P}_{2}}(x)|$. Combining the two cases, we can conclude that
%{\small
\begin{equation*}
\begin{aligned}
&|\mathcal{Q}_{\mathbb{P}_{1}}(u)-\mathcal{Q}_{\mathbb{P}_{2}}(u)|=|x_{2}-x_{1}|\leq \frac{1}{c}\underset{x\in[0,M]}{\sup}|F_{\mathbb{P}_{1}}(x)-F_{\mathbb{P}_{2}}(x)|\\
&\Rightarrow \underset{u\in[0,1]}{\sup}|\mathcal{Q}_{\mathbb{P}_{1}}(u)-\mathcal{Q}_{\mathbb{P}_{2}}(u)|\leq \frac{1}{c}\underset{x\in[0,M]}{\sup}|F_{\mathbb{P}_{1}}(x)-F_{\mathbb{P}_{2}}(x)|.
\end{aligned}
\end{equation*}
%}\noindent
\end{proof}
Proposition \ref{lemma:QDRO and quantile} is useful in proving Theorem \ref{thm:statistical performance normalized}.
\subsection{Proofs of results in the main paper}\label{sec:main paper}
\subsubsection{Proof of Lemma \ref{lemma:expectation equal IPW}}\label{sec:Lemma main}
\begin{proof}
Consider the quantity 
%$\mathbb{E}_{\mathbb{P}_{0}}\Bigg[\frac{e^{\frac{-Y}{\alpha}}\delta(\pi(X)-A)}{f_{0}(A|X)}\Bigg]$.
$\mathbb{E}\Bigg[\frac{e^{\frac{-Y}{\alpha}}\delta(\pi(X)-A)}{f_{0}(A|X)}\Bigg]$.
We have
%{\small
\begin{equation*}
\begin{aligned}
&\mathbb{E}\Bigg[\frac{e^{\frac{-Y}{\alpha}}\delta(\pi(X)-A)}{f_{0}(A|X)}\Bigg]=\mathbb{E}\Bigg[\mathbb{E}\Bigg[\frac{e^{\frac{-Y}{\alpha}}\delta(\pi(X)-A)}{f_{0}(A|X)}\Bigg|X\Bigg]\Bigg]\\
%=&\mathbb{E}\Bigg[\int\mathbb{E}\Bigg[\frac{e^{\frac{-Y}{\alpha}}\delta(\pi(X)-A)}{f_{0}(A|X)}\Bigg|Y=y,A=a,X\Bigg]f_{0}(y,a|X)\;dady\Bigg]\\
=&\mathbb{E}\Bigg[\int\mathbb{E}\Bigg[\frac{e^{\frac{-Y}{\alpha}}\delta(\pi(X)-a)}{f_{0}(a|X)}\Bigg|Y=y,A=a,X\Bigg]f_{0}(y,a|X)\;dady\Bigg]=\mathbb{E}\Bigg[\int e^{\frac{-y}{\alpha}}\frac{f_{0}(y,\pi(X)|X)}{f_{0}(\pi(X)|X)}\;dy\Bigg]\\
=&\mathbb{E}\Bigg[\frac{\mathbb{E}\big[e^{\frac{-Y}{\alpha}}|A=\pi(X),X\big]}{f_{0}(\pi(X)|X)}f_{0}(\pi(X)|X)\Bigg]=\mathbb{E}\bigg[\mathbb{E}\bigg[e^{\frac{-Y}{\alpha}}|A=\pi(X),X\bigg]\bigg]=\mathbb{E}\bigg[e^{\frac{-Y(\pi(X))}{\alpha}}\bigg].
\end{aligned}
\end{equation*}
%}\noindent
\end{proof}

\subsubsection{Proof of Theorem \ref{lemma:asy_res_normalized and thm:asy result normalized}}\label{sec:Theorem main 1}
We restate the Theorem here:
\begin{theorem*}
Suppose that $N\rightarrow\infty$, $h\rightarrow 0$ such that $Nh\rightarrow\infty$ and $Nh^{5}\rightarrow C\in[0,\infty)$. We have
\begin{equation*}
\begin{aligned}
&\sqrt{Nh}\bigg(\hat{W}_{N}^{h}-\mathbb{E}[e^{-\frac{Y(\pi(X))}{\alpha}}]-B_{\pi}(\alpha)h^{2}\bigg)
\overset{d}{\rightarrow}\mathcal{N}(0,\mathbb{V}_{\pi}(\alpha)),
\end{aligned}
\end{equation*}
where
\begin{equation*}
\begin{aligned}
B_{\pi}(\alpha)&=\frac{\big(\int u^{2}K(u)du\big)}{2}\times \mathbb{E}\bigg[\mathbb{E}\bigg[e^{\frac{-Y}{\alpha}}\frac{\partial_{aa}^{2}f_{0}(Y|\pi(X),X)}{f_{0}(Y|\pi(X),X)}\bigg|A=\pi(X),X\bigg]\bigg]\\
\mathbb{V}_{\pi}(\alpha)&=\bigg(\int K(u)^{2}du\bigg)\times\\
&\;\Bigg\{\mathbb{E}\bigg[\mathbb{E}\bigg[\frac{e^{-\frac{2Y}{\alpha}}}{f_{0}(\pi(X)|X)}\bigg|A=\pi(X),X\bigg]\bigg]\\
&\;+\mathbb{E}\bigg[\frac{1}{f_{0}(\pi(X)|X)}\bigg](\mathbb{E}[e^{-\frac{Y(\pi(X))}{\alpha}}])^{2}-2\mathbb{E}\bigg[\mathbb{E}\bigg[\frac{e^{-\frac{Y}{\alpha}}}{f_{0}(\pi(X)|X)}\bigg|A=\pi(X),X\bigg]\bigg]\mathbb{E}[e^{-\frac{Y(\pi(X))}{\alpha}}]\Bigg\}.
\end{aligned}
\end{equation*}
\end{theorem*}

%The proof of Theorem \ref{aaai_2025-lemma:asy_res_normalized and thm:asy result normalized} is split into two parts. We first show that Eqns. \eqref{aaai_2025-eqt:W_IPW hat asy result} holds such that $B_{\pi}^{h}(\alpha)$ and $ $ are given in Eqns. \eqref{aaai_2025-eqt:W_IPW hat asy result B} and \eqref{aaai_2025-eqt:W_IPW hat asy result V}. Then we show that \eqref{aaai_2025-eqt:Q_IPW hat asy result} holds. 
%\paragraph{Proof of Eqns. \eqref{eqt:W_IPW hat asy result}}
\begin{proof}[\textbf{Proof of Theorem \ref{lemma:asy_res_normalized and thm:asy result normalized}}]
Denote {\small $W^{h}(\pi;\alpha)=\frac{e^{\frac{-Y}{\alpha}}K_{h}(\pi(X)-A)}{f_{0}(A|X)}$}.
%\begin{equation*}
%\begin{aligned}
%W^{h}(\pi;\alpha)&=\frac{e^{\frac{-Y}{\alpha}}K_{h}(\pi(X)-A)}{f_{0}(A|X)}.
%\end{aligned}
%\end{equation*}
Considering {\small $\hat{W}_{N}^{h}-\mathbb{E}[W^{h}(\pi;\alpha)]$}, we have
%{\small
\begin{equation*}
\begin{aligned}
&\hat{W}_{N}^{h}-\mathbb{E}[W^{h}(\pi;\alpha)]=\frac{\frac{1}{N}\underset{i=1}{\overset{N}{\sum}}\bigg\{\frac{e^{\frac{-Y_{i}}{\alpha}}K_{h}(\pi(X_{i})-A_{i})}{f_{0}(A_{i}|X_{i})}-\frac{K_{h}(\pi(X_{i})-A_{i})}{f_{0}(A_{i}|X_{i})}\mathbb{E}[W^{h}(\pi;\alpha)]\bigg\}}{\frac{1}{N}\underset{i=1}{\overset{N}{\sum}}\frac{K_{h}(\pi(X_{i})-A_{i})}{f_{0}(A_{i}|X_{i})}}.
\end{aligned}
\end{equation*}
%}\noindent
By the Central Limit Theorem, for each $h$, we have
%{\small
\begin{equation*}
\begin{aligned}
&\begin{aligned}
&\sqrt{N}\Bigg(\frac{1}{N}\underset{i=1}{\overset{N}{\sum}}\bigg\{\frac{e^{\frac{-Y}{\alpha}}K_{h}(\pi(X_{i})-A_{i})}{f_{0}(A_{i}|X_{i})}\\
&\quad\quad\quad\quad\quad\quad-\frac{K_{h}(\pi(X_{i})-A_{i})\mathbb{E}[W^{h}(\pi;\alpha)]}{f_{0}(A_{i}|X_{i})}\bigg\}\sqrt{h}\\
&\quad\quad\quad-\mathbb{E}\Bigg[W^{h}(\pi;\alpha)- \frac{K_{h}(\pi(X)-A)}{f_{0}(A|X)}\mathbb{E}[W^{h}(\pi;\alpha)]\Bigg]\sqrt{h}\Bigg)
\end{aligned}\\
&\overset{d}{\rightarrow}\mathcal{N}\Bigg(0,h\mathbb{V}\bigg(W^{h}(\pi;\alpha)- \frac{K_{h}(\pi(X_{i})-A_{i})}{f_{0}(A_{i}|X_{i})}\mathbb{E}[W^{h}(\pi;\alpha)]\bigg)\Bigg)\\
&\begin{aligned}
\Rightarrow &\sqrt{Nh}\Bigg(\frac{1}{N}\underset{i=1}{\overset{N}{\sum}}\bigg\{\frac{e^{\frac{-Y_{i}}{\alpha}}K_{h}(\pi(X_{i})-A_{i})}{f_{0}(A_{i}|X_{i})}\\
&\quad\quad\quad\quad-\frac{\mathbb{E}[W^{h}(\pi;\alpha)]K_{h}(\pi(X_{i})-A_{i})}{f_{0}(A_{i}|X_{i})}\bigg\}\Bigg)
\end{aligned}\\
&\overset{d}{\rightarrow}\mathcal{N}(\text{Mean}^{h}, \text{Variance}^{h}),
\end{aligned}
\end{equation*}
%}\noindent
where
%{\small
\begin{equation*}
\begin{gathered}
\text{Mean}^{h}=\mathbb{E}\Bigg[W^{h}(\pi;\alpha)- \frac{K_{h}(\pi(X)-A)}{f_{0}(A|X)}\mathbb{E}[W^{h}(\pi;\alpha)]\Bigg]\sqrt{h}\\
\text{Variance}^{h}=h\mathbb{V}\bigg(W^{h}(\pi;\alpha)- \frac{K_{h}(\pi(X)-A)}{f_{0}(A|X)}\mathbb{E}[W^{h}(\pi;\alpha)]\bigg).
\end{gathered}
\end{equation*}
%}\noindent
%Suppose the following convergence assumptions:
%\begin{equation}
%\begin{aligned}\label{eqt:convergence assumptions}
%N\rightarrow\infty,\quad h\rightarrow 0\quad\text{such that}\quad Nh\rightarrow\infty\quad\text{and}\quad Nh^{5}\rightarrow C\in[0,\infty).
%\end{aligned}
%\end{equation}
From Proposition \ref{result:convergence probability h}, we know that $\frac{1}{N}\underset{i=1}{\overset{N}{\sum}}\frac{K_{h}(\pi(X_{i})-A_{i})}{f_{0}(A_{i}|X_{i})}\overset{p}{\rightarrow}1$. Therefore by Slutsky's Theorem, we conclude that
%{\small
\begin{align}\label{eqt:convergence in distribution normalized version}
&\sqrt{Nh}(\hat{W}_{N}^{h}-\mathbb{E}[W^{h}(\pi;\alpha)])\nonumber\\
&=\frac{\sqrt{Nh}\bigg(\frac{1}{N}\underset{i=1}{\overset{N}{\sum}}\{\frac{e^{\frac{-Y_{i}}{\alpha}}K_{h}(\pi(X_{i})-A_{i})}{f_{0}(A_{i}|X_{i})}- \frac{\mathbb{E}[W^{h}(\pi;\alpha)]K_{h}(\pi(X_{i})-A_{i})}{f_{0}(A_{i}|X_{i})}\}\bigg)}{\frac{1}{N}\underset{i=1}{\overset{N}{\sum}}\frac{K_{h}(\pi(X_{i})-A_{i})}{f_{0}(A_{i}|X_{i})}}\nonumber\\
&\overset{d}{\rightarrow}\mathcal{N}(\text{Mean}^{h}, \text{Variance}^{h}).
\end{align}
%}\noindent
We now study the quantities 
\begin{equation*}
\begin{gathered}
\mathbb{E}[W^{h}(\pi;\alpha)],\\
\mathbb{E}\Bigg[W^{h}(\pi;\alpha)- \frac{K_{h}(\pi(X)-A)}{f_{0}(A|X)}\mathbb{E}[W^{h}(\pi;\alpha)]\Bigg]\sqrt{h},\quad\text{and}\\
h\mathbb{V}\bigg(W^{h}(\pi;\alpha)- \frac{K_{h}(\pi(X)-A)}{f_{0}(A|X)}\mathbb{E}[W^{h}(\pi;\alpha)]\bigg)
\end{gathered}
\end{equation*}
accordingly. 

\noindent For the quantity $\mathbb{E}[W^{h}(\pi;\alpha)]$, we have
%{\small
\begin{align}
&\begin{aligned}
\mathbb{E}[W^{h}(\pi;\alpha)]&=\mathbb{E}\bigg[\frac{e^{-\frac{Y}{\alpha}}K_{h}(\pi(X)-A)}{f_{0}(A|X)}\bigg]=\mathbb{E}\bigg[\mathbb{E}\bigg[\frac{e^{-\frac{Y}{\alpha}}K_{h}(\pi(X)-A)}{f_{0}(A|X)}\bigg|X\bigg]\bigg]
\end{aligned}\nonumber\\
=&\mathbb{E}[e^{-\frac{Y(\pi(X))}{\alpha}}]+\nonumber\\
&\underbrace{\frac{h^{2}\big(\int K(u)u^{2}du\big)}{2}\mathbb{E}\bigg[\mathbb{E}\bigg[e^{-\frac{Y}{\alpha}}\frac{\partial_{aa}^{2}f_{0}(Y|\pi(X),X)}{f_{0}(Y|\pi(X),X)}\bigg|A=\pi(X),X\bigg]\bigg]}_{\coloneqq B_{\pi}(\alpha)}+O(h^{3}).\label{eqt:score IPW result}
\end{align}
%}\noindent
Therefore, the quantity $\sqrt{Nh}(\hat{W}_{N}^{h}-\mathbb{E}[W^{h}(\pi;\alpha)])$ in Eqn. \eqref{eqt:convergence in distribution normalized version} becomes
%{\small
\begin{equation*}
\begin{aligned}
&\sqrt{Nh}\bigg(\hat{W}_{N}^{h}-\mathbb{E}[e^{-\frac{Y(\pi(X))}{\alpha}}]-B_{\pi}(\alpha)h^{2}\bigg),\quad\text{where}\\
&B_{\pi}(\alpha)=\frac{\big(\int K(u)u^{2}du\big)}{2}\times\mathbb{E}\bigg[\mathbb{E}\bigg[e^{-\frac{Y}{\alpha}}\frac{\partial_{aa}^{2}f_{0}(Y|\pi(X),X)}{f_{0}(Y|\pi(X),X)}\bigg|A=\pi(X),X\bigg]\bigg].
\end{aligned}
\end{equation*}
%}\noindent
For the quantity {\small $\mathbb{E}\bigg[W^{h}(\pi;\alpha)- \frac{K_{h}(\pi(X)-A)}{f_{0}(A|X)}\mathbb{E}[W^{h}(\pi;\alpha)]\bigg]\sqrt{h}$}, we know that
%{\small
\begin{equation*}
\begin{aligned}
&\mathbb{E}\bigg[\frac{K_{h}(\pi(X)-A)}{f_{0}(A|X)}\bigg]=\mathbb{E}\bigg[\int\mathbb{E}\bigg[\frac{K_{h}(\pi(X)-A)}{f_{0}(A|X)}\bigg|Y=y,A=a,X\bigg]f_{0}(y,a|X)dyda\bigg]\\
=&\mathbb{E}\bigg[\int f_{0}(y|\pi(X),X)dy\bigg]+\frac{\big(\int u^{2}K(u)du\big)h^{2}}{2}\mathbb{E}\bigg[\int\partial_{aa}^{2}f_{0}(y|\pi(X),X)dy\bigg]+O(h^{3})\\
=&1+O(h^{2}).
\end{aligned}
\end{equation*}
%}\noindent
Hence, we have
%{\small
\begin{equation*}
\begin{aligned}
&\mathbb{E}\Bigg[W^{h}(\pi;\alpha)- \frac{K_{h}(\pi(X)-A)}{f_{0}(A|X)}\mathbb{E}[W^{h}(\pi;\alpha)]\Bigg]\sqrt{h}=\Bigg(\mathbb{E}[W^{h}(\pi;\alpha)]- \mathbb{E}\Bigg[\frac{K_{h}(\pi(X)-A)}{f_{0}(A|X)}\Bigg]\mathbb{E}[W^{h}(\pi;\alpha)]\Bigg)\sqrt{h}\\
&=\mathbb{E}[W^{h}(\pi;\alpha)](1-(1+O(h^{2})))\sqrt{h}=\Bigg(\mathbb{E}[e^{-\frac{Y(\pi(X))}{\alpha}}]+B_{\pi}(\alpha)h^{2}+O(h^{3})\Bigg)O(h^{\frac{5}{2}})\overset{}{\rightarrow}0.
\end{aligned}
\end{equation*}
%}\noindent
Finally, we study {\small $h\mathbb{V}\bigg(W^{h}(\pi;\alpha)- \frac{K_{h}(\pi(X)-A)}{f_{0}(A|X)}\mathbb{E}[W^{h}(\pi;\alpha)]\bigg)$}. Notice that
{\small
\begin{equation*}
\begin{aligned}
&h\mathbb{V}\bigg(W^{h}(\pi;\alpha)- \frac{K_{h}(\pi(X)-A)}{f_{0}(A|X)}\mathbb{E}[W^{h}(\pi;\alpha)]\bigg)\\
=&h\mathbb{E}\bigg[\bigg\{W^{h}(\pi;\alpha)- \frac{K_{h}(\pi(X)-A)}{f_{0}(A|X)}\mathbb{E}[W^{h}(\pi;\alpha)]\bigg\}^{2}\bigg]\\
&-h\bigg(\mathbb{E}\bigg[W^{h}(\pi;\alpha)- \frac{K_{h}(\pi(X_{i})-A_{i})}{f_{0}(A_{i}|X_{i})}\mathbb{E}[W^{h}(\pi;\alpha)]\bigg]\bigg)^{2}\\
=&h\mathbb{E}\bigg[\bigg\{W^{h}(\pi;\alpha)- \frac{K_{h}(\pi(X)-A)}{f_{0}(A|X)}\mathbb{E}[W^{h}(\pi;\alpha)]\bigg\}^{2}\bigg]\\
&-h\bigg(\Bigg(\mathbb{E}\bigg[e^{-\frac{Y(\pi(X))}{\alpha}}\bigg]+B_{\pi}(\alpha)h^{2}+O(h^{3})\Bigg)O(h^{\frac{5}{2}})\bigg)^{2}\\
\rightarrow &h\mathbb{E}\bigg[\bigg\{W^{h}(\pi;\alpha)- \frac{K_{h}(\pi(X)-A)}{f_{0}(A|X)}\mathbb{E}[W^{h}(\pi;\alpha)]\bigg\}^{2}\bigg].
\end{aligned}
\end{equation*}
}\noindent
Next, we move to consider the quantity {\small $h\mathbb{E}\biggl[\biggl(W^{h}(\pi;\alpha)-\frac{K_{h}(\pi(X)-A)}{f_{0}(A|X)}\mathbb{E}[W^{h}(\pi;\alpha)]\biggl)^{2}\biggl]$}. Note that
%{\small
\begin{equation}
\begin{aligned}\label{eqt:decomposition of variance}
&h\mathbb{E}\biggl[\biggl(W^{h}(\pi;\alpha)-\frac{K_{h}(\pi(X)-A)}{f_{0}(A|X)}\mathbb{E}[W^{h}(\pi;\alpha)]\biggl)^{2}\biggl]\\
=&\underbrace{h\mathbb{E}[(W^{h}(\pi;\alpha))^{2}]}_{\text{I}}\\
&-2\underbrace{h\mathbb{E}\biggl[W^{h}(\pi;\alpha)\frac{K_{h}(\pi(X)-A)}{f_{0}(A|X)}\biggl]\mathbb{E}[W^{h}(\pi;\alpha)]}_{\text{II}}\\
&+\underbrace{h\mathbb{E}\biggl[\frac{K_{h}^{2}(\pi(X)-A)}{f_{0}^{2}(A|X)}\biggl]\big(\mathbb{E}[W^{h}(\pi;\alpha)]\big)^{2}}_{\text{III}}.
\end{aligned}
\end{equation}
%}\noindent
We consider each quantity sequentially. For the quantity \text{I}, we have
%{\small
\begin{equation*}
\begin{aligned}
&\begin{aligned}
\text{I}&=h\mathbb{E}[(W^{h}(\pi;\alpha))^{2}]=h\mathbb{E}\bigg[\frac{K_{h}^{2}(\pi(X)-A)}{f_{0}^{2}(A|X)}e^{-\frac{2Y}{\alpha}}\bigg]=h\mathbb{E}\bigg[\mathbb{E}\bigg[\frac{K_{h}^{2}(\pi(X)-A)}{f_{0}^{2}(A|X)}e^{-\frac{2Y}{\alpha}}\bigg|X\bigg]\bigg]
\end{aligned}\\
&=h\mathbb{E}\bigg[\int K_{h}^{2}(\pi(X)-a)e^{-\frac{2y}{\alpha}}\tilde{g}(y,a;X)dyda\bigg]\\
&=\bigg(\int K^{2}(u)du\bigg)\mathbb{E}\bigg[\mathbb{E}\bigg[e^{-\frac{2Y}{\alpha}}\frac{\tilde{g}(Y,\pi(X);X)}{f_{0}(Y|\pi(X),X)}\bigg|A=\pi(X),X\bigg]\bigg]+O(h^{2})\\
&\rightarrow \bigg(\int K^{2}(u)du\bigg)\mathbb{E}\bigg[\mathbb{E}\bigg[\frac{e^{-\frac{2Y}{\alpha}}}{f_{0}(\pi(X)|X)}\bigg|A=\pi(X),X\bigg]\bigg].
\end{aligned}
\end{equation*}
%}\noindent
Here, $\tilde{g}(Y,A;X)=\frac{f_{0}(Y|A,X)}{f_{0}(A|X)}$. For the quantity \text{II}, we consider $h\mathbb{E}\biggl[W^{h}(\pi;\alpha)\frac{K_{h}(\pi(X)-A)}{f_{0}(A|X)}\biggl]$ and $\mathbb{E}[W^{h}(\pi;\alpha)]$ separately. First, for the quantity $h\mathbb{E}\biggl[W^{h}(\pi;\alpha)\frac{K_{h}(\pi(X)-A)}{f_{0}(A|X)}\biggl]$, we have
{\small
\begin{equation*}
\begin{aligned}
&h\mathbb{E}\biggl[W^{h}(\pi;\alpha)\frac{K_{h}(\pi(X)-A)}{f_{0}(A|X)}\biggl]=h\mathbb{E}\bigg[\frac{K_{h}^{2}(\pi(X)-A)}{f_{0}^{2}(A|X)}e^{-\frac{Y}{\alpha}}\bigg]\\
%=&h\mathbb{E}\bigg[\int\mathbb{E}\bigg[\frac{K_{h}^{2}(\pi(X)-A)}{f_{0}^{2}(A|X)}e^{-\frac{Y}{\alpha}}\bigg|Y=y,A=a,X\bigg]f_{0}(y,a|X)dyda\bigg],\quad\text{where $\tilde{g}(y,a;X)=\frac{f_{0}(y|a,X)}{f_{0}(a|X)}$}\\
%=&\mathbb{E}\bigg[\int K^{2}(u)e^{-\frac{y}{\alpha}}\bigg\{\tilde{g}(y,\pi(X);X)+\partial_{a}\tilde{g}(y,\pi(X);X)uh+\frac{\partial_{aa}^{2}\tilde{g}(y,\pi(X);X)}{2}u^{2}h^{2}+O_{P}(h^{3})\bigg\}dydu\bigg]\\
=&\bigg(\int K^{2}(u)du\bigg)\mathbb{E}\bigg[\mathbb{E}\bigg[\frac{e^{-\frac{Y}{\alpha}}}{f_{0}(\pi(X)|X)}\bigg|A=\pi(X),X\bigg]\bigg]+O(h^{2}).
\end{aligned}
\end{equation*}
}\noindent
Eqn. \eqref{eqt:score IPW result} investigates the quantity $\mathbb{E}[W^{h}(\pi;\alpha)]$. Recall that
\begin{equation*}
\begin{aligned}
\mathbb{E}[W^{h}(\pi;\alpha)]&=\mathbb{E}[\mathbb{E}[e^{-\frac{Y}{\alpha}}|A=\pi(X),X]]+O(h^{2})=\mathbb{E}[e^{-\frac{Y(\pi(X))}{\alpha}}]+O(h^{2}).
\end{aligned}
\end{equation*}
As such, $\text{II}\rightarrow \big(\int K^{2}(u)du\big)\mathbb{E}\bigg[\mathbb{E}\bigg[\frac{e^{-\frac{Y}{\alpha}}}{f_{0}(\pi(X)|X)}\bigg|A=\pi(X),X\bigg]\bigg]\mathbb{E}[e^{-\frac{Y(\pi(X))}{\alpha}}]$.

\noindent For the quantity $\text{III}$, we compute $h\mathbb{E}\biggl[\frac{K_{h}^{2}(\pi(X)-A)}{f_{0}^{2}(A|X)}\biggl]$ since $(\mathbb{E}[W^{h}(\pi;\alpha)])^{2}$ equals $\big(\mathbb{E}[e^{-\frac{Y(\pi(X))}{\alpha}}]\big)^{2}+O(h^{2})$. In the followings, we compute $h\mathbb{E}\biggl[\frac{K_{h}^{2}(\pi(X)-A)}{f_{0}^{2}(A|X)}\biggl]$:
%{\small
\begin{equation*}
\begin{aligned}
&h\mathbb{E}\biggl[\frac{K_{h}^{2}(\pi(X)-A)}{f_{0}^{2}(A|X)}\biggl]=h\mathbb{E}\biggl[\int\mathbb{E}\biggl[\frac{K_{h}^{2}(\pi(X)-A)}{f_{0}^{2}(A|X)}\biggl|Y=y,A=a,X\biggl]f_{0}(y,a|X)dyda\biggl]\\
=&\bigg(\int K^{2}(u) du\bigg)\mathbb{E}\bigg[\frac{1}{f_{0}(\pi(X)|X)}\bigg]+O(h^{2}).
\end{aligned}
\end{equation*}
%}\noindent
Consequently, the quantity $\text{III}$ would converge to {\small $\bigg(\int K^{2}(u)du\bigg)\mathbb{E}\bigg[\frac{1}{f_{0}(\pi(X)|X)}\bigg]\big(\mathbb{E}[e^{-\frac{Y(\pi(X))}{\alpha}}]\big)^{2}$}. Therefore, we conclude that
{\small
\begin{equation*}
\begin{aligned}
&h\mathbb{E}\biggl[\biggl(W^{h}(\pi;\alpha)-\frac{K_{h}(\pi(X)-A)}{f_{0}(A|X)}\mathbb{E}[W^{h}(\pi;\alpha)]\biggl)^{2}\biggl]\\
&\rightarrow\big(\int K^{2}(u)du\big)\times\\
&\quad\quad\bigg\{\mathbb{E}\bigg[\mathbb{E}\bigg[\frac{e^{-\frac{2Y}{\alpha}}}{f_{0}(\pi(X)|X)}\bigg|A=\pi(X),X\bigg]\bigg]\\
&\quad\quad\quad+\mathbb{E}\bigg[\frac{1}{f_{0}(\pi(X)|X)}\bigg](\mathbb{E}[e^{-\frac{Y(\pi(X))}{\alpha}}])^{2}\\
&\quad\quad\quad-2\mathbb{E}\bigg[\mathbb{E}\bigg[\frac{e^{-\frac{Y}{\alpha}}}{f_{0}(\pi(X)|X)}\bigg|A=\pi(X),X\bigg]\bigg]\mathbb{E}[e^{-\frac{Y(\pi(X))}{\alpha}}]\bigg\}\\
&\quad\quad\coloneqq\mathbb{V}_{\pi}(\alpha).
\end{aligned}
\end{equation*}
}\noindent
The proof is now completed.
\end{proof}
\subsubsection{Proof of Theorem \ref{eqt:Q_IPW hat asy result}}\label{sec:Theorem main added}
We restate the Theorem here:
\begin{theorem*}
Suppose that $N\rightarrow\infty$, $h\rightarrow 0$ such that $Nh\rightarrow\infty$ and $Nh^{5}\rightarrow C\in[0,\infty)$. Further, denote $\alpha_{\ast}(\pi)$ s.t.  $\phi(\pi,\alpha_{\ast}(\pi))\geq \phi(\pi,\alpha)$ $\forall\;\alpha\geq 0$.
Then we have
\begin{equation*}
\begin{aligned}
&\sqrt{Nh}\Biggl(\hat{Q}_{\mathrm{DRO}}^{h}(\pi)-Q_{\mathrm{DRO}}(\pi)+\frac{\alpha_{\ast}(\pi)B_{\pi}(\alpha_{\ast}(\pi))}{\mathbb{E}\bigg[e^{-\frac{Y(\pi(X))}{\alpha_{\ast}(\pi)}}\bigg]}h^{2}\Biggl)\\
&\overset{d}{\rightarrow}\mathcal{N}\left(0,\frac{\alpha_{\ast}^{2}(\pi)\mathbb{V}_{\pi}(\alpha_{\ast}(\pi))}{\bigg(\mathbb{E}\bigg[e^{-\frac{Y(\pi(X))}{\alpha_{\ast}(\pi)}}\bigg]\bigg)^{2}}\right).
\end{aligned}
\end{equation*}
\end{theorem*}
\begin{proof}[\textbf{Proof of Theorem \ref{eqt:Q_IPW hat asy result}}]
Denote 
\begin{equation*}
\begin{aligned}
W^{h}(\pi;\alpha)&=\frac{e^{\frac{-Y}{\alpha}}K_{h}(\pi(X)-A)}{f_{0}(A|X)}\quad\text{and}\\ W_{i}^{h}(\pi;\alpha)&=\frac{e^{\frac{-Y_{i}}{\alpha}}K_{h}(\pi(X_{i})-A_{i})}{f_{0}(A_{i}|X_{i})}.
\end{aligned}
\end{equation*}
We have the following prevalent result:
%{\small
\begin{equation*}
\begin{aligned}
&\sqrt{Nh}(\bar{W}_{N}^{h}-\mathbb{E}[W^{h}(\pi;\alpha)])\\
=&\sqrt{Nh}\bigg(\frac{1}{N}\underset{i=1}{\overset{N}{\sum}}\{W_{i}^{h}(\pi;\alpha)-\frac{K_{h}(\pi(X_{i})-A_{i})}{f_{0}(A_{i}|X_{i})}\mathbb{E}[W^{h}(\pi;\alpha)]\}\bigg)\\
\overset{d}{\rightarrow}&\mathcal{N}(\text{Mean},\text{Variance}):=Z^{h}(\alpha),
\end{aligned}
\end{equation*}
%}\noindent
where 
%{\small
\begin{equation*}
\begin{aligned}
\text{Mean}&=\mathbb{E}\Bigg[W^{h}(\pi;\alpha)-\frac{K_{h}(\pi(X)-A)}{f_{0}(A|X)}\mathbb{E}[W^{h}(\pi;\alpha)]\Bigg]\sqrt{h},\\
\text{Variance}&=h\mathbb{V}\bigg(W^{h}(\pi;\alpha)-\frac{K_{h}(\pi(X)-A)}{f_{0}(A|X)}\mathbb{E}[W^{h}(\pi;\alpha)]\bigg).
\end{aligned}
\end{equation*}
%}\noindent
Since $\frac{1}{N}\underset{i=1}{\overset{N}{\sum}}\frac{K_{h}(\pi(X_{i})-A_{i})}{f_{0}(A_{i}|X_{i})}\overset{p}{\rightarrow}1$, by Slutsky's Theorem, we conclude that
%{\small
\begin{equation*}
\begin{aligned}
&\sqrt{Nh}(\hat{W}_{N}^{h}-\mathbb{E}[W^{h}(\pi;\alpha)])=\frac{\sqrt{Nh}\bigg(\frac{1}{N}\underset{i=1}{\overset{N}{\sum}}\{W_{i}^{h}(\pi;\alpha)-\frac{K_{h}(\pi(X_{i})-A_{i})}{f_{0}(A_{i}|X_{i})}\mathbb{E}[W^{h}(\pi;\alpha)]\}\bigg)}{\frac{1}{N}\underset{i=1}{\overset{N}{\sum}}\frac{K_{h}(\pi(X_{i})-A_{i})}{f_{0}(A_{i}|X_{i})}}\\
\overset{d}{\rightarrow}&\mathcal{N}(\text{Mean}, \text{Variance}):=Z^{h}(\alpha).
\end{aligned}
\end{equation*}
%}\noindent
Denote
\begin{equation*}
\begin{gathered}
\begin{aligned}
\alpha_{\ast}(\pi)&=\underset{\alpha\geq 0}{\arg\max}\;\big\{-\alpha\log\mathbb{E}\big[e^{-\frac{Y(\pi(X))}{\alpha}}\big]-\alpha\eta\big\}\\
&=\underset{\alpha\geq 0}{\arg\min}\;\big\{\alpha\log\mathbb{E}\big[e^{-\frac{Y(\pi(X))}{\alpha}}\big]+\alpha\eta\big\}\\
\end{aligned}\\
\begin{aligned}
\alpha_{\ast}^{h;\mathrm{IPW}}(\pi)&=\underset{\alpha\geq 0}{\arg\max}\;\big\{-\alpha\log\mathbb{E}\big[W^{h}(\pi;\alpha)\big]-\alpha\eta\big\}\\
&=\underset{\alpha\geq 0}{\arg\min}\;\big\{\alpha\log\mathbb{E}\big[W^{h}(\pi;\alpha)\big]+\alpha\eta\big\}
\end{aligned}\\
\begin{aligned}
\bar{\alpha}_{\ast}^{N,h;\mathrm{IPW}}(\pi)&=\underset{\alpha\geq 0}{\arg\max}\;\big\{-\alpha\log \bar{W}^{h}(\pi,\alpha)-\alpha\eta\big\}\\
&=\underset{\alpha\geq 0}{\arg\min}\;\big\{\alpha\log \bar{W}^{h}(\pi,\alpha)+\alpha\eta\big\}
\end{aligned}.
\end{gathered}
\end{equation*}
According to Claim \ref{lemma:prop of fun phi} of Proposition \ref{lemma:finite optimal study}, $\alpha_{\ast}(\pi)$ is finite and $\alpha_{\ast}(\pi)>0$. 

\noindent Note also that $-\eta\alpha\leq \big\{-\alpha\log\mathbb{E}\big[e^{-\frac{Y(\pi(X))}{\alpha}}\big]-\alpha\eta\big\}\leq M-\alpha\eta$. Due to the concavity and continuity of $\big\{-\alpha\log\mathbb{E}\big[e^{-\frac{Y(\pi(X))}{\alpha}}\big]-\alpha\eta\big\}\leq M-\alpha\eta$, its optimal value is positive. 

Since for sufficiently large $N$ and small $h$, $\bar{\alpha}_{\ast}^{h;\mathrm{IPW}}(\pi)$ and $\bar{\alpha}_{\ast}^{N,h;\mathrm{IPW}}(\pi)$ are closed to $\bar{\alpha}_{\ast}(\pi)$. Thus, for sufficiently large $N$ and small $h$, we can choose $\underline{\alpha},\;\bar{\alpha}>0$ such that
\begin{equation*}
\begin{gathered}
0<\underline{\alpha}<\alpha_{\ast}(\pi),\bar{\alpha}_{\ast}^{h;\mathrm{IPW}}(\pi),\bar{\alpha}_{\ast}^{N,h;\mathrm{IPW}}(\pi)<\bar{\alpha},\quad\text{and}\\
\quad-\underline{\alpha}\log \mathcal{E}(\pi,\underline{\alpha})-\underline{\alpha}\eta>0,\\-\bar{\alpha}\log \mathcal{E}(\pi,\bar{\alpha})-\bar{\alpha}\eta>0,\\
\underline{\alpha}\log \mathcal{E}(\pi,\underline{\alpha})+\underline{\alpha}\eta\neq \bar{\alpha}\log \mathcal{E}(\pi,\bar{\alpha})+\bar{\alpha}\eta,
\end{gathered}
\end{equation*}
where $\mathcal{E}(\pi,\alpha)\in\{\mathbb{E}\big[e^{-\frac{Y(\pi(X))}{\alpha}}\big],\;\mathbb{E}\big[W^{h}(\pi,\alpha)\big],\;\bar{W}^{h}(\pi,\alpha)\}$.
%Hence, we can choose $\underline{\alpha}>0$ and $\bar{\alpha}>0$ such that $0<\underline{\alpha}<\alpha_{\ast}(\pi)<\bar{\alpha}$. Additionally, we can choose $\underline{\alpha}$ and $\bar{\alpha}$ such that 
%\begin{equation*}
%\begin{aligned}
%\underline{\alpha}\log \mathbb{E}\big[e^{-\frac{Y(\pi(X))}{\underline{\alpha}}}\big]+\underline{\alpha}\eta\neq \bar{\alpha}\log \mathbb{E}\big[e^{-\frac{Y(\pi(X))}{\bar{\alpha}}}\big]+\bar{\alpha}\eta
%\end{aligned}
%\end{equation*}
%due to the continuity of $\alpha\log \mathbb{E}[e^{-\frac{Y(\pi(X))}{\alpha}}]+\alpha\eta$. 
Next, we aim to show that {\small$\hat{W}_{N}^{h}(\pi,\alpha)=\frac{\frac{1}{N}\underset{i=1}{\overset{N}{\sum}}W_{i}^{h}(\pi,\alpha)}{\frac{1}{N}\underset{j=1}{\overset{N}{\sum}}\frac{K_{h}(\pi(X_{j})-A_{j})}{f_{0}(A_{j}|X_{j})}}$} is Lipschitz continuous on $[\underline{\alpha},\bar{\alpha}]$. From Claim \ref{lemma:useful asy result} of Proposition \ref{result:convergence S_N_h}, we conclude that for $0<\gamma<1$, there exists $\tilde{N}>0$ and $\tilde{h}>0$ such that $0<1-\gamma<\frac{1}{N}\underset{j=1}{\overset{N}{\sum}}\frac{K_{h}(\pi(X_{j})-A_{j})}{f_{0}(A_{j}|X_{j})}<1+\gamma$ for $N>\tilde{N}$ and $0<h<\tilde{h}$. In fact, when $N>\tilde{N}$ and $h<\tilde{h}$, $\frac{1}{\frac{1}{N}\underset{j=1}{\overset{N}{\sum}}\frac{K_{h}(\pi(X_{j})-A_{j})}{f_{0}(A_{j}|X_{j})}}\leq\frac{1}{1-\gamma}$. Note that $W_{i}^{h}(\pi,\alpha)$ is Lipschitz continuous on $[\underline{\alpha},\bar{\alpha}]$. The derivations are as follows: for any $\alpha_{1},\;\alpha_{2}\in[\underline{\alpha},\bar{\alpha}]$, we have
%{\small
\begin{equation*}
\begin{aligned}
&|W^{h}(\pi,\alpha_{1})-W^{h}(\pi,\alpha_{2})|=\bigg|\frac{K_{h}(\pi(X)-A)}{f_{0}(A|X)}[e^{-\frac{Y}{\alpha_{1}}}-e^{-\frac{Y}{\alpha_{2}}}]\bigg|\\
&
%=\frac{K_{h}(\pi(X)-A)}{f_{0}(A|X)}\biggl[\frac{Ye^{-\frac{Y}{\alpha}}}{\alpha^{2}}\biggl]|\alpha_{1}-\alpha_{2}|
\leq \frac{K_{h}(\pi(X)-A)}{f_{0}(A|X)}\biggl[\frac{Ye^{-\frac{Y}{\alpha_{\theta}}}}{\alpha_{\theta}^{2}}\biggl]|\alpha_{1}-\alpha_{2}|\leq L_{h}|\alpha_{1}-\alpha_{2}|.
\end{aligned}
\end{equation*}
%}\noindent
Here, $\alpha_{\theta}$ lies between $\alpha_{1}$ and $\alpha_{2}$ and we assume $|K(\cdot)|\leq M_{K}$ to be such that $L_{h}=\frac{M M_{K}}{h\epsilon\underline{\alpha}^{2}}$. As such, we can conclude that $\hat{W}_{N}^{h}(\pi,\alpha)$ is Lipschitz continuous on $[\underline{\alpha},\bar{\alpha}]$: indeed, for $N>\tilde{N}$, $h<\tilde{h}$ and $\alpha_{1},\alpha_{2}\in[\underline{\alpha},\bar{\alpha}]$, we have
\begin{equation*}
\begin{aligned}
&|\hat{W}_{N}^{h}(\pi,\alpha_{1})-\hat{W}_{N}^{h}(\pi,\alpha_{2})|\\
=&\left|\frac{\frac{1}{N}\underset{i=1}{\overset{N}{\sum}}W_{i}^{h}(\pi,\alpha_{1})}{\frac{1}{N}\underset{j=1}{\overset{N}{\sum}}\frac{K_{h}(\pi(X_{j})-A_{j})}{f_{0}(A_{j}|X_{j})}}-\frac{\frac{1}{N}\underset{i=1}{\overset{N}{\sum}}W_{i}^{h}(\pi,\alpha_{2})}{\frac{1}{N}\underset{j=1}{\overset{N}{\sum}}\frac{K_{h}(\pi(X_{j})-A_{j})}{f_{0}(A_{j}|X_{j})}}\right|\\
\leq&\frac{1}{1-\gamma}L_{h}|\alpha_{1}-\alpha_{2}|,\quad\text{where $L_{h}=\bigg(1+\frac{2M_{K}}{h\epsilon}\bigg)\frac{M M_{K}}{h\epsilon\underline{\alpha}^{2}}$.}
\end{aligned}
\end{equation*}
Consequently, we have the following result:
%{\small
\begin{equation}
\begin{aligned}\label{eqt:uniform convergence result}
&\sqrt{Nh}(\hat{W}_{N}^{h}-\mathbb{E}[W^{h}(\pi;\cdot)])=\sqrt{Nh}\Bigg(\frac{\frac{1}{N}\underset{i=1}{\overset{N}{\sum}}W_{i}^{h}(\pi,\cdot)}{\frac{1}{N}\underset{j=1}{\overset{N}{\sum}}\frac{K_{h}(\pi(X_{j})-A_{j})}{f_{0}(A_{j}|X_{j})}}-\mathbb{E}[W^{h}(\pi;\cdot)]\Bigg)\overset{d}{\rightarrow}Z^{h}(\cdot)
\end{aligned}
\end{equation}
%}\noindent
uniformly in Banach space $\mathcal{C}([\underline{\alpha},\bar{\alpha}])$ of continuous function $\psi:[\underline{\alpha},\bar{\alpha}]\rightarrow\mathbb{R}$ equipped with the sup norm $\|\psi\|:=\underset{x\in[\underline{\alpha},\bar{\alpha}]}{\sup}|\psi(x)|$ (see \cite{araujo1980central}). 

\noindent Define the functional
\begin{equation*}
\begin{aligned}
G(\psi,\alpha)=\alpha\log\psi(\alpha)+\alpha\eta\quad\text{and}\quad V(\psi)=\underset{\alpha\in[\underline{\alpha},\bar{\alpha}]}{\inf}G(\psi,\alpha).
\end{aligned}
\end{equation*}
By the Danskin's Theorem, $V(\cdot)$ is the directional differentiable at any $\mu\in\mathcal{C}([\underline{\alpha},\bar{\alpha}])$, (which is denoted as $V_{\mu}^{'}(\cdot)$), such that
\begin{equation*}
\begin{aligned}
V_{\mu}^{'}(\nu)=\underset{\alpha\in \bar{X}(\mu)}{\inf}\;\frac{\alpha}{\mu(\alpha)}\nu(\alpha),\quad\forall\nu\in\mathcal{C}([\underline{\alpha},\bar{\alpha}])
\end{aligned}
\end{equation*}
where $\bar{X}(\mu)=\underset{\alpha\in[\underline{\alpha},\bar{\alpha}]}{\arg\min}\;\{\alpha\log(\mu(\alpha))+\alpha\eta\}$. Here, $V_{\mu}^{'}(\nu)$ is the directional derivative of $V(\mu)$ at $\mu$ in the direction of $\nu$. Notice that for those $\psi\in\mathcal{C}[\underline{\alpha},\bar{\alpha}]$ such that $\underset{x\in[\underline{\alpha},\bar{\alpha}]}{\min}\psi(x)\geq m>0$, $V(\cdot)$ is a Lipschitz continuous function. The corresponding proofs are given as follows: for $\psi$ and  $\bar{\psi}$ lies $\mathcal{A}=\{\psi\in\mathcal{C}([\underline{\alpha},\bar{\alpha}]):\underset{x\in[\underline{\alpha},\bar{\alpha}]}{\min}\psi(x)\geq m>0\}$, we have
%{\small
\begin{equation*}
\begin{aligned}
V(\psi)&=\underset{\alpha\in[\underline{\alpha},\bar{\alpha}]}{\inf}\{\alpha\log\psi(\alpha)+\alpha\eta\}\\
&=\underset{\alpha\in[\underline{\alpha},\bar{\alpha}]}{\inf}\{\alpha\log\bar{\psi}(\alpha)+\alpha\eta+\alpha\log\psi(\alpha)-\alpha\log\bar{\psi}(\alpha)\}\\
&\geq \underset{\alpha\in[\underline{\alpha},\bar{\alpha}]}{\inf}\{\alpha\log\bar{\psi}(\alpha)+\alpha\eta\}+\underset{\alpha\in[\underline{\alpha},\bar{\alpha}]}{\inf}\{\alpha\log\psi(\alpha)-\alpha\log\bar{\psi}(\alpha)\}.
\end{aligned}
\end{equation*}
Hence, we have
\begin{equation*}
\begin{aligned}
&V(\bar{\psi})-V(\psi)\leq \underset{\alpha\in[\underline{\alpha},\bar{\alpha}]}{\sup}\{\alpha\log\bar{\psi}(\alpha)-\alpha\log\psi(\alpha)\}\\
&\leq \underset{\alpha\in[\underline{\alpha},\bar{\alpha}]}{\sup}\bigg\{\frac{\alpha}{\beta(\alpha)}[\bar{\psi}(\alpha)-\psi(\alpha)]\bigg\}\leq \frac{\bar{\alpha}}{m}\underset{\alpha\in[\underline{\alpha},\bar{\alpha}]}{\sup}[\bar{\psi}(\alpha)-\psi(\alpha)]\leq\frac{\bar{\alpha}}{m}\|\bar{\psi}-\psi\|.
\end{aligned}
\end{equation*}
%}\noindent
Here, $\beta(\alpha)$ lies between $\bar{\psi}(\alpha)$ and $\psi(\alpha)$. Similarly, we can also show that $V(\psi)-V(\bar{\psi})\leq\frac{\bar{\alpha}}{m}\|\psi-\bar{\psi}\|$. As such, we conclude that $|V(\psi)-V(\bar{\psi})|\leq\frac{\bar{\alpha}}{m}\|\psi-\bar{\psi}\|$ and $V(\cdot)$ is a Lipschitz continuous function for those $\psi\in\mathcal{C}([\underline{\alpha},\bar{\alpha}])$ such that $\underset{x\in[\underline{\alpha},\bar{\alpha}]}{\min}\psi(x)\geq m>0$.

\noindent Note that $\mathbb{E}\bigg[e^{-\frac{Y(\pi(X))}{\alpha}}\bigg]\geq e^{-\frac{M}{\underline{\alpha}}}>0$.
%\begin{equation*}
%\begin{aligned}
%\mathbb{E}\bigg[e^{-\frac{Y(\pi(X))}{\alpha}}\bigg]\geq e^{-\frac{M}{\underline{\alpha}}}>0.
%\end{aligned}
%\end{equation*}
Besides, we have $0\leq\mathbb{E}[W^{h}(\pi,\alpha)]=\mathbb{E}\bigg[e^{-\frac{Y(\pi(X))}{\alpha}}\bigg]+O(h^{2})$.
%\begin{equation*}
%\begin{aligned}
%0\leq\mathbb{E}[W^{h}(\pi,\alpha)]=\mathbb{E}\bigg[e^{-\frac{Y(\pi(X))}{\alpha}}\bigg]+O(h^{2}).
%\end{aligned}
%\end{equation*}
When $h\rightarrow 0$, $\mathbb{E}[W^{h}(\pi,\alpha)]\rightarrow\mathbb{E}\big[e^{-\frac{Y(\pi(X))}{\alpha}}\big]>0$. Hence, for sufficiently small $h$, $\mathbb{E}[W^{h}(\pi,\alpha)]>0$ and we can conclude that $V(\cdot)$ is Hadamard differentiable at $\mathbb{E}\bigg[e^{-\frac{Y(\pi(X))}{\alpha}}\bigg]$ and $\mathbb{E}[W^{h}(\pi,\alpha)]$. By the Functional Delta Theorem, we have
\begin{equation*}
\begin{aligned}
\sqrt{Nh}(V(\hat{W}^{h}(\pi,\cdot))-V(\mathbb{E}[W^{h}(\pi,\cdot)]))\overset{d}{\rightarrow}V_{\mathbb{E}[W^{h}(\pi,\cdot)]}^{'}(Z^{h}).
\end{aligned}
\end{equation*}
According to Claim \ref{lemma:prop of fun phi2} of Proposition \ref{lemma:finite optimal study}, we see that $\alpha\log(\mathbb{E}[W^{h}(\pi,\alpha)])+\alpha\eta$ is a strictly convex function for $\alpha>0$. Further, when $h\rightarrow 0$, we have
{\small
\begin{equation*}
\begin{aligned}
&V_{\mathbb{E}[W^{h}(\pi,\cdot)]}^{'}(Z^{h})\rightarrow\\
&\frac{\alpha_{\ast}(\pi)Z^{0}(\alpha_{\ast}(\pi))}{\mathbb{E}\bigg[\exp\bigg(-\frac{Y(\pi(X))}{\alpha_{\ast}(\pi)}\bigg)\bigg]}=\mathcal{N}\bigg(0,\frac{\alpha_{\ast}^{2}(\pi)\mathbb{V}_{\pi}(\alpha_{\ast}(\pi))}{\bigg(\mathbb{E}\bigg[\exp\bigg(-\frac{Y(\pi(X))}{\alpha_{\ast}(\pi)}\bigg)\bigg]\bigg)^{2}}\bigg),
\end{aligned}
\end{equation*}
}\noindent
where
%{\small
\begin{equation*}
\begin{aligned}
\mathbb{V}_{\pi}(\alpha_{\ast}(\pi))&=\big(\int K^{2}(u)du\big)\\
&\quad\times\bigg\{\mathbb{E}\bigg[\mathbb{E}\bigg[\frac{e^{-\frac{2Y}{\alpha_{\ast}(\pi)}}}{f_{0}(\pi(X)|X)}\bigg|A=\pi(X),X\bigg]\bigg]\\
&\quad\quad\quad+\mathbb{E}\bigg[\frac{1}{f_{0}(\pi(X)|X)}\bigg](\mathbb{E}[e^{-\frac{Y(\pi(X))}{\alpha_{\ast}(\pi)}}])^{2}\\
&\quad\quad\quad-2\mathbb{E}\bigg[\mathbb{E}\bigg[\frac{e^{-\frac{Y}{\alpha_{\ast}(\pi)}}}{f_{0}(\pi(X)|X)}\bigg|A=\pi(X),X\bigg]\bigg]\mathbb{E}[e^{-\frac{Y(\pi(X))}{\alpha_{\ast}(\pi)}}]\bigg\}.
\end{aligned}
\end{equation*}
%}\noindent
Besides, we simplify the quantity $\sqrt{Nh}(\mathbb{E}[W^{h}(\pi,\alpha)]-\mathbb{E}[e^{-\frac{Y(\pi(X))}{\alpha}}])$, the derivations are presented as follows: since
%{\small
\begin{equation*}
\begin{aligned}
&\mathbb{E}[W^{h}(\pi,\alpha)]=\mathbb{E}\bigg[\mathbb{E}\bigg[\frac{K_{h}(\pi(X)-A)e^{-\frac{Y}{\alpha}}}{f_{0}(A|X)}\bigg|X\bigg]\bigg]\\
%=&\mathbb{E}\bigg[\int\mathbb{E}\bigg[\frac{K_{h}(\pi(X)-A)e^{-\frac{Y}{\alpha}}}{f_{0}(A|X)}\bigg|Y=y,A=a,X\bigg]f_{0}(y,a|X)dyda\bigg]\\
%=&\mathbb{E}\bigg[\int K(u)e^{-\frac{y}{\alpha}}\bigg\{f_{0}(y|\pi(X),X)+\partial_{a}f_{0}(y|\pi(X),X)uh+\frac{\partial_{aa}^{2}f_{0}(y|\pi(X),X)}{2}u^{2}h^{2}+O_{P}(h^{3})\bigg\}dydu\bigg]\\
%=&\mathbb{E}\bigg[\int e^{-\frac{y}{\alpha}}f_{0}(y|\pi(X),X)dy\bigg]+\frac{\big(\int u^{2}K(u)du\big)}{2}\mathbb{E}\bigg[\int e^{-\frac{y}{\alpha}}\partial_{aa}^{2}f_{0}(y|\pi(X),X)dy\bigg]h^{2}+O(h^{3})\\
=&\mathbb{E}\bigg[\mathbb{E}\bigg[e^{-\frac{Y}{\alpha}}\bigg|A=\pi(X),X\bigg]\bigg]+h^{2}\times\\
&\underbrace{\frac{\big(\int u^{2}K(u)du\big)}{2}\mathbb{E}\bigg[\mathbb{E}\bigg[ e^{-\frac{Y}{\alpha}}\frac{\partial_{aa}^{2}f_{0}(y|\pi(X),X)}{f_{0}(y|\pi(X),X)}\bigg|A=\pi(X),X\bigg]\bigg]}_{\coloneqq B_{\pi}(\alpha)}+O(h^{3})\\
=&\mathbb{E}\bigg[e^{-\frac{Y(\pi(X))}{\alpha}}\bigg]+B_{\pi}(\alpha)h^{2}+O(h^{3}),
\end{aligned}
\end{equation*}
%}\noindent
we conclude that
%{\small
\begin{equation*}
\begin{aligned}
\sqrt{Nh}(\mathbb{E}[W^{h}(\pi,\alpha)]-\mathbb{E}[e^{-\frac{Y(\pi(X))}{\alpha}}])=\sqrt{Nh}(B_{\pi}^{w}(\alpha)h^{2}+O(h^{3})).
\end{aligned}
\end{equation*}
%}\noindent
Under the given convergence assumptions (i.e., $N\rightarrow \infty$, $h\rightarrow 0$, $Nh\rightarrow \infty$ and $Nh^{5}\rightarrow C\in[0,\infty)$), we notice that $\sqrt{Nh}(\mathbb{E}[W^{h}(\pi,\alpha)]-\mathbb{E}[e^{-\frac{Y(\pi(X))}{\alpha}}])$ is Lipschitz continuous w.r.t. $\alpha$ when $\alpha\in[\underline{\alpha},\bar{\alpha}]$, and we can thus conclude that
\begin{equation*}
\begin{aligned}
\sqrt{Nh}(\mathbb{E}[W^{h}(\pi,\cdot)]-\mathbb{E}[e^{-\frac{Y(\pi(X))}{\cdot}}])&\overset{d}{\rightarrow}\sqrt{Nh}B_{\pi}^{w}(\cdot)h^{2}.
\end{aligned}
\end{equation*}
%As proven in Lemma \ref{lemma:asy_res_unnormalized}, we obtain
%\begin{equation*}
%\begin{aligned}
%\sqrt{Nh}(\mathbb{E}[W^{h}(\pi,\cdot)]-\mathbb{E}[e^{-\frac{Y(\pi(X))}{\cdot}}])&=\sqrt{Nh}(\mathbb{E}[e^{-\frac{Y(\pi(X))}{\cdot}}]+B_{\pi}(\cdot)h^{2}+O(h^{3})-\mathbb{E}[e^{-\frac{Y(\pi(X))}{\cdot}}])\\
%&=\sqrt{Nh}(B_{\pi}^{w}(\cdot)h^{2}+O(h^{3})).
%\end{aligned}
%\end{equation*}
Hence, by applying the Functional Delta Theorem again, we have
\begin{equation*}
\begin{aligned}
&\sqrt{Nh}(V(\mathbb{E}[W^{h}(\pi,\cdot)])-V(\mathbb{E}[e^{-\frac{Y(\pi(X))}{\cdot}}]))\overset{d}{\rightarrow}V_{\mathbb{E}\big[e^{-\frac{Y(\pi(X))}{\cdot}}\big]}^{'}(\sqrt{Nh}(B_{\pi}(\cdot)h^{2}),
\end{aligned}
\end{equation*}
where
\begin{equation*}
\begin{aligned}
&V_{\mathbb{E}\big[e^{-\frac{Y(\pi(X))}{\cdot}}\big]}^{'}(\sqrt{Nh}(B_{\pi}(\cdot)h^{2})=\sqrt{Nh}\frac{\alpha_{\ast}(\pi)}{\mathbb{E}\bigg[\exp\bigg(-\frac{Y(\pi(X))}{\alpha_{\ast}(\pi)}\bigg)\bigg]}B_{\pi}(\alpha_{\ast}(\pi))h^{2}\\
\end{aligned}
\end{equation*}
We rewrite $\hat{Q}_{\text{DRO}}^{h}(\pi)$ as follows:
%{\small
\begin{equation*}
\begin{aligned}
\hat{Q}_{\text{DRO}}^{h}(\pi)&=\underset{\alpha\geq 0}{\max}\;\Bigg\{-\alpha\log\frac{\frac{1}{N}\underset{i=1}{\overset{N}{\sum}}W_{i}^{h}(\pi,\alpha)}{\frac{1}{N}\underset{i=1}{\overset{N}{\sum}}\frac{K_{h}(\pi(X_{i})-A_{i})}{f_{0}(A_{i}|X_{i})}}-\alpha\eta\Bigg\}=-\underset{\alpha\geq 0}{\min}\;\Bigg\{\alpha\log\frac{\frac{1}{N}\underset{i=1}{\overset{N}{\sum}}W_{i}^{h}(\pi,\alpha)}{\frac{1}{N}\underset{i=1}{\overset{N}{\sum}}\frac{K_{h}(\pi(X_{i})-A_{i})}{f_{0}(A_{i}|X_{i})}}+\alpha\eta\Bigg\}.
\end{aligned}
\end{equation*}
%}\noindent
Similarly, we also rewrite $Q_{\text{DRO}}(\pi)$ as follows:
%{\small
\begin{equation*}
\begin{aligned}
Q_{\text{DRO}}(\pi)&=\underset{\alpha\geq 0}{\max}\;\bigg\{-\alpha\log \mathbb{E}\bigg[e^{-\frac{Y(\pi(X))}{\alpha}}\bigg]-\alpha\eta\bigg\}=-\underset{\alpha\geq 0}{\min}\;\bigg\{\alpha\log \mathbb{E}\bigg[e^{-\frac{Y(\pi(X))}{\alpha}}\bigg]+\alpha\eta\bigg\}.
\end{aligned}
\end{equation*}
%}\noindent
The optimal solution of $Q_{\text{DRO}}(\pi)$ is $\alpha_{\ast}(\pi)$ which is finite according to Claim \ref{lemma:prop of fun phi} of Proposition \ref{lemma:finite optimal study}. According to our definition of $V(\psi)$, we have $Q_{\text{DRO}}(\pi)=-V\bigg(\mathbb{E}\bigg[e^{-\frac{Y(\pi(X))}{\alpha}}\bigg]\bigg)$.

\noindent Our objective is showing that 
\begin{equation*}
\begin{aligned}
\mathbb{P}\{\hat{Q}_{\text{DRO}}^{h}(\pi)\neq -V(\hat{W}^{h}(\pi,\alpha))\}\rightarrow 0
\end{aligned}
\end{equation*}
under the convergence conditions (i.e., $N\rightarrow\infty$, $h\rightarrow 0$, $Nh\rightarrow\infty$, and $Nh^{5}\rightarrow C\in[0,\infty)$). First, the convergence of 
%{\small
\begin{equation*}
\begin{aligned}
&\sqrt{Nh}(\hat{W}_{N}^{h}-\mathbb{E}[W^{h}(\pi;\alpha)])=\sqrt{Nh}\bigg(\frac{\frac{1}{N}\underset{i=1}{\overset{N}{\sum}}W_{i}^{h}(\pi,\alpha)}{\frac{1}{N}\underset{j=1}{\overset{N}{\sum}}\frac{K_{h}(\pi(X_{j})-A_{j})}{f_{0}(A_{j}|X_{j})}}-\mathbb{E}[W^{h}(\pi;\alpha)]\bigg)\overset{d}{\rightarrow}Z^{h}(\alpha)
\end{aligned}
\end{equation*}
%}\noindent
also implies the uniform convergence
%{\small
\begin{equation*}
\begin{aligned}
&\underset{\alpha\in[\underline{\alpha},\bar{\alpha}]}{\sup}\bigg|\hat{W}^{h}(\pi;\alpha)-\mathbb{E}[W^{h}(\pi;\alpha)]\bigg|=\underset{\alpha\in[\underline{\alpha},\bar{\alpha}]}{\sup}\bigg|\frac{\frac{1}{N}\underset{i=1}{\overset{N}{\sum}}W_{i}^{h}(\pi,\alpha)}{\frac{1}{N}\underset{j=1}{\overset{N}{\sum}}\frac{K_{h}(\pi(X_{j})-A_{j})}{f_{0}(A_{j}|X_{j})}}-\mathbb{E}[W^{h}(\pi;\alpha)]\bigg|\rightarrow 0\quad \text{a.s..}
\end{aligned}
\end{equation*}
%}\noindent
As a result, we can show that
%{\small
\begin{equation*}
\begin{aligned}
\underset{\alpha\in[\underline{\alpha},\bar{\alpha}]}{\sup}\big|\big(\alpha\log\hat{W}^{h}(\pi;\alpha)+\alpha\eta\big)-\big(\alpha\log\mathbb{E}[W^{h}(\pi;\alpha)]+\alpha\eta\big)\big|\rightarrow 0
%\quad \text{a.s..}
\end{aligned}
\end{equation*}
%}\noindent
almost surely. The derivations are as follows:
%{\small
\begin{equation*}
\begin{aligned}
&\bigg|\bigg(\alpha\log\hat{W}^{h}(\pi;\alpha)+\alpha\eta\bigg)-\bigg(\alpha\log\mathbb{E}[W^{h}(\pi;\alpha)]+\alpha\eta\bigg)\bigg|\\
=&\bigg|\alpha\log\hat{W}^{h}(\pi;\alpha)-\alpha\log\mathbb{E}[W^{h}(\pi;\alpha)]\bigg|=|\alpha|\bigg|\log\hat{W}^{h}(\pi;\alpha)-\log\mathbb{E}[W^{h}(\pi;\alpha)]\bigg|\\
\leq&|\alpha|\bigg|\log\bar{W}^{h}(\pi;\alpha)-\log\mathbb{E}[W^{h}(\pi;\alpha)]\bigg|+|\alpha|\bigg|\log\frac{1}{N}\underset{j=1}{\overset{N}{\sum}}\frac{K_{h}(\pi(X_{j})-A_{j})}{f_{0}(A_{j}|X_{j})}\bigg|\\
\leq&|\alpha|\bigg|\log\frac{1}{N}\underset{j=1}{\overset{N}{\sum}}\frac{K_{h}(\pi(X_{j})-A_{j})}{f_{0}(A_{j}|X_{j})}\bigg|+|\alpha|\bigg|\frac{1}{\beta(\alpha)}\bigg||\bar{W}^{h}(\pi;\alpha)-\mathbb{E}[W^{h}(\pi;\alpha)]|\\
\leq&|\alpha|\bigg|\log\frac{1}{N}\underset{j=1}{\overset{N}{\sum}}\frac{K_{h}(\pi(X_{j})-A_{j})}{f_{0}(A_{j}|X_{j})}\bigg|+\frac{|\bar{\alpha}|}{\tilde{M}}|\bar{W}^{h}(\pi;\alpha)-\mathbb{E}[W^{h}(\pi;\alpha)]|\quad \text{for some $\tilde{M}>0$.}
\end{aligned}
\end{equation*}
%}\noindent
Here, $\beta(\alpha)$ lies in between $\bar{W}^{h}(\pi;\alpha)$ and $\mathbb{E}[W^{h}(\pi;\alpha)]$. To justify the last inequality, note that 
\begin{equation*}
\begin{aligned}
0<\tilde{M}=&\underset{\alpha\in[\underline{\alpha},\bar{\alpha}]}{\inf}\min\{\bar{W}^{h}(\pi;\alpha),\mathbb{E}[W^{h}(\pi;\alpha)]\}\leq\min\{\bar{W}^{h}(\pi;\alpha),\mathbb{E}[W^{h}(\pi;\alpha)]\}\leq \beta(\alpha)
\end{aligned}
\end{equation*}
and
\begin{equation*}
\begin{aligned}
\beta(\alpha)&\leq\max\{\bar{W}^{h}(\pi;\alpha),\mathbb{E}[W^{h}(\pi;\alpha)]\}\leq\underset{\alpha\in[\underline{\alpha},\bar{\alpha}]}{\sup}\max\{\bar{W}^{h}(\pi;\alpha),\mathbb{E}[W^{h}(\pi;\alpha)]\}.
\end{aligned}
\end{equation*}
%\begin{alignat*}{2}
%0<\tilde{M}=&\underset{\alpha\in[\underline{\alpha},\bar{\alpha}]}{\inf}\min\{\bar{W}^{h}(\pi;\alpha),\mathbb{E}[W^{h}(\pi;\alpha)]\}\\
%&\leq\min\{\bar{W}^{h}(\pi;\alpha),\mathbb{E}[W^{h}(\pi;\alpha)]\}\leq \beta(\alpha)&&\leq\max\{\bar{W}^{h}(\pi;\alpha),\mathbb{E}[W^{h}(\pi;\alpha)]\}\\
%&\quad&&\leq\underset{\alpha\in[\underline{\alpha},\bar{\alpha}]}{\sup}\max\{\bar{W}^{h}(\pi;\alpha),\mathbb{E}[W^{h}(\pi;\alpha)]\}.
%\end{alignat*}
Together with the result from Claim \ref{lemma:useful asy result} of Proposition \ref{result:convergence S_N_h} where $\frac{1}{N}\underset{j=1}{\overset{N}{\sum}}\frac{K_{h}(\pi(X_{j})-A_{j})}{f_{0}(A_{j}|X_{j})}\overset{a.s.}{\rightarrow} 1$ when $N\rightarrow\infty$ and $h\rightarrow 0$, we conclude that
%{\small
\begin{equation*}
\begin{aligned}
&\underset{\alpha\in[\underline{\alpha},\bar{\alpha}]}{\sup}\big|\big(\alpha\log\hat{W}^{h}(\pi;\alpha)+\alpha\eta\big)-\big(\alpha\log\mathbb{E}[W^{h}(\pi;\alpha)]+\alpha\eta\big)\big|\\
\leq &\frac{|\bar{\alpha}|}{\tilde{M}}\underset{\alpha\in[\underline{\alpha},\bar{\alpha}]}{\sup}|\bar{W}^{h}(\pi;\alpha)-\mathbb{E}[W^{w;h}(\pi;\alpha)]|+\underset{\alpha\in[\underline{\alpha},\bar{\alpha}]}{\sup}|\alpha|\bigg|\log\frac{1}{N}\underset{j=1}{\overset{N}{\sum}}\frac{K_{h}(\pi(X_{j})-A_{j})}{f_{0}(A_{j}|X_{j})}\bigg|\rightarrow 0\quad\textit{a.s..}
\end{aligned}
\end{equation*}
%}\noindent
The above result means that, for arbitrary $\epsilon>0$, given the event
{\small
\begin{equation*}
\begin{aligned}
&\bigg\{\underset{\alpha\in[\underline{\alpha},\bar{\alpha}]}{\sup}\big|\big(\alpha\log\hat{W}^{h}(\pi;\alpha)+\alpha\eta\big)-\big(\alpha\log\mathbb{E}[W^{h}(\pi;\alpha)]+\alpha\eta\big)\big|\leq \epsilon\bigg\},
%\quad\text{we have}
\end{aligned}
\end{equation*}
}\noindent
we have
%{\small
\begin{align}
\begin{aligned}\label{eqt:W_hat and W_exp}
&\bigg(\alpha\log\mathbb{E}[W^{h}(\pi;\alpha)]+\alpha\eta\bigg)-\epsilon\leq \bigg(\alpha\log\hat{W}^{h}(\pi;\alpha)+\alpha\eta\bigg)\leq \bigg(\alpha\log\mathbb{E}[W^{h}(\pi;\alpha)]+\alpha\eta\bigg)+\epsilon.
\end{aligned}\raisetag{15pt}
\end{align}
%}\noindent
Suppose that
\begin{equation*}
\begin{aligned}
\alpha_{\ast}^{h;\mathrm{IPW}}(\pi)=\underset{\alpha\geq 0}{\arg\min}\;\bigg\{\alpha\log\mathbb{E}[W^{h}(\pi,\alpha)]+\alpha\eta\bigg\},
\end{aligned}
\end{equation*}
we therefore have that 
%{\small
\begin{equation*}
\begin{aligned}
&\alpha_{\ast}^{h;\mathrm{IPW}}(\pi)\log\mathbb{E}[W^{h}(\pi,\alpha_{\ast}^{h;\mathrm{IPW}}(\pi))]+\alpha_{\ast}^{h;\mathrm{IPW}}(\pi)\eta\\
&<\min\bigg\{\underline{\alpha}\log\mathbb{E}[W^{h}(\pi,\underline{\alpha})]+\underline{\alpha}\eta,\;\bar{\alpha}\log\mathbb{E}[W^{h}(\pi,\bar{\alpha})]+\bar{\alpha}\eta\bigg\}.
\end{aligned}
\end{equation*}
%}\noindent
From Eqn. \eqref{eqt:W_hat and W_exp}, for sufficiently small $\epsilon>0$ such that $\epsilon$ is negligible, we can treat
{\small
\begin{equation*}
\begin{aligned}
\bigg(\alpha\log\hat{W}^{h}(\pi;\alpha)+\alpha\eta\bigg)\approx \bigg(\alpha\log\mathbb{E}[W^{h}(\pi;\alpha)]+\alpha\eta\bigg).
\end{aligned}
\end{equation*}
}\noindent
In addition, for small $h$ such that the convergence conditions still hold, $\alpha_{\ast}^{h;\mathrm{IPW}}(\pi)$ should lie in $[\underline{\alpha},\bar{\alpha}]$ since $\alpha_{\ast}^{h;\mathrm{IPW}}(\pi)\rightarrow\alpha_{\ast}(\pi)$. As a consequence, we also have
%{\small
\begin{equation*}
\begin{aligned}
&\alpha_{\ast}^{h;\mathrm{IPW}}(\pi)\log \hat{W}^{h}(\pi,\alpha_{\ast}^{h;\mathrm{IPW}}(\pi))+\alpha_{\ast}^{h;\mathrm{IPW}}(\pi)\eta\\
&<\min\bigg\{\underline{\alpha}\log \hat{W}^{h}(\pi,\underline{\alpha})+\underline{\alpha}\eta,\;\bar{\alpha}\log \hat{W}^{h}(\pi,\bar{\alpha})+\bar{\alpha}\eta\bigg\}.
\end{aligned}
\end{equation*}
%}\noindent
Due to the fact that $\alpha\log \hat{W}^{h}(\pi,\alpha)+\alpha\eta$ is a convex function, we can conclude that $\hat{Q}_{\text{DRO}}^{h}(\pi)=-V(\hat{W}^{h}(\pi,\alpha))$ under the convergence conditions:
%{\small
\begin{equation*}
\begin{aligned}
&\sqrt{Nh}(\hat{Q}_{\text{DRO}}^{h}(\pi)-Q_{\text{DRO}}(\pi))=\sqrt{Nh}(\hat{Q}_{\text{DRO}}^{h}(\pi)+V\bigg(\mathbb{E}\bigg[e^{-\frac{Y(\pi(X))}{\alpha}}\bigg]\bigg))\\
&=\sqrt{Nh}(\hat{Q}_{\text{DRO}}^{h}(\pi)+V(\hat{W}^{h}(\pi,\alpha)))\\
&\quad-\sqrt{Nh}(V(\hat{W}^{h}(\pi,\alpha))-V(\mathbb{E}[\hat{W}^{h}(\pi,\alpha)]))\\
&\quad-\sqrt{Nh}(V(\mathbb{E}[\hat{W}^{h}(\pi,\alpha)])-V\bigg(\mathbb{E}\bigg[e^{-\frac{Y(\pi(X))}{\alpha}}\bigg]\bigg))\\
\Rightarrow &\sqrt{Nh}(\hat{Q}_{\text{DRO}}^{h}(\pi)-Q_{\text{DRO}}(\pi)+V(\mathbb{E}[\hat{W}^{h}(\pi,\alpha)])-V\bigg(\mathbb{E}\bigg[e^{-\frac{Y(\pi(X))}{\alpha}}\bigg]\bigg))\\
&=\sqrt{Nh}(\hat{Q}_{\text{DRO}}^{h}(\pi)+V(\hat{W}^{h}(\pi,\alpha)))-\sqrt{Nh}(V(\hat{W}^{h}(\pi,\alpha))-V(\mathbb{E}[\hat{W}^{h}(\pi,\alpha)])).
\end{aligned}
\end{equation*}
%}\noindent
Ultimately, using the Slutsky's Theorem, we have
%{\small
\begin{equation*}
\begin{aligned}
&\sqrt{Nh}\Biggl(\hat{Q}_{\text{DRO}}^{h}(\pi)-Q_{\text{DRO}}(\pi)+\frac{\alpha_{\ast}(\pi)B_{\pi}(\alpha_{\ast}(\pi))h^{2}}{\mathbb{E}\big[\exp\big(-\frac{Y(\pi(X))}{\alpha_{\ast}(\pi)}\big)\big]}\Biggl)\overset{d}{\rightarrow}\mathcal{N}\left(0,\frac{\alpha_{\ast}^{2}(\pi)\mathbb{V}_{\pi}(\alpha_{\ast}(\pi))}{\bigg(\mathbb{E}\bigg[\exp\bigg(-\frac{Y(\pi(X))}{\alpha_{\ast}(\pi)}\bigg)\bigg]\bigg)^{2}}\right).
\end{aligned}
\end{equation*}
%}\noindent
\end{proof}
\subsubsection{Proof of Theorem \ref{thm:statistical performance normalized}}\label{sec:Theorem main 2}
%The proofs of Theorem \ref{thm:statistical performance normalized} are stated as follows:
\begin{theorem*}\label{thm:statistical performance normalized appendix}
Suppose that the kernel function $K(x)$ is bounded where $|K(x)|\leq M_{K}$. Given $\delta>0$, $h>0$, and a policy class $\Pi$, denote 
%{\small
\begin{equation*}
\begin{gathered}
\mathcal{F}_{\Pi}\coloneqq\Bigg\{\frac{K_{h}(\pi(X)-A)}{f_{0}(A|X)}:\pi\in\Pi\Bigg\},\\
\mathcal{F}_{\Pi,x}\coloneqq\Bigg\{\frac{K_{h}(\pi(X)-A)\mathbf{1}_{\{Y(\pi(X))\leq x\}}}{f_{0}(A|X)}:\pi\in\Pi,\;x\in[0,M]\Bigg\}.
\end{gathered}
\end{equation*}
%}\noindent
%{\color{blue}
Denote $\alpha_{\cdot}^{\ast}(\pi;h)=\underset{\alpha\geq 0}{\sup}\{-\alpha\log\mathbb{E}_{\mathbb{P}_{N}^{h}}[e^{-\frac{\cdot}{\alpha}}]-\alpha\eta\}$ where $\mathbb{E}_{\mathbb{P}_{N}^{h}}[\mathcal{Z}]=\underset{i=1}{\overset{N}{\sum}}\frac{\frac{K_{h}(\pi(X_{i})-A_{i})}{f_{0}(A_{i}|X_{i})}}{\underset{j=1}{\overset{N}{\sum}}\frac{K_{h}(\pi(X_{j})-A_{j})}{f_{0}(A_{j}|X_{j})}}\mathcal{Z}_{i}$. 
If $\underset{\pi\in\Pi}{\sup}|\alpha_{Y}^{\ast}(\pi;h)-\alpha_{Y(\pi(X))}^{\ast}(\pi;h)|=o(h)$, then
%}
with probability $1-\delta$, we have
%{\color{blue}
%{\small
\begin{gather*}
\begin{aligned}
R_{\mathrm{DRO}}(\hat{\pi}_{\mathrm{DRO}}^{h})\leq &\frac{4}{\epsilon}\mathcal{R}_{N}(\mathcal{F}_{\Pi,x})+\frac{4}{\epsilon}\mathcal{R}_{N}(\mathcal{F}_{\Pi})+\frac{4\sqrt{2}M_{K}\sqrt{\ln\big(\frac{2}{\delta}\big)}}{h\epsilon^{2}\sqrt{N}}+O(h^{2}).
\end{aligned}
\end{gather*}
%}\noindent
%}\noindent
%{\small
%\begin{gather}
%\begin{aligned}\label{eqt:statistical performance normalized Rademacher results}
%R_{\mathrm{DRO}}(\hat{\pi}_{\mathrm{DRO}}^{h})\leq &2M+\frac{4}{\epsilon}\mathcal{R}_{N}(\mathcal{F}_{\Pi,x})+\frac{4}{\epsilon}\mathcal{R}_{N}(\mathcal{F}_{\Pi})\\
%&\;+\frac{4\sqrt{2}M_{K}\sqrt{\ln\big(\frac{2}{\delta}\big)}}{h\epsilon^{2}\sqrt{N}}+O(h^{2}).
%\end{aligned}
%\end{gather}
%}\noindent
\end{theorem*}
\begin{proof}
Recall that $R_{\mathrm{DRO}}(\pi)=Q_{\mathrm{DRO}}(\pi_{\mathrm{DRO}}^{\ast})-Q_{\mathrm{DRO}}(\pi)$. Now, denote
\begin{equation*}
\begin{aligned}
\hat{Q}_{\mathrm{DRO}}^{h}(\pi)=\underset{\alpha\geq 0}{\sup}\Bigg\{-\alpha\log\frac{\frac{1}{N}\underset{i=1}{\overset{N}{\sum}}W_{i}^{h}(\pi,\alpha)}{\frac{1}{N}\underset{j=1}{\overset{N}{\sum}}\frac{K_{h}(\pi(X_{j})-A_{j})}{f_{0}(A_{j}|X_{j})}}-\alpha\eta\Bigg\}.
\end{aligned}
\end{equation*}
As {\small $\hat{\pi}_{\mathrm{DRO}}^{h}\in\underset{\pi\in\Pi}{\arg\max}\;\hat{Q}_{\mathrm{DRO}}^{h}(\pi)$}, we have {\small $\hat{Q}_{\mathrm{DRO}}^{h}(\hat{\pi}_{\mathrm{DRO}}^{h})\geq\hat{Q}_{\mathrm{DRO}}^{h}(\pi)$} for any $\pi\in\Pi$. In particular, $\hat{Q}_{\mathrm{DRO}}^{h}(\hat{\pi}_{\mathrm{DRO}}^{h})\geq\hat{Q}_{\mathrm{DRO}}^{h}(\pi_{\mathrm{DRO}}^{\ast})$. As a consequence, we can reformulate $R_{\mathrm{DRO}}(\hat{\pi}_{\mathrm{DRO}}^{h})$ as follows: 
%{\small
\begin{align}
&R_{\mathrm{DRO}}(\hat{\pi}_{\mathrm{DRO}}^{h})=Q_{\mathrm{DRO}}(\pi_{\mathrm{DRO}}^{\ast})-Q_{\mathrm{DRO}}(\hat{\pi}_{\mathrm{DRO}}^{h})\nonumber\\
&=Q_{\mathrm{DRO}}(\pi_{\mathrm{DRO}}^{\ast})-\hat{Q}_{\mathrm{DRO}}^{h}(\hat{\pi}_{\mathrm{DRO}}^{h})+\hat{Q}_{\mathrm{DRO}}^{h}(\hat{\pi}_{\mathrm{DRO}}^{h})-Q_{\mathrm{DRO}}(\hat{\pi}_{\mathrm{DRO}}^{h})\nonumber\\
&\leq Q_{\mathrm{DRO}}(\pi_{\mathrm{DRO}}^{\ast})-\hat{Q}_{\mathrm{DRO}}^{h}(\pi_{\mathrm{DRO}}^{\ast})+\hat{Q}_{\mathrm{DRO}}^{h}(\hat{\pi}_{\mathrm{DRO}}^{h})-Q_{\mathrm{DRO}}(\hat{\pi}_{\mathrm{DRO}}^{h})\nonumber\\
&\leq 2\underset{\pi\in\Pi}{\sup}\;|Q_{\mathrm{DRO}}(\pi)-\hat{Q}_{\mathrm{DRO}}^{h}(\pi)|\nonumber\\
&= 2\underset{\pi\in\Pi}{\sup}\;\bigg|\underset{\alpha\geq 0}{\sup}\Bigg\{-\alpha\log\frac{\frac{1}{N}\underset{i=1}{\overset{N}{\sum}}W_{i}^{h}(\pi,\alpha)}{\frac{1}{N}\underset{i=1}{\overset{N}{\sum}}\frac{K_{h}(\pi(X_{j})-A_{j})}{f_{0}(A_{j}|X_{j})}}-\alpha\eta\Bigg\}-\underset{\alpha\geq 0}{\sup}\{-\alpha\log\mathbb{E}[e^{-\frac{Y(\pi(X))}{\alpha}}]-\alpha\eta\}\bigg|.\label{eqt:RDO normalized}
\end{align}
%}\noindent
Denote the empirical measure $\mathbb{P}_{N}$ and the empirical weighted measure $\mathbb{P}_{N}^{h}$ such that 
%{\small
\begin{equation*}
\begin{aligned}
\mathbb{E}_{\mathbb{P}_{N}}[\mathcal{Z}]=\frac{1}{N}\underset{i=1}{\overset{N}{\sum}}\mathcal{Z}_{i}\quad \text{and}\quad\mathbb{E}_{\mathbb{P}_{N}^{h}}[\mathcal{Z}]=\underset{i=1}{\overset{N}{\sum}}\frac{\frac{K_{h}(\pi(X_{i})-A_{i})}{f_{0}(A_{i}|X_{i})}}{\underset{j=1}{\overset{N}{\sum}}\frac{K_{h}(\pi(X_{j})-A_{j})}{f_{0}(A_{j}|X_{j})}}\mathcal{Z}_{i}.
\end{aligned}
\end{equation*}
%}\noindent
Eqn. \eqref{eqt:RDO normalized} can be bounded as follows:
%{\small
\begin{subequations}
\begin{align}
&\begin{aligned}
\text{Eqn. \eqref{eqt:RDO normalized}}&=2\underset{\pi\in\Pi}{\sup}\;\bigg|\underset{\alpha\geq 0}{\sup}\{-\alpha\log\frac{\frac{1}{N}\underset{i=1}{\overset{N}{\sum}}\frac{K_{h}(\pi(X_{i})-A_{i})}{f_{0}(A_{i}|X_{i})}e^{-\frac{Y_{i}}{\alpha}}}{\frac{1}{N}\underset{j=1}{\overset{N}{\sum}}\frac{K_{h}(\pi(X_{j})-A_{j})}{f_{0}(A_{j}|X_{j})}}-\alpha\eta\}-\underset{\alpha\geq 0}{\sup}\{-\alpha\log\mathbb{E}[e^{-\frac{Y(\pi(X))}{\alpha}}]-\alpha\eta\}\bigg|
\end{aligned}\nonumber\\
&\begin{aligned}
\begin{split}
\leq&2\underset{\pi\in\Pi}{\sup}\;\bigg|\underset{\alpha\geq 0}{\sup}\{-\alpha\log\mathbb{E}_{\mathbb{P}_{N}^{h}}[e^{-\frac{Y}{\alpha}}]-\alpha\eta\}-\underset{\alpha\geq 0}{\sup}\{-\alpha\log\mathbb{E}_{\mathbb{P}_{N}^{h}}[e^{-\frac{Y(\pi(X))}{\alpha}}]-\alpha\eta\}\bigg|
\end{split}
\end{aligned}\label{eqt:IPW_bound_1 normalized}\\
&\begin{aligned}
\begin{split}
&+2\underset{\pi\in\Pi}{\sup}\;\bigg|\underset{\alpha\geq 0}{\sup}\{-\alpha\log\mathbb{E}_{\mathbb{P}_{N}^{h}}[e^{-\frac{Y(\pi(X))}{\alpha}}]-\alpha\eta\}-\underset{\alpha\geq 0}{\sup}\{-\alpha\log\mathbb{E}[e^{-\frac{Y(\pi(X))}{\alpha}}]-\alpha\eta\}\bigg|.
\end{split}
\end{aligned}\label{eqt:IPW_bound_2 normalized}
\end{align}
\end{subequations}
%}\noindent
We consider Eqns. \eqref{eqt:IPW_bound_1 normalized} - \eqref{eqt:IPW_bound_2 normalized} sequentially.
\paragraph{\underline{\textbf{Eqn. \eqref{eqt:IPW_bound_1 normalized}}:}} 
%{\color{blue}
According to the given conditions, we know that
%{\small
\begin{equation*}
\begin{aligned}
&\text{Eqn. \eqref{eqt:IPW_bound_1 normalized}}\leq 2o(h)=O(h^{2}).
\end{aligned}
\end{equation*}
%}\noindent
%}\noindent
%\paragraph{\underline{\textbf{Eqn. \eqref{eqt:IPW_bound_1 normalized}}:}} Note that
%{\small
%\begin{equation*}
%\begin{aligned}
%&\text{Eqn. \eqref{eqt:IPW_bound_1 normalized}}\leq2\underset{\pi\in\Pi}{\sup}\;\underset{\alpha\geq 0}{\sup}\;\alpha\big|\log\mathbb{E}_{\mathbb{P}_{N}^{h}}[e^{-\frac{Y}{\alpha}}]-\log\mathbb{E}_{\mathbb{P}_{N}^{h}}[e^{-\frac{Y(\pi(X))}{\alpha}}]\big|\\
%&\begin{aligned}
%\leq2\underset{\pi\in\Pi}{\sup}\;\underset{\alpha\geq 0}{\sup}\;&\alpha\left|\log\frac{1}{N}\underset{j=1}{\overset{N}{\sum}}\frac{K_{h}(\pi(X_{j})-A_{j})}{f_{0}(A_{j}|X_{j})}\right.\\
%&\quad\left.-\log\frac{1}{N}\underset{j=1}{\overset{N}{\sum}}\frac{K_{h}(\pi(X_{j})-A_{j})}{f_{0}(A_{j}|X_{j})}e^{-\frac{M}{\alpha}}\right|
%\end{aligned}\\
%&\begin{aligned}
%=2\underset{\pi\in\Pi}{\sup}\;\underset{\alpha\geq 0}{\sup}\;&\alpha\left|\log\frac{1}{N}\underset{j=1}{\overset{N}{\sum}}\frac{K_{h}(\pi(X_{j})-A_{j})}{f_{0}(A_{j}|X_{j})}\right.\\
%&\quad\left.-\log\frac{1}{N}\underset{j=1}{\overset{N}{\sum}}\frac{K_{h}(\pi(X_{j})-A_{j})}{f_{0}(A_{j}|X_{j})}-\log e^{-\frac{M}{\alpha}}\right|
%\end{aligned}\\
%&=2M.
%\end{aligned}
%\end{equation*}
%}\noindent
%\paragraph{\underline{\textbf{Eqn. \eqref{eqt:IPW_bound_2 normalized}}:}} Since $\underset{x}{\sup}f(x)=\underset{x}{\sup}\{f(x)-g(x)+g(x)\}\leq\underset{x}{\sup}\{f(x)-g(x)\}+\underset{x}{\sup}\;g(x)\leq\underset{x}{\sup}|f(x)-g(x)|+\underset{x}{\sup}\;g(x)$, we conclude that $|\underset{x}{\sup}f(x)-\underset{x}{\sup}g(x)|\leq\underset{x}{\sup}|f(x)-g(x)|$. Consequently, we can show that
\paragraph{\underline{\textbf{Eqn. \eqref{eqt:IPW_bound_2 normalized}}:}} Since $|\underset{x}{\sup}\;f(x)-\underset{x}{\sup}\;g(x)|\leq\underset{x}{\sup}|f(x)-g(x)|$, we have
%{\small
\begin{equation*}
\begin{aligned}
\text{Eqn. \eqref{eqt:IPW_bound_2 normalized}}&\leq2\underset{\pi\in\Pi}{\sup}\;\underset{\alpha\geq 0}{\sup}\;\alpha\big|\log\mathbb{E}_{\mathbb{P}_{N}^{h}}[e^{-\frac{Y(\pi(X))}{\alpha}}]-\log\mathbb{E}[e^{-\frac{Y(\pi(X))}{\alpha}}]\big|.
\end{aligned}
\end{equation*}
%}\noindent
Combining the resulting bounds of Eqns. \eqref{eqt:IPW_bound_1 normalized} and \eqref{eqt:IPW_bound_2 normalized}, we have
%{\small
\begin{equation*}
\begin{aligned}
\text{Eqn. \eqref{eqt:RDO normalized}}&\leq\;O(h^{2})+2\underset{\pi\in\Pi}{\sup}\;\underset{\alpha\geq 0}{\sup}\;\alpha\big|\log\mathbb{E}_{\mathbb{P}_{N}^{h}}[e^{-\frac{Y(\pi(X))}{\alpha}}]-\log\mathbb{E}[e^{-\frac{Y(\pi(X))}{\alpha}}]\big|.
\end{aligned}
\end{equation*}
%}\noindent
It remains to bound the term 
\begin{equation*}
\begin{aligned}
2\underset{\pi\in\Pi}{\sup}\;\underset{\alpha\geq 0}{\sup}\;\alpha\big|\log\mathbb{E}_{\mathbb{P}_{N}^{h}}[e^{-\frac{Y(\pi(X))}{\alpha}}]-\log\mathbb{E}[e^{-\frac{Y(\pi(X))}{\alpha}}]\big|.
\end{aligned}
\end{equation*}
We now show that the claim holds. The derivations are as follows: first, undergoing usual derivations gives that
%{\small
\begin{equation*}
\begin{aligned}
\mathbb{E}\bigg[\frac{K_{h}(\pi(X)-A)}{f_{0}(A|X)}\mathbf{1}_{\{Y(\pi(X))\leq x\}}\bigg]=\mathbb{E}[\mathbf{1}_{\{Y(\pi(X))\leq x\}}]+O(h^{2}).
\end{aligned}
\end{equation*}
%}\noindent
Hence, we have
{\small
\begin{equation*}
\begin{aligned}
&2\underset{\pi\in\Pi}{\sup}\;\underset{\alpha\geq 0}{\sup}\;\alpha\big|\log\mathbb{E}_{\mathbb{P}_{N}^{h}}[e^{-\frac{Y(\pi(X))}{\alpha}}]-\log\mathbb{E}[e^{-\frac{Y(\pi(X))}{\alpha}}]\big|\\
&\overset{\ddagger}{\leq}2\underset{\pi\in\Pi}{\sup}\;\underset{x\in[0,M]}{\sup}\;\frac{1}{\epsilon}\big|\mathbb{E}_{\mathbb{P}_{N}^{h}}[\mathbf{1}_{\{Y(\pi(X))\leq x\}}]-\mathbb{E}[\mathbf{1}_{\{Y(\pi(X))\leq x\}}]\big|\\
%=&2\underset{\pi\in\Pi}{\sup}\;\underset{x\in[0,M]}{\sup}\;\frac{1}{\epsilon}\bigg|\frac{1}{NS_{N}^{h}}\underset{i=1}{\overset{N}{\sum}}\frac{K_{h}(\pi(X_{i})-A_{i})\mathbf{1}_{\{Y_{i}(\pi(X_{i}))\leq x\}}}{f_{0}(A_{i}|X_{i})}-\mathbb{E}[\mathbf{1}_{\{Y(\pi(X))\leq x\}}]\bigg|\\
&\begin{aligned}
\leq2\underset{\pi\in\Pi}{\sup}\;\underset{x\in[0,M]}{\sup}\;&\frac{1}{\epsilon}\left|\frac{1}{NS_{N}^{h}}\underset{i=1}{\overset{N}{\sum}}\frac{K_{h}(\pi(X_{i})-A_{i})\mathbf{1}_{\{Y_{i}(\pi(X_{i}))\leq x\}}}{f_{0}(A_{i}|X_{i})}-\mathbb{E}\bigg[\frac{K_{h}(\pi(X)-A)}{f_{0}(A|X)}\mathbf{1}_{\{Y(\pi(X))\leq x\}}\bigg]\right|+O(h^{2})
\end{aligned}\\
&\begin{aligned}
\leq2\underset{\pi\in\Pi,x\in[0,M]}{\sup}\;&\frac{1}{\epsilon}\left|\frac{1}{N}\underset{i=1}{\overset{N}{\sum}}\frac{K_{h}(\pi(X_{i})-A_{i})\mathbf{1}_{\{Y_{i}(\pi(X_{i}))\leq x\}}}{f_{0}(A_{i}|X_{i})}-\mathbb{E}\bigg[\frac{K_{h}(\pi(X)-A)}{f_{0}(A|X)}\mathbf{1}_{\{Y(\pi(X))\leq x\}}\bigg]\right|+O(h^{2})
\end{aligned}\\
&+2\underset{\pi\in\Pi,x\in[0,M]}{\sup}\;\frac{1}{\epsilon}\bigg|\frac{(S_{N}^{h}-1)}{NS_{N}^{h}}\underset{i=1}{\overset{N}{\sum}}\frac{K_{h}(\pi(X_{i})-A_{i})\mathbf{1}_{\{Y_{i}(\pi(X_{i}))\leq x\}}}{f_{0}(A_{i}|X_{i})}\bigg|.
\end{aligned}
\end{equation*}
}\noindent
$\ddagger$ is due to Proposition \ref{lemma:QDRO and quantile}. By \cite{wainwright2019high}, we have with probability at least $1-e^{-\frac{Nh^{2}\epsilon^{2}\gamma^{2}}{2M_{K}^{2}}}$ that
%{\small
\begin{align}
&\begin{aligned}
\underset{\pi\in\Pi,x\in[0,M]}{\sup}\;\frac{1}{\epsilon}\left|\frac{1}{N}\underset{i=1}{\overset{N}{\sum}}\frac{K_{h}(\pi(X_{i})-A_{i})\mathbf{1}_{\{Y_{i}(\pi(X_{i}))\leq x\}}}{f_{0}(A_{i}|X_{i})}-\mathbb{E}\bigg[\frac{K_{h}(\pi(X)-A)}{f_{0}(A|X)}\mathbf{1}_{\{Y(\pi(X))\leq x\}}\bigg]\right|
\end{aligned}\nonumber\\
&\leq\frac{1}{\epsilon}(2\mathcal{R}_{N}(\mathcal{F}_{\Pi,x})+\gamma),\label{eqt:Rademacher bound prob 1}\raisetag{20pt}
\end{align}
%}\noindent
where the function class $\mathcal{F}_{\Pi,x}$ is defined such that
%{\small
%\begin{equation*}
%\begin{aligned}
%\mathcal{F}_{\Pi,x}:=\Bigg\{f_{\pi,x}(X,Y,A)=\frac{K_{h}(\pi(X)-A)\mathbf{1}_{\{Y(\pi(X))\leq x\}}}{f_{0}(A|X)}:\pi\in\Pi,\;x\in[0,M]\Bigg\}.
%\end{aligned}
%\end{equation*}
\begin{equation*}
\begin{aligned}
\mathcal{F}_{\Pi,x}:=\Bigg\{\frac{K_{h}(\pi(X)-A)\mathbf{1}_{\{Y(\pi(X))\leq x\}}}{f_{0}(A|X)}:\pi\in\Pi,\;x\in[0,M]\Bigg\}.
\end{aligned}
\end{equation*}
%}\noindent
Additionally, we have
%{\small
\begin{equation*}
\begin{aligned}
&\underset{\pi\in\Pi,x\in[0,M]}{\sup}\;\frac{1}{\epsilon}\bigg|\frac{(S_{N}^{h}-1)}{NS_{N}^{h}}\underset{i=1}{\overset{N}{\sum}}\frac{K_{h}(\pi(X_{i})-A_{i})\mathbf{1}_{\{Y_{i}(\pi(X_{i}))\leq x\}}}{f_{0}(A_{i}|X_{i})}\bigg|\leq\underset{\pi\in\Pi}{\sup}\;\frac{1}{\epsilon}|S_{N}^{h}-1|.
\end{aligned}
\end{equation*}
%}\noindent
Again, by \cite{wainwright2019high}, we have with probability at least $1-e^{-\frac{Nh^{2}\epsilon^{2}\gamma^{2}}{2M_{K}^{2}}}$ that
\begin{equation}
\begin{aligned}\label{eqt:Rademacher bound prob 2}
\underset{\pi\in\Pi}{\sup}|S_{N}^{h}-1|\leq 2\mathcal{R}_{N}(\mathcal{F}_{\Pi})+\gamma+O(h^{2})
\end{aligned}
\end{equation}
where the function class $\mathcal{F}_{\Pi}$ is defined such that
\begin{equation*}
\begin{aligned}
\mathcal{F}_{\Pi}:=\Bigg\{f_{\pi}(X,Y,A)=\frac{K_{h}(\pi(X)-A)}{f_{0}(A|X)}:\pi\in\Pi\Bigg\}.
\end{aligned}
\end{equation*}
Hence, combining Eqns. \eqref{eqt:Rademacher bound prob 1} and \eqref{eqt:Rademacher bound prob 2}, the following result hold: with probability $1-2e^{-\frac{Nh^{2}\epsilon^{2}\gamma^{2}}{2M_{K}^{2}}}$, we have
%{\small
\begin{equation}
\begin{aligned}\label{eqt:Rademacher bound result}
&2\underset{\pi\in\Pi}{\sup}\;\underset{\alpha\geq 0}{\sup}\;\alpha\big|\log\mathbb{E}_{\mathbb{P}_{N}^{h}}[e^{-\frac{Y(\pi(X))}{\alpha}}]-\log\mathbb{E}[e^{-\frac{Y(\pi(X))}{\alpha}}]\big|\\
&\leq\frac{2}{\epsilon}(2\mathcal{R}_{N}(\mathcal{F}_{\Pi,x})+\gamma)+\frac{2}{\epsilon}(2\mathcal{R}_{N}(\mathcal{F}_{\Pi})+\gamma+O(h^{2})).
%\\
%&=\frac{4}{\epsilon}(\mathcal{R}_{N}(\mathcal{F}_{\Pi,x})+\mathcal{R}_{N}(\mathcal{F}_{\Pi})+\gamma)+O(h^{2}).
\end{aligned}
\end{equation}
%}\noindent
\end{proof}
\subsubsection{Proof of Corollary \ref{cor:RDRO bound finite covering}}\label{sec:Corollary main}
We restate the corollary here.
\begin{corollary*}
If the kernel function $K(x)$ is Lipschitz continuous with constant $L_{K}>0$ (i.e., $|K(x)-K(y)|\leq L_{K}|x-y|$) and there exists a finite value $\kappa$ such that
%{\small
\begin{equation*}
\begin{aligned}
\kappa\coloneqq\mathbb{E}\Bigg[\int_{0}^{\frac{2M_{K}h}{L_{K}}}\sqrt{\log\;\mathfrak{N}\big(t,\Pi(\{X_{1},\cdots,X_{N}\}),\|\cdot\|_{\mathcal{L}_{2}(\mathbb{P}_{N})}\big)}dt\Bigg].
\end{aligned}
\end{equation*}
%}\noindent
Then for some constant $\mathcal{K}$, Eqn. \eqref{eqt:statistical performance normalized Rademacher results} becomes
\begin{gather*}
\begin{aligned}
&R_{\mathrm{DRO}}(\hat{\pi}_{\mathrm{DRO}}^{h})\leq \frac{288L_{K}\kappa}{\sqrt{N}h^{2}\epsilon^{2}}+\frac{192M_{K}(\sqrt{\log\mathcal{K}}+2\sqrt{2})}{\sqrt{N}h\epsilon^{2}}+\frac{4M_{K}\sqrt{2\log\big(\frac{2}{\delta}\big)}}{\sqrt{N}h\epsilon^{2}}+O(h^{2}).
\end{aligned}
%\raisetag{50pt}
\end{gather*}
\end{corollary*}
\begin{proof}
We consider bounding $\mathcal{R}_{N}(\mathcal{F}_{\Pi,x})$ and $\mathcal{R}_{N}(\mathcal{F}_{\Pi})$ in Eqn. \eqref{eqt:Rademacher bound result}.
%Now, we turn to study the Rademacher complexity of the classes $\mathcal{F}_{\Pi}$ and $\mathcal{F}_{\Pi,x}$.
Consider the class $\mathcal{F}_{\Pi}$. Since $K(\cdot)$ is Lipschitz, we have
%{\small
\begin{equation*}
\begin{aligned}
&\sqrt{\frac{1}{N}\underset{i=1}{\overset{N}{\sum}}\bigg(\frac{K_{h}(\pi_{1}(X_{i})-A_{i})}{f_{0}(A_{i}|X_{i})}-\frac{K_{h}(\pi_{2}(X_{i})-A_{i})}{f_{0}(A_{i}|X_{i})}\bigg)^{2}}\\
%\leq&\sqrt{\frac{1}{Nh^{2}\epsilon^{2}}\underset{i=1}{\overset{N}{\sum}}\bigg(K\bigg(\frac{\pi_{1}(X_{i})-A_{i}}{h}\bigg)-K\bigg(\frac{\pi_{2}(X_{i})-A_{i}}{h}\bigg)\bigg)^{2}}\\
%\leq&\sqrt{\frac{L_{K}^{2}}{Nh^{2}\epsilon^{2}}\underset{i=1}{\overset{N}{\sum}}\bigg(\frac{\pi_{1}(X_{i})-A_{i}}{h}-\frac{\pi_{2}(X_{i})-A_{i}}{h}\bigg)^{2}}\\
\leq&\sqrt{\frac{L_{K}^{2}}{Nh^{4}\epsilon^{2}}\underset{i=1}{\overset{N}{\sum}}(\pi_{1}(X_{i})-\pi_{2}(X_{i}))^{2}}=\frac{L_{K}}{h^{2}\epsilon}\sqrt{\frac{1}{N}\underset{i=1}{\overset{N}{\sum}}(\pi_{1}(X_{i})-\pi_{2}(X_{i}))^{2}}.
\end{aligned}
\end{equation*}
%}\noindent
Thus, we can conclude the following result for the coverage number:
%{\small
\begin{equation}
\begin{aligned}\label{eqt:covering number bound 1}
&\mathfrak{N}(t,\mathcal{F}_{\Pi}(X_{1},\cdots,X_{N}),\|\cdot\|_{\mathcal{L}_{2}(\mathbb{P}_{N})})\\
&\leq \mathfrak{N}\bigg(t,\Pi(X_{1},\cdots,X_{N}),\frac{L_{K}}{h^{2}\epsilon}\|\cdot\|_{\mathcal{L}_{2}(\mathbb{P}_{N})}\bigg)=\mathfrak{N}\bigg(\frac{h^{2}\epsilon}{L_{K}}t,\Pi(X_{1},\cdots,X_{N}),\|\cdot\|_{\mathcal{L}_{2}(\mathbb{P}_{N})}\bigg).
\end{aligned}
\end{equation}
%}\noindent
Here, $\mathfrak{N}(t,\Pi(X_{1},\cdots,X_{N}),\|\cdot\|_{\mathcal{L}_{2}(\mathbb{P}_{N})})$ is the covering number of $\Pi$ under the $\mathcal{L}_{2}$ norm with the probability measure $\mathbb{P}_{N}$. Mathematically, we can find a cover $\mathcal{A}$ for $\Pi$ such that for any $y\in\Pi$, there exists $\tilde{y}\in\mathcal{A}$ such that $\|\tilde{y}-y\|_{\mathcal{L}_{2}(\mathbb{P}_{N})} \leq t$. We then consider the class $\mathcal{F}_{\Pi,x}$, and we claim that
%{\small
\begin{equation}
\begin{aligned}\label{eqt:covering number bound 2}
&\mathfrak{N}(t,\mathcal{F}_{\Pi,x}(X_{1},\cdots,X_{N}),\|\cdot\|_{\mathcal{L}_{2}(\mathbb{P}_{N})})\\
\leq&\mathfrak{N}\bigg(\frac{h^{2}\epsilon t}{2L_{K}},\Pi(X_{1},\cdots,X_{N}),\|\cdot\|_{\mathcal{L}_{2}(\mathbb{P}_{N})}\bigg)\times \underset{\mathbb{P}}{\sup}\;\mathfrak{N}\bigg(\frac{h\epsilon t}{2M_{K}},\mathcal{F}_{\mathbf{I}}(X_{1},\cdots,X_{N}),\|\cdot\|_{\mathcal{L}_{2}(\mathbb{P})}\bigg)\\
\coloneqq&\mathfrak{N}_{\Pi}(t)\times \mathfrak{N}_{\mathbf{I}}(t),
\end{aligned}
\end{equation}
%}\noindent
where $\mathcal{F}_{\mathbf{I}}=\{f(t)=\mathbf{1}_{\{t\leq x\}}:x\in[0,M]\}$. Suppose now $\{\pi_{1},\cdots,\pi_{\mathfrak{N}_{\Pi}(t)}\}$ is a cover of $\Pi$ and $\{\mathbf{1}_{\{t\leq x_{1}\}},\cdots,\mathbf{1}_{\{t\leq x_{\mathfrak{N}_{\mathbf{I}}(t)}\}}\}$ is a cover of $\mathcal{F}_{\mathbf{I}}$ under the distance $\|\cdot\|_{\mathcal{L}_{2}(\mathbb{P}_{N})}$
We aim to show that $\mathcal{F}_{\Pi,x}^{t}$ is $t$-cover set of $\mathcal{F}_{\Pi,x}$, where
%{\small
\begin{equation*}
\begin{aligned}
\mathcal{F}_{\Pi,x}^{t}=\Bigg\{\frac{K_{h}(\pi_{i}(X)-A)\mathbf{1}_{\{Y(\pi_{i}(X))\leq x_{j}\}}}{f_{0}(A|X)}:i\leq \mathfrak{N}_{\Pi}(t),j\leq \mathfrak{N}_{\mathbf{I}}(t)\Bigg\}.
\end{aligned}
\end{equation*}
%}\noindent
Indeed, for any $f_{\pi,x}(X,Y,A)\in\mathcal{F}_{\Pi,x}$, we can pick $\tilde{\pi},\tilde{x}$ such that $f_{\tilde{\pi},\tilde{x}}(X,Y,A)\in\mathcal{F}_{\Pi,x}^{t}$ such that $\|\mathbf{1}_{\{Y(\pi(X))\leq x\}}-\mathbf{1}_{\{Y(\tilde{\pi}(X))\leq \tilde{x}\}}\|_{\mathcal{L}_{2}(\mathbb{P}_{N})}\leq \frac{h\epsilon t}{2M_{K}}$ and $\|\pi-\tilde{\pi}\|_{\mathcal{L}_{2}(\mathbb{P}_{N})}\leq \frac{h^{2}\epsilon t}{2L_{K}}$. Denote
%{\small
\begin{equation*}
\begin{aligned}
\text{Diff}=\frac{1}{N}\underset{i=1}{\overset{N}{\sum}}&\left(\frac{K_{h}(\pi(X_{i})-A_{i})\mathbf{1}_{\{Y(\pi(X))\leq x\}}}{f_{0}(A_{i}|X_{i})}-\frac{K_{h}(\tilde{\pi}(X_{i})-A_{i})\mathbf{1}_{\{Y(\tilde{\pi}(X))\leq \tilde{x}\}}}{f_{0}(A_{i}|X_{i})}\right)^{2}
\end{aligned}
\end{equation*}
%}\noindent
Then we have
%{\small
\begin{equation*}
\begin{aligned}
&\text{Diff}\\
\leq&\frac{1}{h^{2}\epsilon^{2}}\times\frac{1}{N}\underset{i=1}{\overset{N}{\sum}}\left(K\bigg(\frac{\pi(X_{i})-A_{i}}{h}\bigg)\mathbf{1}_{\{Y_{i}(\pi(X_{i}))\leq x\}}-K\bigg(\frac{\tilde{\pi}(X_{i})-A_{i}}{h}\bigg)\mathbf{1}_{\{Y_{i}(\tilde{\pi}(X_{i}))\leq \tilde{x}\}}\right)^{2}\\
\leq&\frac{1}{h^{2}\epsilon^{2}}\times\Bigg\{\frac{2}{N}\underset{i=1}{\overset{N}{\sum}}\bigg[K\bigg(\frac{\pi(X_{i})-A_{i}}{h}\bigg)\times[\mathbf{1}_{\{Y_{i}(\pi(X_{i}))\leq x\}}-\mathbf{1}_{\{Y_{i}(\tilde{\pi}(X_{i}))\leq \tilde{x}\}}]\bigg]^{2}\\
&\quad\quad\quad\quad+\frac{2}{N}\underset{i=1}{\overset{N}{\sum}}\bigg(K\bigg(\frac{\pi(X_{i})-A_{i}}{h}\bigg)-K\bigg(\frac{\tilde{\pi}(X_{i})-A_{i}}{h}\bigg)\bigg)^{2}\times\mathbf{1}_{\{Y(\tilde{\pi}(X_{i}))\leq \tilde{x}\}}\Bigg\}\\
\leq&\frac{1}{h^{2}\epsilon^{2}}\Bigg\{\frac{2}{N}\underset{i=1}{\overset{N}{\sum}}M_{K}^{2}[\mathbf{1}_{\{Y_{i}(\pi(X_{i}))\leq x\}}-\mathbf{1}_{\{Y_{i}(\tilde{\pi}(X_{i}))\leq \tilde{x}\}}]^{2}+\frac{2}{N}\underset{i=1}{\overset{N}{\sum}}\frac{L_{K}^{2}}{h^{2}}(\pi(X_{i})-\tilde{\pi}(X_{i}))^{2}\Bigg\}\\
\leq&\frac{1}{h^{2}\epsilon^{2}}\bigg(2M_{K}^{2}\times\frac{h^{2}\epsilon^{2}t^{2}}{4M_{K}^{2}}+\frac{2L_{K}^{2}}{h^{2}}\times\frac{h^{4}\epsilon^{2}t^{2}}{4L_{K}^{2}}\bigg)\leq t^{2}.
\end{aligned}
\end{equation*}
%}\noindent
Hence, we have
%{\small
\begin{equation*}
\begin{aligned}
&\sqrt{\text{Diff}}\leq t.
\end{aligned}
\end{equation*}
%}\noindent
To further proceed with the proof, we need the result of the Dudley's integral formula given in \cite{wainwright2019high}. We state the result here as a Proposition:
\begin{proposition}
Given that a function class $\mathcal{F}$ and Rademacher variables $(\sigma_{i})_{i=1}^{N}$ where $\mathbb{P}\{\sigma_{i}=1\}=\mathbb{P}\{\sigma_{i}=-1\}=\frac{1}{2}$, then we have
%{\small
\begin{equation*}
\begin{aligned}
&\hat{\mathcal{R}}_{N}(\mathcal{F})=\mathbb{E}_{\sigma}\bigg[\underset{f\in\mathcal{F}}{\sup}\bigg|\frac{1}{N}\underset{i=1}{\overset{N}{\sum}}\sigma_{i}f(X_{i})\bigg|\bigg|X_{1},\cdots,X_{N}\bigg]\\
&\leq \frac{24}{\sqrt{N}}\int_{0}^{2b}\sqrt{\log\;\mathfrak{N}(t,\mathcal{F}(\{X_{1},\cdots,X_{N}\}),\|\cdot\|_{\mathcal{L}_{2}(\mathbb{P}_{N})})}dt,
\end{aligned}
\end{equation*}
%}\noindent
%where $\mathcal{R}_{N}(\mathcal{F})=\mathbb{E}[\hat{\mathcal{R}}_{N}(\mathcal{F})]$ such that the expectation $\mathbb{E}[\cdot]$ is taken over $X_{1},\cdots,X_{N}$ and $b$ is chosen such that $\underset{f,g\in\mathcal{F}}{\sup}\|f-g\|_{\mathcal{L}_{2}(\mathbb{P}_{N})}\leq 2b$. Here, $\mathbb{E}_{\sigma}[\cdot]$ is the expectation over the Rademacher variables.
where $b$ is chosen such that $\underset{f,g\in\mathcal{F}}{\sup}\|f-g\|_{\mathcal{L}_{2}(\mathbb{P}_{N})}\leq 2b$ and $\mathbb{E}_{\sigma}[\cdot]$ is the expectation over the Rademacher variables.
\end{proposition} 
\noindent Now, we apply Dudley's integral formula on $\mathcal{F}_{\Pi}$. Since $\underset{f,g\in\mathcal{F}_{\Pi}}{\sup}\|f-g\|_{\mathcal{L}_{2}(\mathbb{P}_{N})}\leq \frac{2M_{K}}{\epsilon h}$ and we have that $\mathcal{R}_{N}(\mathcal{F})=\mathbb{E}[\hat{\mathcal{R}}_{N}(\mathcal{F})]$ where the expectation $\mathbb{E}[\cdot]$ is taken over $X_{1},\cdots,X_{N}$ according to Definition \ref{def:Rademacher}, we obtain
%{\small
\begin{equation*}
\begin{aligned}
&\mathcal{R}_{N}(\mathcal{F}_{\Pi})\leq\frac{24}{\sqrt{N}}\mathbb{E}\bigg[\int_{0}^{\frac{2M_{K}}{\epsilon h}}\sqrt{\log\mathfrak{N}(t,\mathcal{F}_{\Pi}(\{X_{1},\cdots,X_{N}\}),\|\cdot\|_{\mathcal{L}_{2}(\mathbb{P}_{N})})}dt\bigg]\\
\leq&\frac{24}{\sqrt{N}}\mathbb{E}\bigg[\int_{0}^{\frac{2M_{K}}{\epsilon h}}\sqrt{\log\mathfrak{N}\bigg(\frac{h^{2}\epsilon t}{L_{K}},\Pi(\{X_{1},\cdots,X_{N}\}),\|\cdot\|_{\mathcal{L}_{2}(\mathbb{P}_{N})}\bigg)}dt\bigg]\\
=&\frac{24L_{K}}{\sqrt{N}h^{2}\epsilon}\mathbb{E}\bigg[\int_{0}^{\frac{2M_{K}h}{L_{K}}}\sqrt{\log\mathfrak{N}_{\Pi}(s)}ds\bigg].
\end{aligned}
\end{equation*}
%}\noindent
where 
%{\small
\begin{equation*}
\begin{gathered}
\mathfrak{N}_{\Pi}\bigg(\frac{h^{2}\epsilon t}{L_{K}}\bigg)=\mathfrak{N}\bigg(\frac{h^{2}\epsilon t}{L_{K}},\Pi(\{X_{1},\cdots,X_{N}\}),\|\cdot\|_{\mathcal{L}_{2}(\mathbb{P}_{N})}\bigg).
\end{gathered}
\end{equation*}
%}\noindent
Simultaneously, since $\underset{f,g\in\mathcal{F}_{\Pi,x}}{\sup}\|f-g\|\leq\frac{2M_{K}}{\epsilon h}$, applying the Dudley's Integral formula on the $\mathcal{F}_{\Pi,x}$, together with the result given in Eqn. \eqref{eqt:covering number bound 2}, gives
%{\small
\begin{equation*}
\begin{aligned}
&\mathcal{R}_{N}(\mathcal{F}_{\Pi,x})\leq \frac{24}{\sqrt{N}}\mathbb{E}\bigg[\int_{0}^{\frac{2M_{K}}{\epsilon h}}\sqrt{\log \mathfrak{N}(t,\mathcal{F}_{\Pi,x}(X_{1},\cdots,X_{N}),\|\cdot\|_{\mathcal{L}_{2}(\mathbb{P}_{N})})}dt\bigg]\\
&\leq \frac{24}{\sqrt{N}}\mathbb{E}\left[\int_{0}^{\frac{2M_{K}}{\epsilon h}}\sqrt{\log\mathfrak{N}_{\Pi}\bigg(\frac{h^{2}\epsilon t}{2L_{K}}\bigg)+\log \mathfrak{N}_{\mathcal{F}_{\mathbf{I}}}\bigg(\frac{h\epsilon t}{2M_{K}}\bigg)}dt\right],
\end{aligned}
\end{equation*}
%}\noindent
where 
%{\small
\begin{equation*}
\begin{gathered}
%\mathfrak{N}_{\Pi}\bigg(\frac{h^{2}\epsilon t}{2L_{K}}\bigg)=\mathfrak{N}\bigg(\frac{h^{2}\epsilon t}{2L_{K}},\Pi(\{X_{1},\cdots,X_{N}\}),\|\cdot\|_{\mathcal{L}_{2}(\mathbb{P}_{N})}\bigg)\\
\mathfrak{N}_{\mathcal{F}_{\mathbf{I}}}\bigg(\frac{h\epsilon t}{2M_{K}}\bigg)=\underset{\mathbb{P}}{\sup}\;\mathfrak{N}\bigg(\frac{h\epsilon t}{2M_{K}},\mathcal{F}_{\mathbf{I}}(\{X_{1},\cdots,X_{N}\}),\|\cdot\|_{\mathcal{L}_{2}(\mathbb{P}_{N})}\bigg).
\end{gathered}
\end{equation*}
%}\noindent
Further, according to \cite{van2000asymptotic}, we have
\begin{equation}
\begin{aligned}\label{eqt:covering number indicator function}
\underset{\mathbb{P}}{\sup}\;\mathfrak{N}(t,\mathcal{F}_{\mathbf{I}}(\{X_{1},\cdots,X_{N}\}),\mathbb{P})\leq\mathcal{K}\bigg(\frac{1}{t}\bigg)^{2}
\end{aligned}
\end{equation}
for any arbitrary $(X_{1},\cdots,X_{N})$, $N$, and some universal constant $\mathcal{K}$. Together with the fact that $\sqrt{a+b}\leq\sqrt{a}+\sqrt{b}$ for $a,\;b>0$ and the fact that $\int_{0}^{1}\sqrt{\log\big(\frac{1}{t}\big)}dt\leq \int_{0}^{1}\sqrt{\frac{1}{t}}dt=2$, we therefore conclude that
%{\small
\begin{equation*}
\begin{aligned}
&\mathcal{R}_{N}(\mathcal{F}_{\Pi,x})\\
&\leq \frac{24}{\sqrt{N}}\mathbb{E}\Bigg[\int_{0}^{\frac{2M_{K}}{\epsilon h}}\sqrt{\log\;\mathfrak{N}_{\Pi}\bigg(\frac{h^{2}\epsilon t}{2L_{K}}\bigg)}dt\Bigg]+\frac{24}{\sqrt{N}}\mathbb{E}\Bigg[\int_{0}^{\frac{2M_{K}}{\epsilon h}} \sqrt{\log\mathfrak{N}_{\mathcal{F}_{\mathbf{I}}}\bigg(\frac{h\epsilon t}{2M_{K}}\bigg)}dt\Bigg]\\
&\leq \frac{48L_{K}}{\sqrt{N}h^{2}\epsilon}\mathbb{E}\Bigg[\int_{0}^{\frac{M_{K}h}{L_{K}}}\sqrt{\log\;\mathfrak{N}_{\Pi}(t)}dt\Bigg]+\frac{48M_{K}}{\sqrt{N}h\epsilon}\mathbb{E}\Bigg[\int_{0}^{1} \sqrt{\log\mathfrak{N}_{\mathcal{F}_{\mathbf{I}}}(t)}dt\Bigg]\\
&\leq \frac{48L_{K}}{\sqrt{N}h^{2}\epsilon}\mathbb{E}\Bigg[\int_{0}^{\frac{M_{K}h}{L_{K}}}\sqrt{\log\;\mathfrak{N}_{\Pi}(t)}dt\Bigg]+\frac{48M_{K}}{\sqrt{N}h\epsilon}\Bigg[\sqrt{\log\mathcal{K}}+\int_{0}^{1} \sqrt{2\log\bigg(\frac{1}{t}\bigg)} dt\Bigg]\\
&\leq \frac{48L_{K}}{\sqrt{N}h^{2}\epsilon}\mathbb{E}\Bigg[\int_{0}^{\frac{M_{K}h}{L_{K}}}\sqrt{\log\;\mathfrak{N}_{\Pi}(t)}dt\Bigg]+\frac{48M_{K}}{\sqrt{N}h\epsilon}[\sqrt{\log\mathcal{K}}+2\sqrt{2}].
\end{aligned}
\end{equation*}
As a result, we have
%{\small
\begin{equation*}
\begin{aligned}
&2\underset{\pi\in\Pi}{\sup}\;\underset{\alpha\geq 0}{\sup}\;\alpha\big|\log\mathbb{E}_{\mathbb{P}_{N}^{h}}[e^{-\frac{Y(\pi(X))}{\alpha}}]-\log\mathbb{E}[e^{-\frac{Y(\pi(X))}{\alpha}}]\big|\\
&\leq\frac{2}{\epsilon}(2\mathcal{R}_{N}(\mathcal{F}_{\Pi,x})+\gamma)+\frac{2}{\epsilon}(2\mathcal{R}_{N}(\mathcal{F}_{\Pi})+\gamma+O(h^{2}))\\
&\begin{aligned}
\leq\frac{2}{\epsilon}&\left\{\frac{96L_{K}}{\sqrt{N}h^{2}\epsilon}\mathbb{E}\Bigg[\int_{0}^{\frac{M_{K}h}{L_{K}}}\sqrt{\log\;\mathfrak{N}_{\Pi}(t)}dt\Bigg]+\frac{96M_{K}}{\sqrt{N}h\epsilon}[\sqrt{\log\mathcal{K}}+2\sqrt{2}]+\gamma\right\}\\
\end{aligned}\\
&\quad+\frac{2}{\epsilon}\Bigg\{\frac{48L_{K}}{\sqrt{N}h^{2}\epsilon}\mathbb{E}\bigg[\int_{0}^{\frac{2M_{K}h}{L_{K}}}\sqrt{\log\;\mathfrak{N}_{\Pi}(s)}ds\bigg]+\gamma\Bigg\}+O(h^{2})\\
&=\frac{192L_{K}}{\sqrt{N}h^{2}\epsilon^{2}}\mathbb{E}\Bigg[\int_{0}^{\frac{M_{K}h}{L_{K}}}\sqrt{\log\;\mathfrak{N}_{\Pi}(t)}dt\Bigg]+\frac{192M_{K}(\sqrt{\log\mathcal{K}}+2\sqrt{2})}{\sqrt{N}h\epsilon^{2}}\\
&\quad+\Bigg\{\frac{96L_{K}}{\sqrt{N}h^{2}\epsilon^{2}}\mathbb{E}\bigg[\int_{0}^{\frac{2M_{K}h}{L_{K}}}\sqrt{\log\;\mathfrak{N}_{\Pi}(s)}ds\bigg]\Bigg\}+\frac{4\gamma}{\epsilon}+O(h^{2})\\
&\leq\Bigg\{\frac{288L_{K}}{\sqrt{N}h^{2}\epsilon^{2}}\mathbb{E}\bigg[\int_{0}^{\frac{2M_{K}h}{L_{K}}}\sqrt{\log\;\mathfrak{N}_{\Pi}(t)}dt\bigg]\Bigg\}+\frac{192M_{K}(\sqrt{\log\mathcal{K}}+2\sqrt{2})}{\sqrt{N}h\epsilon^{2}}+\frac{4\gamma}{\epsilon}+O(h^{2}).
\end{aligned}
\end{equation*}
To conclude, we set $\delta=2e^{-\frac{Nh^{2}\epsilon^{2}\gamma^{2}}{2M_{K}^{2}}}$ such that $\gamma=\frac{M_{K}\sqrt{2\log \big(\frac{2}{\delta}\big)}}{\sqrt{N}h\epsilon}$, then with probability $1-\delta$, we have
%{\small
\begin{equation*}
\begin{aligned}
R_{\mathrm{DRO}}(\hat{\pi}_{\mathrm{DRO}}^{h})&\leq2M+\frac{288L_{K}}{\sqrt{N}h^{2}\epsilon^{2}}\mathbb{E}\Bigg[\int_{0}^{\frac{2M_{K}h}{L_{K}}}\sqrt{\log\;\mathfrak{N}_{\Pi}(t)}dt\Bigg]\\
&\quad+\frac{192M_{K}(\sqrt{\log\mathcal{K}}+2\sqrt{2})}{\sqrt{N}h\epsilon^{2}}+\frac{4M_{K}\sqrt{2\log\big(\frac{2}{\delta}\big)}}{\sqrt{N}h\epsilon^{2}}+O(h^{2}).
\end{aligned}
\end{equation*}
%}\noindent
\end{proof}
\end{document}